\documentclass[11pt]{article}
\usepackage{mathrsfs,amsmath,amsfonts,amssymb,bm,bbm,dsfont,pifont,amscd,stmaryrd,euscript,amsthm,color,epsfig,xr,yfonts,tikz,verbatim}

\usepackage{tikz-cd} 
\usepackage{authblk}
\usepackage[utf8]{inputenc} 
\usepackage[T1]{fontenc}    
\usepackage{url}            
\usepackage{booktabs}       
\usepackage{amsfonts}  
\usepackage{amsmath} 
\usepackage{nicefrac}       
\usepackage{microtype}      
\usepackage{xcolor}
\usepackage{hyperref}  
\usepackage{graphicx}
\usepackage{array}
\usepackage{subcaption,mathrsfs}
\usepackage{natbib}
\usepackage{wrapfig}
\usepackage{tikz}
\usepackage[makeroom]{cancel}
\usepackage{algorithm}
\usepackage{algpseudocode}

\newenvironment{talign*}
 {\csname align*\endcsname}
 {\endalign}

\allowdisplaybreaks
\usetikzlibrary{decorations.markings}

\definecolor{mydarkblue}{rgb}{0,0.08,0.45}

\hypersetup{ %
    pdftitle={},
    pdfauthor={},
    pdfsubject={},
    pdfkeywords={},
    pdfborder=0 0 0,
    pdfpagemode=UseNone,
    colorlinks=true,
    linkcolor=mydarkblue,
    citecolor=mydarkblue,
    filecolor=mydarkblue,
    urlcolor=mydarkblue,
    pdfview=FitH
}

\makeatletter
\newcommand\subsubsubsection{\@startsection{subsubsubsection}{4}{\z@}%
  {-3.25ex\@plus -1ex \@minus -.2ex}%
  {1.5ex \@plus .2ex}%
  {\normalfont\normalsize\bfseries}}
\newcommand\subsubsubsectionmark[1]{}

\makeatother

\definecolor{ifred}{RGB}{190,30,45}
\definecolor{idblue}{RGB}{25,94,160}
\definecolor{sgreen}{RGB}{24,140,69}

\newcommand{\mr}[1]{\textcolor{ifred}{#1}}
\newcommand{\mb}[1]{\textcolor{idblue}{#1}}
\newcommand{\mg}[1]{\textcolor{sgreen}{#1}}
\newcommand{\nocolourb}[1]{\textcolor{black}{#1}}

\usepackage[utf8]{inputenc}
\usepackage{graphicx}
\usepackage{caption}
\usepackage{subcaption}
\usepackage{wrapfig}
\usepackage{hyperref}
\usepackage{amsfonts,amsmath,amsthm,amssymb,stmaryrd,bbm,bm}
\usepackage{mathrsfs}
\usepackage{mathtools}
\usepackage{enumitem}
\usepackage{thmtools}
\usepackage{url}
\usepackage{booktabs}
\usepackage{fancyhdr}
\usepackage{cleveref}
\usepackage{ninecolors}
\usepackage{xspace}
\usepackage{soul}
\usepackage{tabularx}
\usepackage{nicefrac}
\usepackage{microtype}
\usepackage{algorithm}
\usepackage{algpseudocode}

\usepackage{tikz}
\usepackage{pgfplots} 

\usetikzlibrary{positioning,decorations.pathreplacing,calc}

\usepackage[dvipsnames]{xcolor}
\definecolor{dgreen}{rgb}{0.00,0.49,0.00}
\definecolor{dblue}{rgb}{0,0.08,0.75}
\hypersetup{
    colorlinks = true,
    citecolor=dgreen,
    urlcolor=dblue,
    linkcolor=dblue
}

\newcommand{\esssup}{\operatorname*{ess\,sup}}

\newcommand{\N}{\mathbb{N}}

\newcommand{\R}{\mathbb{R}}

\newcommand{\cX}{\mathcal{X}}

\newcommand{\cZ}{\mathcal{Z}}
\newcommand{\cH}{\mathcal{H}}

\newcommand{\cN}{\mathcal{N}}

\newcommand{\cG}{\mathcal{G}}

\def\sH{{\mathbb{H}}}

\renewcommand{\P}{\mathbb{P}}

\newcommand{\tod}{\rightsquigarrow}

\DeclareMathOperator{\Tr}{Tr}

\DeclareMathOperator*{\argmin}{arg\,min}
\newcommand{\ddllm}{\mathrm{d}}
\newcommand{\opllm}{\mathrm{op}}
\newcommand{\ipllm}[2]{\left\langle #1,#2\right\rangle}
\newcommand{\normllm}[1]{\left\lVert #1\right\rVert}

\newtheorem{definition}{Definition}

\newtheorem{prop}{Proposition}
\newtheorem{theorem}{Theorem}
\newtheorem{lemma}{Lemma}
\newtheorem{asst}{Assumption}
\newtheorem{corollary}{Corollary}
\newtheorem{rem}{Remark}

\newtheorem{requirement}{Requirement}

\newcommand{\asstpartsbreak}{\leavevmode\par}
\newlist{asstparts}{enumerate}{1}
\setlist[asstparts,1]{
    label=(\Alph*),
    ref=\theasst\Alph*,
    leftmargin=*,
    before=\asstpartsbreak
}
\newlist{reqparts}{enumerate}{1}
\setlist[reqparts,1]{
    label=(\Alph*),
    ref=\therequirement\Alph*,
    leftmargin=*,
    before=\asstpartsbreak
}

\crefname{asst}{assumption}{assumptions}
\crefname{asstpartsi}{assumption}{assumptions}
\Crefname{asstpartsi}{Assumption}{Assumptions}
\crefname{requirement}{requirement}{requirements}
\Crefname{requirement}{Requirement}{Requirements}
\crefname{reqpartsi}{requirement}{requirements}
\Crefname{reqpartsi}{Requirement}{Requirements}
\crefname{equation}{Eq.}{Eqs.}

\makeatletter
\DeclareRobustCommand\onedot{\futurelet\@let@token\@onedot}
\def\@onedot{\ifx\@let@token.\else.\null\fi\xspace}
\makeatother

\def\iid{i.i.d\onedot}

\renewcommand{\emph}[1]{\textbf{\textit{#1}}}

\newcommand{\eps}{\epsilon}

\newcommand{\lr}[1]{\left(#1\right)}
\newcommand{\lra}[1]{\left\langle #1 \right\rangle}

\newcommand{\Tt}{^\mathsf{\scriptscriptstyle T}} 
\newcommand{\inv}{^{-1}}

\def\1{\bm{1}}

\DeclareMathAlphabet{\mathsfit}{\encodingdefault}{\sfdefault}{m}{sl}
\SetMathAlphabet{\mathsfit}{bold}{\encodingdefault}{\sfdefault}{bx}{n}

\def\sH{{\mathbb{H}}}

\newcommand{\E}{\mathbb{E}}

\newcommand{\Var}{\mathrm{Var}}

\newcolumntype{Y}{>{\centering\arraybackslash}X}

\newcommand{\hsic}{\mathrm{HSIC}}

\newcommand{\blue}[1]{{\color{black}#1}}

\title{Semiparametrically Efficient Inference for Kernel Measures of Noise Heterogeneity}

\author[1]{Jakub Wornbard$^{*}$}
\author[1]{Zikai Shen$^{*}$}
\author[1]{Dimitri Meunier}
\author[1]{Arthur Gretton}

\affil[1]{University College London}

\date{}
\begin{document}

\maketitle
\setcounter{footnote}{0}
\renewcommand{\thefootnote}{\fnsymbol{footnote}} 
\footnotetext[1]{Equal contribution.}
\begin{abstract}
We develop semiparametrically efficient inference for kernel measures of noise heterogeneity in additive noise models. In many applications, the regression function is estimated using flexible machine learning methods. Downstream procedures based on the resulting residuals can then inherit first-stage bias: regression error may induce spurious dependence between covariates and residuals, invalidating the assumptions needed for standard analysis. We construct a novel Hilbert-valued one-step estimator of the kernel covariance operator between covariates and residuals. Our estimator yields bootstrap-calibrated tests for residual independence and goodness of fit in additive noise models, while also providing asymptotically efficient confidence intervals for the kernel dependence measure under noise heterogeneity. The framework extends to settings with additional covariates, enabling inference on distributional heterogeneity of residual noise across treatment groups. Simulations show improved calibration and power relative to naive plug-in residual methods.

\end{abstract}

\section{Introduction}
 \vspace{-0.2cm}
Many statistical and machine learning workflows use residuals to assess whether a fitted model has captured all systematic structure in the data. In additive noise models, this idea takes a particularly simple form: after regressing an outcome on covariates, the remaining noise should be independent of the variables whose effect has been explained. More broadly, one may ask whether the distribution of residual noise is homogeneous across covariates, treatment groups, or other features of interest. Kernel dependence measures provide a flexible way to quantify such residual noise heterogeneity, since they can detect general distributional differences beyond changes in mean or variance.

Classical nonparametric independence testing assumes that the variables to be tested are directly observed \citep{gretton2007kernel,heller2013consistent, berrett2019nonparametric, shekhar2023permutation}. In residual-based problems, however, the noise is not observed directly. It is generated by a first-stage regression estimator, often fit using flexible machine learning methods. This creates a multi-stage inference problem: downstream kernel procedures may inherit bias from the first-stage nuisance estimator. In this work, we bridge two strands of the statistical inference literature. On the one hand, recent work combines machine learning with semiparametric efficiency theory to estimate finite-dimensional target parameters \citep{chernozhukov2018double, hines2022demystifying, kennedy2024semiparametric}. On the other hand, reproducing kernel Hilbert space (RKHS)-based dependence measures have become standard tools for nonparametric dependence testing \citep{gretton2005measuring, gretton2007kernel,ChwSejGre14, zhang2018large,albert2022adaptive, shekhar2023permutation} and related tasks in machine learning \citep{song2007supervised, SheJegGre09, peters2013causal,li2021self}. This raises the question of whether semiparametric debiasing can be used to construct valid kernel inference procedures when the variables entering the kernel statistic are themselves estimated.

To make the exposition concrete, we begin with the additive noise model studied by \citet{peters2013causal} in the context of causal discovery. The data consists of $n$ i.i.d. pairs $(x_i,y_i)_{i=1}^n$ drawn from the same distribution as $(X,Y) \sim P$. The null hypothesis is that $P$ is correctly specified by an additive noise model: there exists a function $m$, whose functional form is unrestricted, and a centered noise variable $\epsilon$ independent of $X$ such that
\(
    Y = m(X) + \epsilon.
\)
The null is not a restriction on the functional form of the regression function; rather, it is the requirement that the regression residual be independent of the covariates.

If the true residuals $\epsilon_i=y_i-m(x_i)$ were observed, goodness of fit could be assessed by testing whether $X$ and $\epsilon$ are independent. In practice, $m$ is unknown, so an investigator estimates it by $\widehat m$ and constructs residuals $\widehat \epsilon_i = y_i-\widehat m(x_i)$. 
A natural plug-in approach is then to apply an off-the-shelf nonparametric independence test to the sample $\{(x_i,\widehat\epsilon_i)\}_{i=1}^n$.

We focus on kernel dependence measures based on the RKHS covariance operator between covariates and residuals. The squared Hilbert--Schmidt norm of this operator is the Hilbert--Schmidt Independence Criterion (HSIC), a standard kernel measure of dependence \citep{gretton2005measuring}. With characteristic kernels \citep{SriGreFukLanetal10,sriperumbudur2011universality}, this squared norm vanishes if and only if the covariate and residual noise are independent. Thus, in the oracle setting where $\epsilon$ is observed, HSIC provides a natural test statistic for additive-noise goodness of fit and, more generally, for residual noise heterogeneity.

The setting with plug-in estimator $\widehat{m}$ poses significant challenges. Estimated residuals satisfy $ \widehat \epsilon_i  = \epsilon_i + m(x_i)-\widehat m(x_i)$, 
so regression error can induce dependence between $X_i$ and $\widehat\epsilon_i$ even when the additive noise model is correctly specified. As a simple illustration of the challenges for hypothesis testing, suppose the data is split into two disjoint folds: one fold is used to estimate $m$, and the other is used to construct the independence test. If the fitting fold is small, $\widehat m$ may have substantial prediction error, and the kernel test may interpret this first-stage error as evidence of dependence, leading to inflated Type I error. If the fitting fold is large, the residuals better approximate the oracle residuals, but the holdout sample used for testing is small, reducing power. These observations highlight the practical difficulty of controlling Type I error while retaining power when nuisance functions are estimated. Related empirical failures have been observed in prior work on conditional independence testing \citep{pogodin2024practical, wang2026practical, he2026on}. Sample splitting removes the direct dependence between the fitted regression algorithm and the test sample, but conditional on the fitted function the residual still contains the systematic error \(m(X)-\widehat m(X)\), which is a function of \(X\).

The discussion above motivates a debiased estimator for kernel residual dependence. A standard route to debiasing is to use the efficient influence function of a target functional \citep{bickel1993efficient, kennedy2024semiparametric}. However, direct scalar debiasing of HSIC presents a complication: the empirical unbiased HSIC is a degenerate U-statistic under the null \citep{lee1990ustatistics}. At the population level, this first-order degeneracy appears as a vanishing efficient influence function: the first-order efficient influence function of
\[
    \mathrm{HSIC}(X,Y-\mathbb E[Y\mid X])
\]
is identically zero under the null. This is not a pathology of HSIC, but a consequence of its squared-norm structure: at independence, the underlying kernel covariance operator is zero, and the derivative of \(x\mapsto \|x\|^2\) vanishes at zero. Therefore, first-order scalar debiasing of HSIC cannot remove the first-stage bias that matters for testing near the null. The insufficiency of first-order linearization for Maximum Mean  Discrepancy based test statistic has been observed in \citet{luedtke2019omnibus, williamson2023general}. It has also been observed in a related setting by the concurrent work \citet{agarwal2026sinkhorn}. 

Our key idea is to debias before taking the squared norm. We construct a Hilbert-valued one-step estimator of the kernel covariance operator, building on recent work on Hilbert-valued semiparametric efficiency theory \citep{luedtke2024one, morzywolek2025inference}. We then obtain scalar kernel dependence measures by mapping the debiased RKHS-valued estimator to $\mathbb R$ through its squared Hilbert--Schmidt norm. This operator-level construction yields a calibrated test under the null, where scalar HSIC is degenerate, while also giving efficient confidence intervals for the value of the kernel dependence measure under alternatives.

\textbf{Contributions.}
We propose a Hilbert-valued one-step estimator for the kernel covariance operator between covariates and residuals, whose squared norm gives the corresponding kernel dependence measure. We prove asymptotic linearity and weak convergence of the estimator under nonparametric nuisance estimation.
These results yield a bootstrap-calibrated independence test with asymptotic validity under the null and consistency against fixed alternatives.
Under alternatives, our approach also provides asymptotically efficient confidence intervals for the value of the kernel dependence measure. We further extend the framework to settings with additional covariates, enabling covariate-adjusted inference on distributional heterogeneity of residual noise across treatment groups. Finally, we conduct simulations showing that the proposed method improves calibration and power relative to naive plug-in residual methods.

\textbf{Outline.} In \Cref{sec:setup}, we introduce the problem setup and the RKHS covariance operator
target. In \Cref{sec: estimator}, we derive the Hilbert-valued efficient influence function, define the one-step estimator, and establish asymptotic linearity, weak convergence, and
bootstrap calibration. \Cref{sec:experiments} presents synthetic experiments illustrating calibration, power, and confidence-interval coverage. \Cref{sec: discussion} discusses limitations and future directions.

\section{Problem Setup}
\label{sec:setup}
 \vspace{-0.2cm}
We study independence testing between random variables $X$ and $Y$, after adjusting for a covariate, denoted as $W$. We apply our method to two practical settings, namely goodness-of-fit testing for additive noise models and testing for noise homogeneity across different treatment groups, see \Cref{sec:experiments}. We refer to these as \emph{Settings 1} and \emph{2} respectively. In the former case, $X=W$, and our data consists of $(X,Y)$-pairs. In the latter case, $X$ is a treatment indicator, and $W = (X,T)$ where $T$ are additional covariates. We do not test for homogeneity with respect to the covariates $T$. For the remainder of the paper, we make the following assumption, although our theory readily generalizes to the case where $X$ is not measurable with respect to the covariate $W$ we adjust on. 
\begin{asst}
\label{asst:nested_covariates}
There exists a measurable map $\pi:\mathcal W\to\mathcal X$ such that
$X=\pi(W)$.
\end{asst}
Let $\mathcal O=\mathcal X\times\mathcal W\times\mathbb R^{d_y}$, and let
$\mathcal M$ denote a nonparametric model of probability laws on $\mathcal O$. We write $P_0\in\mathcal M$ for the true data-generating distribution. For $P\in\mathcal M$, define $m_P(w)\doteq\E_P[Y\mid W=w]$ and $\xi_P(w,y)\doteq y-m_P(w)$.
We test 
\(
    H_0:P_0\in \{P\in\mathcal M:\xi_P(W,Y)\perp X\}.
\)
In \emph{Setting 2}, testing for $H_0$ addresses the question of whether the marginal distribution of the residual is independent of $X$, after adjusting for $X$ and additional covariates $T$. 

We next introduce necessary RKHS concepts. Let $k:\cX\times\cX\to\mathbb R$ be a symmetric positive definite kernel and let $\cH_k$ be its reproducing kernel Hilbert space (RKHS) \citep{berlinet2011reproducing}. We denote the canonical feature map of $\cH_k$ as $\varphi_k(x)=k(\cdot,x)$. Let $l:\mathbb R^{d_y}\times\mathbb R^{d_y}\to\mathbb R$ be a symmetric positive definite kernel with RKHS $\cH_l$ and feature map $\varphi_l(u)=l(\cdot,u)$. We refer the reader to Appendix~\ref{sec:notations} for additional background on RKHSs. We assume all RKHSs mentioned are separable. This is satisfied, e.g. if the reproducing kernel is continuous and the domain is separable. 

Throughout, we assume the kernels are measurable and satisfy the moment conditions needed for the following Bochner expectations to exist; for example, the following assumption is sufficient.
\begin{asst}[Kernel boundedness]
	\label{asst:kernel_boundedness}
	$\sup_{x\in\mathcal X}k(x,x)\leq \kappa_k^2$ and $\sup_{u\in\mathbb R^{d_y}}l(u,u)\leq \kappa_l^2$.
\end{asst}

The dependence between $X$ and the residual $\xi_P(W,Y)$ is encoded by first featurizing them with infinite dimensional kernel features and then taking the expectation of their outer product. In particular, we use the language of kernel mean embeddings \citep{smola2007hilbert}. We denote the kernel mean embeddings of $X$ and $\xi_P(W,Y)$ as $\mu_{X,P}\doteq\mathbb E_P[\varphi_k(X)]\in\cH_k$ and $\mu_{\xi,P}\doteq\mathbb E_P[\varphi_{l}(\xi_P(W,Y))]\in\cH_l$. For characteristic kernels, testing for dependence between $X$ and $\xi_P(W,Y)$ nonparametrically is equivalent to testing whether the kernel mean embedding of the tensor product feature map $\mathbb{E}_P[\varphi_{k}(X)\otimes \varphi_{l}(\xi_P(W,Y))]$ is equal to  $\mu_{X,P} \otimes \mu_{\xi, P}$ \citep{gretton2015simpler}. This is algebraically equivalent to testing for $\Psi(P) = 0$, where
\begin{align*}
    \Psi(P) = \mathbb{E}_P \left[ \widetilde\varphi_{k,P}(X)\otimes \widetilde{\varphi}_{l,P}(\xi_P(W,Y))\right] \in  \cH_{kl},
\end{align*}
where we define the \textit{centered} feature maps $\widetilde{\varphi}_{k,P}(x) \doteq \varphi_k(x) - \mu_{X,P}$ and $\widetilde{\varphi}_{l,P}(\xi_P(w,y)) \doteq \varphi_{l}(\xi_P(w,y)) - \mu_{\xi, P}$. $\cH_{kl}\doteq\cH_k\otimes\cH_l$ denotes the tensor product RKHS induced by the product kernel $(x,u),(x',u')\mapsto k(x,x')l(u,u')$. 
$\|\Psi(P)\|^2_{\mathcal H_{kl}}$ is the Hilbert--Schmidt Independence Criterion
(HSIC) between $X$ and $Y-m_P(W)$ \citep{gretton2005measuring}. We define the  hypothesis test
\[
    H_0:\Psi(P_0)=0
    \qquad
    \text{against}
    \qquad
    H_1:\Psi(P_0)\neq 0.
\]
Throughout, a subscript $0$ denotes evaluation at the true distribution $P_0$; for example, $m_0:=m_{P_0}$ and $\Psi_0:=\Psi(P_0)$.

\section{One-step Estimator}
\label{sec: estimator}
 \vspace{-0.2cm}
The key challenge is that regression error can be indistinguishable from dependence between the estimated residual and the covariate \(X\) in the downstream hypothesis test.
Semiparametric efficiency theory offers an attractive possibility of debiased two-stage learning, by computing an efficient influence function (EIF) and then removing the first order bias of the first stage estimator. However, the scalar-valued HSIC is not regular at the null $H_0$ as its EIF at the null is identically zero.
The correct object to debias is thus the Hilbert-valued cross-covariance operator $\Psi(P)$, and we should compute the squared norm \emph{after} debiasing. We now derive the EIF of the cross-covariance.
Let $L^2_0(P_0)$ denote the space of square-integrable functions with zero mean under $P_0$.

\begin{asst}[Model, moments, and differentiability] \label{asst:asst_1}
	Let $\mathcal M$ be a nonparametric model with tangent space
	$\dot{\mathcal M}_{P_0}=L^2_0(P_0)$. Assume that
	$\mathbb E_0\|Y\|^2<\infty$, and that $u\mapsto \varphi_l(u)$ is continuously Fréchet differentiable as a map
	$\mathbb R^{d_y}\to\mathcal H_l$, with
	$
		D\varphi_l(u)[a] = \sum_{j=1}^{d_y} a_j\,\partial_j\varphi_l(u), a\in\mathbb R^{d_y}.
	$
	Moreover, for each $j$, we assume that  $\sup_{u\in\mathbb R^{d_y}}\|\partial_j\varphi_l(u)\|_{\mathcal H_l}<\infty$.
\end{asst}

\textbf{Intuition about the form of the EIF.}
The correction term in the EIF can be understood from a first-order expansion of the
residual feature map. Let
\(
    \widehat\delta(w) \doteq \widehat m(w)-m_0(w).
\)
Then
\[
    \varphi_l\{Y-\widehat m(W)\}
    =
    \varphi_l\{\xi_0(W,Y)-\widehat\delta(W)\}
    \approx
    \varphi_l\{\xi_0(W,Y)\}
    -
    \sum_{j=1}^{d_y}
    \widehat\delta_j(W)
    \partial_j\varphi_l\{\xi_0(W,Y)\}.
\]
Therefore, conditional on \(W\), the
first-order shift in the residual feature is
\[
 \text{first-order shift}=
    -
    \sum_{j=1}^{d_y}
    \widehat\delta_j(W)
    v_{0,j}(W),
    \qquad
    v_{0,j}(w)
    \doteq
    \mathbb E_0[
        \partial_j\varphi_l\{\xi_0(W,Y)\}
        \mid W=w
    ].
\]
This is the component of the plug-in error that can create spurious dependence with
\(X\). The EIF subtracts
\(
    \sum_{j=1}^{d_y}\xi_{0,j}(W,Y)v_{0,j}(W).
\)
For the plug-in residual \(\widehat\xi_j=\xi_{0,j}-\widehat\delta_j(W)\), this correction
changes by
\(
    +
    \sum_{j=1}^{d_y}
    \widehat\delta_j(W)v_{0,j}(W),
\)
which cancels the conditional first-order Taylor bias above. Equivalently, after the
correction, the leading perturbation is proportional to
\(
    \sum_{j=1}^{d_y}
    \widehat\delta_j(W)
    \left[
        \partial_j\varphi_l\{\xi_0(W,Y)\}
        -
        v_{0,j}(W)
    \right],
\)
whose conditional mean given \(W\) is zero. Thus the first-order effect of regression
error is removed, leaving only higher-order remainder terms.

\begin{theorem}[Efficient influence function]
	\label{thm:cco_eif}
	Suppose \Cref{asst:nested_covariates,asst:asst_1,asst:kernel_boundedness} hold. Then
	$\Psi:\mathcal M\to\mathcal H_{kl}$ is pathwise differentiable at $P_0$. For any regular parametric submodel $t\mapsto P_t$ through $P_0$ with score $s\in L^2_0(P_0)$, $\dot\Psi_0(s) \doteq \lim_{t\to 0}\frac{1}{t}(\Psi(P_t)-\Psi(P_0))=
		\mathbb E_0[\phi_0(O)s(O)] \in \mathcal H_{kl}$,
	where for $o=(x,w,y)$, $\phi_0(o) \doteq \widetilde\varphi_{k,0}(x)\otimes \bigg[\widetilde\varphi_{l,0}(\xi_0(w,y)) - \sum_{j=1}^{d_y} \xi_{0,j}(w,y)v_{0,j}(w)\bigg] - \Psi_0$. $\phi_0$ belongs to $L^2_0(P_0;\mathcal H_{kl})$, therefore it is the $\mathcal H_{kl}$-valued EIF of $\Psi$ at $P_0$ in the sense of \citet{luedtke2024one}.
\end{theorem}
A slightly more general form of \Cref{thm:cco_eif} is proved in Appendix~\ref{sec:EIF_appendix}. 

\subsection{One-step estimator and asymptotic normality}
\label{subsec: one-step}
We now use the EIF to perform  one-step debiasing for $\Psi_0$. We refer to the debiased estimator as a one-step estimator
\citep{schick1986asymptotically, luedtke2024one}. We will establish Hilbert-valued asymptotic normality of this estimator under a set of
reasonable assumptions. We may use the asymptotic normality theorem to construct a confidence interval for any scalar-valued transformation of
$\Psi_0$. We choose the RKHS norm of $\Psi_0$, which is sufficient for testing for whether $\Psi_0=0$. A hypothesis test
with Type I error control directly leads to a confidence interval for the true value $\Psi_0$. We refer to such a test as calibrated. By including
all values of $\Psi_0$ not rejected by a test with Type I error control, one obtains a confidence set with the desired coverage level. We leverage
this duality to simultaneously obtain a confidence interval for $\|\Psi_0\|_{\cH_{kl}}$ and a hypothesis test for $H_0$.

We assume the dataset is $\{o_i = (x_i,w_i,y_i)\}_{i=1}^n$, where $n$ is even. Let $I_1=\{1,\dots,n/2\}$ and $I_2=\{n/2+1,\dots,n\}$. We fit the
nuisance functions on $I_2$ and evaluate the one-step estimator on $I_1$. Write $\widehat\xi_i\doteq y_i-\widehat m(w_i)$, where $\widehat m$ is
\textit{any} machine learning estimator for $m_0$. Define the empirical kernel mean embeddings $\widehat{\mu}_X \doteq \frac{2}{n}\sum_{i\in
		I_2}\varphi_k(x_i)$ and $\widehat{\mu}_{\xi} \doteq \frac{2}{n}\sum_{i\in I_2}\varphi_l(\widehat\xi_i)$.

For $j=1,\dots,d_y$, we estimate the $\mathcal H_{l}$-valued nuisance $v_{0,j}$ by vector-valued kernel ridge regression
\citep{caponnetto2007optimal,park2020measure,li2024towards}. Let $\mathcal H_{ql}$ denote the vector-valued RKHS induced by the kernel
$q(\cdot,\cdot)\mathrm{Id}_{\mathcal H_{l}}$. For details on vector-valued kernels and their associated RKHS, we refer the reader to
\citet{li2024towards} and Appendix~\ref{app:statistical_learning_theory}. We define, for each $j=1,\dots,d_y$,
\begin{align*}
	\widehat v_j =\argmin_{v \in\mathcal H_{ql}} \frac{2}{n}\sum_{i\in I_2} \left\|(\partial_j\varphi_l)(\widehat\xi_i) - v(w_i)\right\|_{\mathcal H_{l}}^2 + \lambda_{j,n}\|v\|_{\mathcal H_{ql}}^2.
\end{align*}
For $i\in I_1$, define $\widehat M_i \doteq \varphi_l(\widehat\xi_i)-\widehat\mu_\xi -\sum_{j=1}^{d_y}\widehat\xi_{i,j}\widehat v_j(w_i)$.
The one-step estimator evaluated on $I_1$ is
\begin{equation} \label{eq: two_fold_os}
	\widehat\Psi_n \doteq  \frac{2}{n}\sum_{i\in I_1}
	\left(\varphi_k(x_i)-\widehat\mu_X\right)\otimes\widehat M_i.
\end{equation}

\begin{rem}[Cross-fitting]
\label{rem:cross_fitting}
	To simplify our exposition and notations, in the main document we work exclusively with a sample split estimator $\widehat\Psi_n$, introduced above in \Cref{eq: two_fold_os}. Since $\widehat{\Psi}_n$ is constructed by evaluating half the samples, it leads to a loss of statistical power. To remedy this, we construct a $K$-fold cross-fitted estimator, where we use all $K$ folds except one for fitting, one fold for evaluation, and average over all $K$ choices of the evaluation fold. We emphasize that all simulation experiments and theoretical guarantees in our work apply to the $K$-fold cross-fitted estimator, as in done in the Appendix. The cross-fitted estimator allows us remove restrictive Donsker-type conditions on the function class \citep{schick1986asymptotically, chernozhukov2018double}. We provide generalized results for the cross-fitted estimator in the Appendix, denoted with \emph{reverse hat} notation as $\check{\Psi}_n$. We develop the core asymptotic results for those in \Cref{sec:EIF_appendix} assuming $K=2$. Obtaining analogues for higher values is not harder but requires more careful notation. This simplification is standard in the literature \citep{chernozhukov2018double}
\end{rem}

\begin{asst}[Main kernel differentiability and Lipschitz continuity assumption]\label{asst: differentiability_kernel_y}
	We assume that the kernel $l:\mathbb R^{d_y}\times\mathbb R^{d_y}\to\mathbb R$ is such that, for all $j\in[d_y]$, the mixed partial derivative $\partial_j\partial_{j+d_y}l$ exists and is continuous, where the derivative is taken with respect to the $j$th coordinate of the first argument and the $j$th coordinate of the second argument. We further assume that there exists a global constant $L>0$ such that $\|\varphi_l(u)-\varphi_l(v)\|_{\mathcal H_l}\leq L\|u-v\|_{\mathbb R^{d_y}}$ for all $u,v\in\mathbb R^{d_y}$. We assume that $\varphi_l:\mathbb R^{d_y}\to\mathcal H_l$ is twice Fréchet differentiable and that its second derivative is uniformly bounded: there exists $L_2<\infty$ such that $\|D^2\varphi_l(u)\|_{\mathrm{op}}\leq L_2$ for all $u\in\mathbb R^{d_y}$.
\end{asst}
We provide an assumption stated solely in terms of partial derivatives of $l$ that implies \Cref{asst: differentiability_kernel_y} in \Cref{ass:regularity_l_appendix} in the Appendix. This allows us to show that \Cref{asst: differentiability_kernel_y} is satisfied by popular kernels, including Gaussian kernels and sufficiently smooth Matérn kernels. By \citet[Lemma 4.34, Corollary 4.36]{steinwart2008support}, \Cref{asst: differentiability_kernel_y} in particular implies that, for each $j=1,\dots,d_y$, the partial derivative $\partial_j\varphi_l:\mathbb R^{d_y}\to\mathcal H_l$ exists and is continuous, and for all $y,y'\in\mathbb R^{d_y}$ we have $ \langle\partial_j\varphi_l(y),\partial_j\varphi_l(y')\rangle_{\mathcal H_l} =
	\partial_j\partial_{j+d_y}l(y,y')$.

\begin{requirement}[Nuisance-level requirements]
	\label{asst: nuisance_level}
	\begin{reqparts}
		\item\label{asst: boundedness_reg_fn}
		There exists a constant $C>0$ such that $\|m_0\|_{L^\infty(P_{0,W};\mathbb R^{d_y})}\leq C$ and 		$\esssup_{w\in\mathrm{supp}(P_{0,W})}\mathbb E_0[\|Y-m_0(W)\|_{\mathbb R^{d_y}}^2\mid W=w]\leq C.$

		\item\label{asst: consistency}
		We have $\|\widehat m-m_0\|_{L^\infty(P_{0,W};\mathbb R^{d_y})}=o_p(1)$, $\|\widehat m-m_0\|_{L^2(P_{0,W};\mathbb R^{d_y})}=o_p(1)$, $\|\widehat{v}-~v_0\|_{L^2(P_{0,W};\mathcal H_{l}^{d_y})}=o_p(1)$, and $\|\widehat\mu_{\xi}-\mu_{\xi,0}\|_{\mathcal H_l} = o_p(1)$.
		\item\label{asst: suff_cond_an}
		We have $\|\widehat m-m_0\|^2_{L^2(P_{0,W};\mathbb R^{d_y})}
			+\|\widehat m-m_0\|_{L^2(P_{0,W};\mathbb R^{d_y})}\|\widehat v-v_0\|_{L^2(P_{0,W};\mathcal H_{l}^{d_y})} = o_p(n^{-1/2})$.
	\end{reqparts}
\end{requirement}
\begin{theorem}[Asymptotic linearity and weak convergence of the sample-split one-step estimator]
	\label{thm: asymp_linearity_main}
	Suppose the conclusion of \Cref{thm:cco_eif} holds and let $\phi_0$ be given as in \Cref{thm:cco_eif}. Under \Cref{asst: differentiability_kernel_y} and \Cref{asst: nuisance_level}, the sample-split one-step estimator $\widehat \Psi_n$ is regular and asymptotically linear:
	\begin{align}
		\label{eq:hilbert_valued_clt_main}
		\sqrt{n/2}\left(\widehat \Psi_n-\Psi_0\right) = \sqrt{\frac{2}{n}}\sum_{i \in I_1}\phi_0(o_i)+o_p(1)
		\rightsquigarrow \mathbb H,
	\end{align}
	where $\mathbb H$ is a tight $\cH_{kl}$-valued Gaussian random variable such that, for each $h\in\cH_{kl}$, $\langle\mathbb H,h\rangle_{\cH_{kl}}\sim\mathcal N(0,\mathbb E_0[\langle\phi_0(X,W,Y),h\rangle_{\cH_{kl}}^2])$.
\end{theorem}

\textbf{When is \Cref{asst: nuisance_level} satisfied?} \Cref{asst: consistency,asst: suff_cond_an} are a set of sufficient conditions that hold under more fundamental assumptions, such as membership in a Hölder smoothness class.  $\widehat{m}$ is a standard nonparametric regression estimator for $m_0$. Its minimax rate of convergence is well understood \citep{tsybakov2004introduction}. Both $\widehat{v}$ and $\widehat{\mu}_{\xi}$ depend on the nonparametric first stage $\widehat{m}$. In the second stage, $\widehat{v}$ is obtained from performing a vector-valued kernel ridge regression, while $\widehat{\mu}_{\xi}$ is obtained from computing a sample average. In the former case, the convergence rates depend on the smoothness of the vector-valued target $v_{0}$ and the capacity of the vector-valued RKHS $\cH_{ql}$. In the latter case, the second stage error is a Monte Carlo error. We emphasize that thanks to the second order nature of the nuisance rates in \Cref{asst: suff_cond_an}, nowhere do we require a nuisance function to converge at parametric rate $n^{-\frac{1}{2}}$. In Appendix~\ref{app:statistical_learning_theory}, we provide the necessary background on vector-valued kernel learning, state and prove technical theorems on the convergence rates of $\widehat{v}$ and $\widehat{\mu}_{\xi}$, and provide a concrete example  where \Cref{asst: nuisance_level} is satisfied.

\vspace{-0.2cm} 
\subsection{Debiased confidence set and hypothesis test}\label{sec:main_hyp_test}
\vspace{-0.2cm} 
We now  construct a
confidence interval estimator for $\|\Psi_0\|_{\cH_{kl}}$, and then convert it into a hypothesis test for the null $H_0$ that $\Psi_0 = 0$ with desired Type I error control. We introduce two
methods, both of which rely on the asymptotic normality result of \Cref{thm: asymp_linearity_main}. By the continuous mapping theorem applied to
\Cref{eq:hilbert_valued_clt_main}, we deduce that $(n/2)\|\widehat\Psi_n-\Psi_0\|_{\mathcal H_{kl}}^2\rightsquigarrow\|\mathbb H\|_{\mathcal
	H_{kl}}^2$. 
    We can approximate $\|\mathbb{H}\|_{\cH_{kl}}^2$ by the
bootstrapped distribution of $\|\widehat{\Psi}_n - \widehat{\Psi}_n^{(b)}\|_{\cH_{kl}}^2$ \citep{efron1992bootstrap}. We  briefly sketch our procedure, and refer the reader
to Appendix~\ref{sec:inference} for a formal exposition. For $b\in [B]$, we let $o_{1}^{(b)}, \dots, o_{\frac{n}{2}}^{(b)} \overset{\text{i.i.d.}}{\sim}
	\frac{2}{n}\sum_{i\in I_1}\delta_{o_i}$, conditionally on the original data. Let $\zeta_n$ denote the $(1-\alpha)$-quantile of $\left\|\sqrt{n/2}
	(\widehat{\Psi}_n - \widehat{\Psi}_n^{(b)})\right\|_{\cH_{kl}}^2$, where $\widehat{\Psi}_n^{(b)}$ is given as in \Cref{eq: two_fold_os} except
evaluated on the $b$th bootstrap sample. Then $\zeta_n$ converges to the $(1-\alpha)$-quantile $\zeta_{1-\alpha}$ of
$\|\mathbb{H}\|^2_{\mathcal{H}_{kl}}$, provided the estimated EIF is a consistent estimator for the true EIF. We refer the reader to \citet[Theorem
	4]{luedtke2024one} for the proof. In simulations, we replace \(\zeta_n\) by its finite-\(B\) approximation \(\widehat\zeta_n\).

Armed with a consistent estimate ${\zeta}_n$ for the $(1-\alpha)$-quantile of $\|\mathbb{H}\|^2_{\mathcal{H}_{kl}}$, we propose the following
confidence interval for $\|\Psi_0\|_{\cH_{kl}}$:
\begin{align}
	\label{eq: curlyI_zetan}
	\mathcal{I}_{\zeta_n} \doteq  \left[\left\{\|\widehat\Psi_n\|_{\mathcal H_{kl}}-\sqrt{2{\zeta}_n/n}\right\}_+,\ \|\widehat\Psi_n\|_{\mathcal H_{kl}}+\sqrt{2{\zeta}_n/n}\right].
\end{align}
It is a $(1-\alpha)$-confidence interval for $\|\Psi_0\|_{\mathcal H_{kl}}$, due to the reverse triangle inequality $\left|\|\widehat\Psi_n\|_{\mathcal H_{kl}}-\|\Psi_0\|_{\mathcal H_{kl}}\right|\leq \|\widehat\Psi_n-\Psi_0\|_{\mathcal H_{kl}}$. It is asymptotically exact at the null in the sense that, when $\Psi_0=0$, the reverse triangle inequality becomes an equality. We obtain a calibrated hypothesis test for $H_0$ by rejecting if and only if $0 \not\in \mathcal{I}_{\zeta_n}$. In general, we both expect and observe in simulations that \Cref{eq: curlyI_zetan} behaves conservatively when the null does not hold. We emphasize the key role of our main theorem \Cref{thm: asymp_linearity_main} in the validity of this construction.

However, when $\Psi_0 \neq 0$, we may construct a confidence interval that performs even better in terms of asymptotic variance. We observe the following decomposition
\begin{align}
	\sqrt{n/2}\left(\|\widehat\Psi_n\|^2_{\cH_{kl}} - \|\Psi_0\|^2_{\cH_{kl}}\right) & =\sqrt{n/2}\left\langle \widehat\Psi_n - \Psi_0, \widehat\Psi_n + \Psi_0\right\rangle_{\cH_{kl}}\nonumber                                              \\
	                                                                                 & = \sqrt{n/2}\|\widehat\Psi_n - \Psi_0\|^2_{\cH_{kl}} + \sqrt{2n}\langle \widehat\Psi_n - \Psi_0, \Psi_0\rangle_{\cH_{kl}}.\label{eq:delta_decomp_main}
\end{align}
For fixed $\Psi_0 \neq 0$, the second term  is asymptotically normal with mean zero with asymptotic variance $\mathbb E_0[4\langle\phi_0(X,W,Y),\Psi_0\rangle_{\cH_{kl}}^2]$, whereas the first term is $O_p(1/\sqrt{n})$. As a result, we have a CLT for $\|\widehat\Psi_n\|^2_{\cH_{kl}}$, namely  $\sqrt{n/2}\left(\|\widehat\Psi_n\|^2_{\cH_{kl}} - \|\Psi_0\|^2_{\cH_{kl}}\right) \rightsquigarrow \mathcal{N}(0, \mathbb E_0[4\langle\phi_0(X,W,Y),\Psi_0\rangle_{\cH_{kl}}^2])$. Since $\widehat{\Psi}_n$ is the one-step estimator for $\Psi_0$, our method is an example of functional delta method.  The fact that this vanishes at $\Psi_0 = 0$ is precisely what prevents us from debiasing the scalar-valued HSIC directly. We define the following confidence interval
\begin{align*}
	\mathcal{I}_{\delta}^{1-\alpha} \doteq \left[\|\widehat\Psi_n\|_{\mathcal H_{kl}}^2-2z_{1-\alpha/2}\sqrt{\frac{2\widehat\sigma_n^2}{n}},\;\|\widehat\Psi_n\|_{\mathcal H_{kl}}^2+2z_{1-\alpha/2}\sqrt{\frac{2\widehat\sigma_n^2}{n}}
	\right],
\end{align*}
where $z_{1-\alpha/2}$ is the $(1-\alpha/2)$-quantile of a standard normal distribution, and $\widehat{\sigma}_n$ is an estimator for $\mathbb E_0[\langle\phi_0(X,W,Y),\Psi_0\rangle_{\cH_{kl}}^2]$. From \citet[Theorem 25.47]{van2000asymptotic}, we know that $\mathbb E_0[4\langle\phi_0(X,W,Y),\Psi_0\rangle_{\cH_{kl}}^2]$ is the semiparametric efficiency bound for estimating $\|\Psi_0\|^2_{\cH_{kl}}$ with $\frac{n}{2}$ samples, which suggests that when $\Psi_0\neq 0$ holds, then $\mathcal{I}_{\delta}^{1-\alpha}$ may be much less wide than $\mathcal{I}_{\zeta_n}$ while  being asymptotically calibrated. We confirm this empirically in \Cref{sec:experiments}. In the Appendix, we construct such an estimator in \Cref{eq:sigma_hat_estimator_sq}, and establish its consistency in \Cref{prop: sigma_consistency}.

As we have access to an asymptotically always valid confidence interval, we can adopt the following approach. We define a data-adaptive union confidence interval
that is $\mathcal{I}_{\delta}^{1-\alpha}$ augmented with the singleton set $\{0\}$ if $0\in \mathcal{I}_{\zeta_n}$, for $\mathcal{I}_{\zeta_n}$ defined in
Eq.~\eqref{eq: curlyI_zetan} (see \citet[Theorem 4]{morzywolek2025inference} for a similar approach). We would nevertheless like to understand the
validity of $\mathcal{I}_{\delta}^{1-\alpha}$ in finite samples, where the sample size $n$ is small and the true $\Psi_0$ is close to zero. This is
motivated by the asymptotics where we consider a sequence of experiments where $\Psi_n \asymp \frac{1}{\sqrt{n}}$. Along this sequence, the two terms
in \Cref{eq:delta_decomp_main} are of the same order of magnitude. Therefore a confidence interval based on the asymptotic variance of the second
term may not be trustworthy. We propose a lightweight data-driven diagnostic that determines when it is plausible that $2\sqrt{n/2}\langle \Psi_0,
	\widehat \Psi_n-\Psi_0\rangle_{\cH_{kl}} \gg_p \sqrt{n/2}\|\widehat \Psi_n - \Psi_0\|^2_{\cH_{kl}}$. The details of our construction can be found in
Appendix~\ref{sec:finite_sample_delta_validity}.

\vspace{-0.2cm}
\section{Experiments} \label{sec:experiments}
\vspace{-0.2cm}
We verify the validity of the proposed procedures on a range of synthetic experiments, where we are able to explicitly control the extent to which noise homogeneity is violated, and to adjust the
statistical difficulty of all the intermediate procedures.

\textbf{Baselines.}

\setlength{\intextsep}{2pt}
\setlength{\columnsep}{6pt}

\begin{wrapfigure}[16]{r}{0.28\textwidth}
    \vspace{-10pt}
    \centering
    \includegraphics[
        width=0.94\linewidth,
        trim={0mm 0mm 0mm 6mm},
        clip
    ]{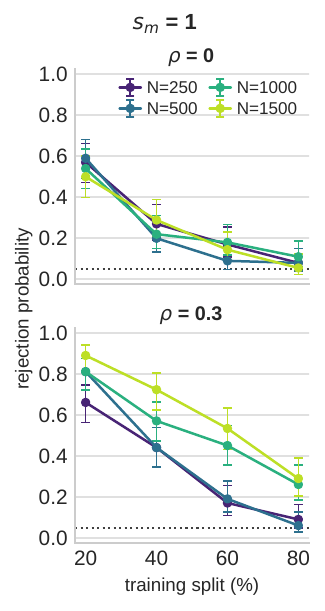}
    \captionsetup{font=small,skip=1pt}
    \caption{Split fit-test baseline at $s_m=1,s_\eps=0.25$.}
    \label{fig:x_dependent_residuals_split_test_rejection_probability}
    \vspace{-12pt}
\end{wrapfigure}

We compare against plug-in residual procedures that estimate $m$ and then apply standard HSIC inference without the EIF correction.
The split-sample baseline follows the ANM diagnostic of \citet{peters2013causal}: fit $\hat m$ on a training split, compute residuals $Y_i-\hat m(W_i)$ on held-out observations, and run the permutation HSIC test \citep{gretton2007kernel}.
We report several training fractions because the split controls a calibration--power tradeoff.

We also consider a full-sample, cross-fitted plug-in permutation test.
Let $I_1,\ldots, I_K$ be the folds, and write $\hat m^{(-k)}$ for the regression estimator trained on all observations outside fold $\mathcal I_k$.
For $i\in\mathcal I_k$, define the out-of-fold residual $\hat\xi_i=Y_i-\hat m^{(-k)}(W_i)$, form Gram matrices $K_{ij}=k(X_i,X_j)$ and $L_{ij}=l(\hat\xi_i,\hat\xi_j)$, 
and perform a full-sample permutation test using these.
This uses all observations and avoids in-sample residuals, but remains heuristic because the permutation null does not account for first-stage estimation error.

For confidence intervals, we also compare against plug-in methods, which use the same bootstrap or delta-method construction as our estimator, but with the EIF correction term omitted. This helps us evaluate the effect of bias correction on the performance.

\vspace{-0.2cm}
\paragraph{\texorpdfstring{$X$}{X}-dependent residuals.}We first evaluate goodness-of-fit testing for the scalar additive-noise model
\(
    Y=m(X)+\epsilon,\,\epsilon\perp X .
\)
Data are generated from
\[
    X\sim\mathrm{Unif}[-3,3],\qquad
    Y=m(X)+\sigma(X)Z,\qquad
    Z\sim\mathcal N(0,1).
\]
Both $m$ and $\log\sigma$ are sampled as truncated random Fourier series.
The decay parameters $s_m$ and $s_\epsilon$ control smoothness, and hence first-stage estimation difficulty: $s_m$ is the exponent in the power-law decay rate of the Fourier coefficients of $m$, and $s_\epsilon$ plays an analogous role in $\log\sigma(X)$. $\rho$ controls the strength of residual heterogeneity. Smaller values of those lead to rougher regression functions and less regular variance structure in the noise. $\rho$ multiplies the non-constant $X$-dependent term in $\log\sigma(X)$: $\rho=0$ gives homogeneous noise and larger $\rho$ yields stronger $X$-dependence in the residual distribution.
In each setting, we estimate $\sqrt{\hsic(X,\epsilon)}$ using Gaussian kernels with bandwidths chosen by the median heuristic.
We keep $d_x=d_y=1$ in the reported experiments to isolate first-stage residual-estimation bias from the additional effects of high-dimensional kernel testing and nuisance estimation.
The exact data-generating mechanism is given in Appendix~\ref{sec:experiment_details_synthetic_data}.

\Cref{fig:x_dependent_residuals_rejection_probability} reports rejection curves over $\rho$ for a representative setting.
At $\rho=0$, the debiased bootstrap test is close to nominal, while the cross-fitted permutation baseline is conservative.
As $\rho$ increases, the debiased test gains power without requiring a split-ratio choice. So does the baseline,
though it attains lower power. \Cref{fig:x_dependent_residuals_split_test_rejection_probability} shows the split-sample plug-in baseline for the same synthetic residual experiment.
Under the null, rejection should be near the nominal $5\%$ line, but small training splits leave enough residual-estimation bias to inflate rejection.
Under the alternative, larger training splits fit the regression better but leave fewer observations for the HSIC test, so power can fall sharply.
This calibration--power tradeoff is the tuning problem that \Cref{fig:x_dependent_residuals_rejection_probability} avoids: the debiased test uses the full sample and gains power without requiring a split-ratio choice.

\par\vspace{-\parskip}
\noindent\begin{minipage}{\linewidth}
    \centering
    \includegraphics[width=0.99\linewidth,trim={2mm 0mm 0mm 7mm},clip]{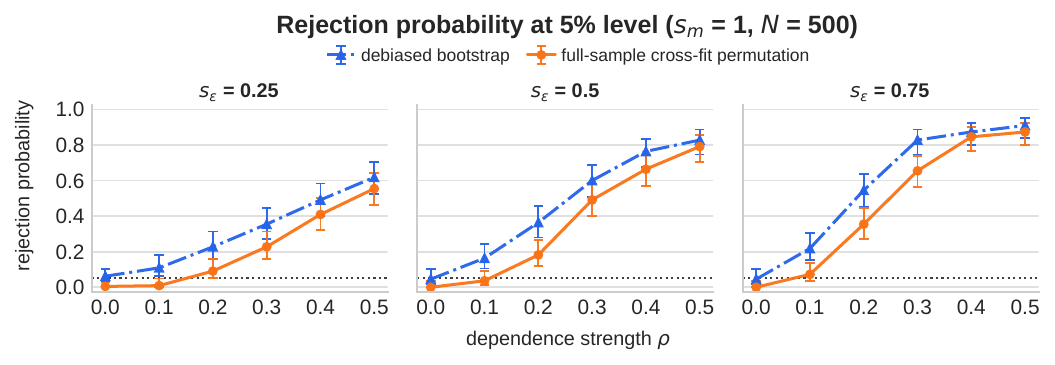}
    \captionsetup{font=small,skip=1pt,hypcap=false}
    \captionof{figure}{Rejection probabilities at $s_m=1$ and $n=500$ across noise-dependence levels and smoothness parameters $s_\eps$ of $\sigma(x)$, for the debiased CI of \Cref{eq: curlyI_zetan} and the cross-fitted permutation baseline.}
    \label{fig:x_dependent_residuals_rejection_probability}
\end{minipage}
\par\vspace{-\parskip}

\paragraph{Residual heterogeneity with additional covariates.}
\begin{wrapfigure}[19]{r}{0.28\textwidth}
    \centering
    \includegraphics[width=0.94\linewidth,trim=0mm 0mm 0mm 0mm,clip]{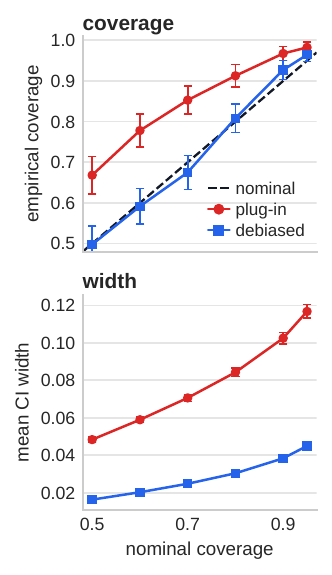}
    \captionsetup{font=small,skip=1pt}
    \caption{Delta-method CI calibration for $\Psi_0$ when $W=(X,T)$.}
    \label{fig:covariate_confidence_set_coverage}
\end{wrapfigure}
We next consider a setting where the magnitude of $\Psi_0$, rather than only the null
$\Psi_0=0$, is scientifically meaningful. Let $W=(X,T)$, where $X\in\{0,1\}$ is a binary group
indicator and $T$ is an additional covariate. Writing
$
    R = Y-\E[Y\mid X,T],
$
the target measures whether the covariate-adjusted residual law differs across the two groups.

This connection is especially transparent with the discrete kernel
$k(x,x')=\1\{x=x'\}$ on $X$. When the two groups are balanced,
\[
    \hsic(X,R)
    =
    \frac18
    \mathrm{MMD}_l^2\!\left(
        \mathcal L(R\mid X=0),
        \mathcal L(R\mid X=1)
    \right),
\]
is the Maximum Mean Discrepancy \citep{gretton2007kernel} associated with the residual kernel $l$, between the laws of the residuals under $X=0$ and under $X=1$. We evaluate this target in a synthetic model
\[
    Y=m(X,T)+\sigma_X Z,\qquad Z\sim\mathcal N(0,1),
\]
where $m$ is a finite trigonometric series in $T$, and $l$ is Gaussian. The null has equal residual
scales across groups, while the alternative has $\sigma_0\neq\sigma_1$. Full data-generating details
are given in Appendix~\ref{sec:experiment_details_covariate_diagnostics}.

\Cref{fig:covariate_confidence_set_coverage} confirms that confidence intervals based on plug-in $\Psi$ are highly suboptimal in this setting. They are very conservative and too wide, while the delta-method
CIs track their nominal coverage near-perfectly and are substantially narrower.
\paragraph{Causal discovery by two-way goodness-of-fit testing.} Additive-noise causal discovery compares the two regression directions by fitting
\[
    Y=m_{Y|X}(X)+\epsilon_{Y|X}
    \qquad\text{and}\qquad
    X=m_{X|Y}(Y)+\epsilon_{X|Y},
\]
and testing whether the fitted residuals are independent of the corresponding covariate
\citep{peters2013causal}. In the ideal additive-noise setting, the causal direction should pass this
goodness-of-fit test, while the reverse direction in general should fail, except for special non-identifiable
 cases such as linear $m$ and Gaussian noise. 
This requires calibration under a fitted-residual null and power against a non-additive reverse direction.

We study a synthetic bivariate example with known structural direction $X\to Y$. The structural
function is nonlinear and resembles the low-dimensional cause-effect pairs used in the T\"ubingen
cause-effect pair benchmark \citep{mooij2016distinguishing}. We consider two regimes. In the first,
the forward model has homogeneous additive noise, so the ANM null holds for the regression of
$Y$ on $X$ and fails in the reverse direction. In the second, the structural direction remains
$X\to Y$, but the forward residual law depends on $X$; in this case both ANM goodness-of-fit
nulls should be rejected. Full data-generating details and example datasets are given in
Appendix~\ref{sec:experiment_details_causal_arrow_violin}. 
\Cref{fig:causal_arrow_rejection_probability} shows that the debiased bootstrap gives the most favourable two-way goodness-of-fit profile, and higlights the failure mode of split-sample tests. The latter only become calibrated under the null when 80\% of the sample is used for fitting, but the remaining small testing sample leads to very low power.
\par\vspace{-\parskip}
\noindent\begin{minipage}{\linewidth}
    \centering
    \includegraphics[width=0.99\linewidth,trim=2mm 2mm 0mm 7mm,clip]{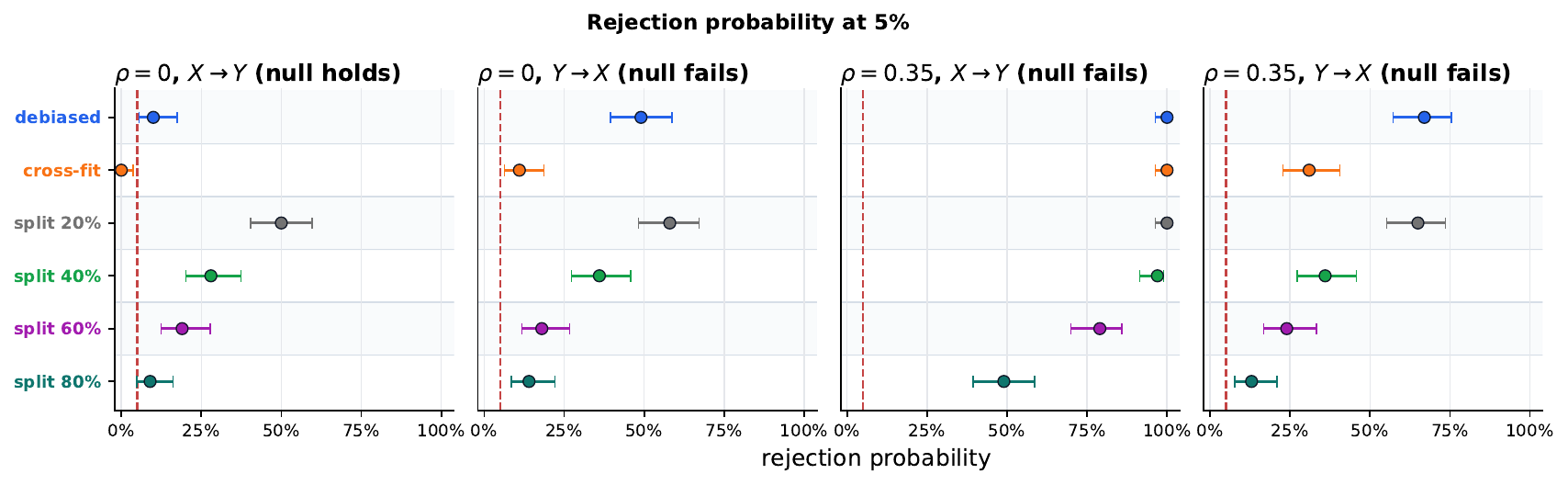}
    \captionsetup{font=small,skip=1pt,hypcap=false}
    \captionof{figure}{Rejection probabilities for nominal level-$5\%$ tests in the causal discovery experiment. ``debiased'' is the test base on the \Cref{eq: curlyI_zetan} ``cross-fit'' is the proposed permutation baseline which utilises the entire sample and cross-fits the nuisances, and``split $x$\%'' refers to tests where $x$\% of the sample is used for fiting the nuisances and the rest for a permutation test.}
    \label{fig:causal_arrow_rejection_probability}
\end{minipage}
\par\vspace{-\parskip}

\section{Discussion}
\label{sec: discussion}
\vspace{-0.2cm}

We developed an operator-level debiasing approach for kernel measures of residual noise heterogeneity. 
Rather than applying a semiparametric correction directly to scalar HSIC, whose first-order derivative 
degenerates at  independence, we target the RKHS cross-covariance operator between covariates 
and residuals. The resulting Hilbert-valued one-step estimator yields a bootstrap-calibrated test of 
residual independence and a confidence interval for the root-HSIC signal $\|\Psi_0\|_{H_{kl}}$. When dependence holds, the same estimator gives delta-method confidence intervals for $\|\Psi_0\|_{H_{kl}}^2$, 
attaining the semiparametric efficiency bound for this scalar functional. In synthetic experiments, the 
debiased procedure improves calibration and power relative to plug-in residual baselines.

A number of areas warrant further investigation.
First, it would be valuable to develop a non-asymptotic theory for our estimators and confidence sets. A second direction is suggested by our findings in Appendix~\ref{sec:inference}, where centering the delta-method confidence
sets at U-statistics leads to better finite sample performance despite being asymptotically equivalent. The same idea does not trivially extend to our more conservative
confidence sets based on the triangle inequality, but it would be interesting to understand how their small-sample performance can be optimized. Beyond this,
there are many alternative methods of constructing everywhere-valid confidence intervals from an efficient estimator of $\Psi$. It would be interesting to investigate
the optimal way of constructing a confidence set which has good coverage when $\Psi=0$ and reduces to the derived asymptotically efficient one
when $\Psi\neq 0$.

\bibliographystyle{plainnat}
\bibliography{bibliography}     

@article{gretton2015simpler,
  title={A simpler condition for consistency of a kernel independence test},
  author={Gretton, Arthur},
  journal={arXiv preprint arXiv:1501.06103},
  year={2015}
}

@article{agarwal2026sinkhorn,
  title={Sinkhorn Treatment Effects: A Causal Optimal Transport Measure},
  author={Agarwal, Medha and Luedtke, Alex},
  journal={arXiv preprint arXiv:2605.08485},
  year={2026}
}

@InProceedings{ChwSejGre14,
  author = 	 {K. Chwialkowski and D. Sejdinovic and A. Gretton},
  title = 	 {A Wild Bootstrap for Degenerate Kernel Tests},
   booktitle = {NeurIPS},
  year = 	 {2014}
}

@book{adams2003sobolev,
  author    = {Robert A. Adams and John J. F. Fournier},
  title     = {Sobolev spaces},
  edition   = {2nd},
  volume    = {140},
  publisher = {Elsevier Science \& Technology},
  year      = {2003},
  address   = {Chantilly},
  note      = {Web}
}

@InProceedings{chen2025nested,
  title = 	 {Nested expectations with kernel quadrature},
  author =       {Chen, Zonghao and Naslidnyk, Masha and Briol, Francois-Xavier},
  booktitle = 	 {Proceedings of the 42nd International Conference on Machine Learning},
  pages = 	 {8760--8793},
  year = 	 {2025},
  editor = 	 {Singh, Aarti and Fazel, Maryam and Hsu, Daniel and Lacoste-Julien, Simon and Berkenkamp, Felix and Maharaj, Tegan and Wagstaff, Kiri and Zhu, Jerry},
  volume = 	 {267},
  month = 	 {13--19 Jul},
  publisher =    {PMLR}
}

@article{luedtke2019omnibus,
  title={An omnibus non-parametric test of equality in distribution for unknown functions},
  author={Luedtke, Alex and Carone, Marco and van der Laan, Mark J},
  journal={Journal of the Royal Statistical Society Series B: Statistical Methodology},
  volume={81},
  number={1},
  pages={75--99},
  year={2019},
  publisher={Oxford University Press}
}

@article{williamson2023general,
  title={A general framework for inference on algorithm-agnostic variable importance},
  author={Williamson, Brian D and Gilbert, Peter B and Simon, Noah R and Carone, Marco},
  journal={Journal of the American Statistical Association},
  volume={118},
  number={543},
  pages={1645--1658},
  year={2023},
  publisher={Taylor \& Francis}
}

@article{SheJegGre09,
  Author =	 {H. Shen and S. Jegelka and A. Gretton},
  Issue =	 9,
  Journal =	 {IEEE Transactions on Signal Processing},
  Pages =	 {3498-3511},
  Title =	 {Fast Kernel-Based Independent Component Analysis},
  Volume =	 57,
  Year =	 2009
}

@article{albert2022adaptive,
  author  = {Albert, M\'{e}lisande and Laurent, B\'{e}atrice and Marrel, Amandine and Meynaoui, Anouar},
  title   = {Adaptive test of independence based on {HSIC} measures},
  journal = {The Annals of Statistics},
  volume  = {50},
  number  = {2},
  pages   = {858--879},
  year    = {2022},
  month   = {apr}
}

@inproceedings{smola2007hilbert,
  title={A Hilbert space embedding for distributions},
  author={Smola, Alex and Gretton, Arthur and Song, Le and Sch{\"o}lkopf, Bernhard},
  booktitle={International conference on algorithmic learning theory},
  pages={13--31},
  year={2007},
  organization={Springer}
}

@article{SriGreFukLanetal10,
  Author =	 {B. Sriperumbudur and A. Gretton and K. Fukumizu and
                  B. Sch{\"o}lkopf and G. Lanckriet},
  Journal =	 {Journal of Machine Learning Research},
  Pages =	 {1517-1561},
  Title =	 {Hilbert Space Embeddings and Metrics on Probability
                  Measures},
  Volume =	 11,
  Year =	 2010
}

@article{luedtke2025simplifying,
  title        = {Simplifying debiased inference via automatic differentiation and probabilistic programming},
  url          = {https://par.nsf.gov/biblio/10649160},
  doi          = {10.1093/jrsssb/qkaf052},
  journal      = {Journal of the Royal Statistical Society Series B: Statistical Methodology},
  publisher    = {Oxford Academic},
  author       = {Luedtke, Alex},
  year = {2025}
}

@article{smale2004shannon,
  title={Shannon sampling and function reconstruction from point values},
  author={Smale, Steve and Zhou, Ding-Xuan},
  journal={Bulletin of the American Mathematical Society},
  volume={41},
  number={3},
  pages={279--305},
  year={2004}
}

@book{tsybakov2004introduction,
  title={Introduction to nonparametric estimation},
  author={Tsybakov, Alexandre B},
  publisher={Springer},
  year={2009}
}

@article{de2006discretization,
  title={Discretization error analysis for Tikhonov regularization},
  author={De Vito, Ernesto and Rosasco, Lorenzo and Caponnetto, Andrea},
  journal={Analysis and Applications},
  volume={4},
  number={01},
  pages={81--99},
  year={2006},
  publisher={World Scientific}
}

@article{smale2005shannon,
  title={Shannon sampling {II}: Connections to learning theory},
  author={Smale, Steve and Zhou, Ding-Xuan},
  journal={Applied and Computational Harmonic Analysis},
  volume={19},
  number={3},
  pages={285--302},
  year={2005},
  publisher={Elsevier}
}

@article{steinwart2012mercer,
  title={Mercer’s theorem on general domains: On the interaction between measures, kernels, and RKHSs},
  author={Steinwart, Ingo and Scovel, Clint},
  journal={Constructive Approximation},
  volume={35},
  number={3},
  pages={363--417},
  year={2012},
  publisher={Springer}
}

@article{luedtke2024one,
author = {Alex Luedtke and Incheoul Chung},
title = {{One-step estimation of differentiable Hilbert-valued parameters}},
volume = {52},
journal = {The Annals of Statistics},
number = {4},
publisher = {Institute of Mathematical Statistics},
pages = {1534 -- 1563},
year = {2024},
doi = {10.1214/24-AOS2403},
URL = {https://doi.org/10.1214/24-AOS2403}
}

@incollection{efron1992bootstrap,
  title={Bootstrap methods: another look at the jackknife},
  author={Efron, Bradley},
  booktitle={Breakthroughs in statistics: Methodology and distribution},
  pages={569--593},
  year={1992},
  publisher={Springer}
}

@inproceedings{song2007supervised,
  title={Supervised feature selection via dependence estimation},
  author={Song, Le and Smola, Alex and Gretton, Arthur and Borgwardt, Karsten M and Bedo, Justin},
  booktitle={Proceedings of the 24th international conference on Machine learning},
  pages={823--830},
  year={2007}
}

@article{zhang2018large,
  title={Large-scale kernel methods for independence testing},
  author={Zhang, Qinyi and Filippi, Sarah and Gretton, Arthur and Sejdinovic, Dino},
  journal={Statistics and Computing},
  volume={28},
  number={1},
  pages={113--130},
  year={2018},
  publisher={Springer}
}

@inproceedings{rahimi2007random,
  title={Random Features for Large-Scale Kernel Machines},
  author={Rahimi, Ali and Recht, Benjamin},
  booktitle={Advances in Neural Information Processing Systems},
  volume={20},
  pages={1177--1184},
  year={2007}
}

@inproceedings{williams2001using,
  title={Using the Nystr{\"o}m Method to Speed Up Kernel Machines},
  author={Williams, Christopher K. I. and Seeger, Matthias},
  booktitle={Advances in Neural Information Processing Systems},
  volume={13},
  pages={682--688},
  year={2001},
  publisher={MIT Press}
}

@article{gretton2007kernel,
  title={A kernel statistical test of independence},
  author={Gretton, Arthur and Fukumizu, Kenji and Teo, Choon and Song, Le and Sch{\"o}lkopf, Bernhard and Smola, Alex},
  journal={Advances in neural information processing systems},
  volume={20},
  year={2007}
}

@article{berrett2019nonparametric,
  title={Nonparametric independence testing via mutual information},
  author={Berrett, Thomas B and Samworth, Richard J},
  journal={Biometrika},
  volume={106},
  number={3},
  pages={547--566},
  year={2019},
  publisher={Oxford University Press}
}

@article{heller2013consistent,
  title={A consistent multivariate test of association based on ranks of distances},
  author={Heller, Ruth and Heller, Yair and Gorfine, Malka},
  journal={Biometrika},
  volume={100},
  number={2},
  pages={503--510},
  year={2013},
  publisher={Oxford University Press}
}

@article{shekhar2023permutation,
  title={A permutation-free kernel independence test},
  author={Shekhar, Shubhanshu and Kim, Ilmun and Ramdas, Aaditya},
  journal={Journal of Machine Learning Research},
  volume={24},
  number={369},
  pages={1--68},
  year={2023}
}

@article{kennedy2024semiparametric,
  title={Semiparametric doubly robust targeted double machine learning: a review},
  author={Kennedy, Edward H},
  journal={Handbook of statistical methods for precision medicine},
  pages={207--236},
  year={2024},
  publisher={Chapman and Hall/CRC}
}

@article{hines2022demystifying,
  title={Demystifying statistical learning based on efficient influence functions},
  author={Hines, Oliver and Dukes, Oliver and Diaz-Ordaz, Karla and Vansteelandt, Stijn},
  journal={The American Statistician},
  volume={76},
  number={3},
  pages={292--304},
  year={2022},
  publisher={Taylor \& Francis}
}

@article{chernozhukov2018double,
  title={Double/debiased machine learning for treatment and structural parameters},
  author={Chernozhukov, Victor and Chetverikov, Denis and Demirer, Mert and Duflo, Esther and Hansen, Christian and Newey, Whitney and Robins, James},
  journal={The Econometrics Journal},
  volume={21},
  number={1},
  year={2018},
  publisher={The Econometrics Journal}
}

@article{schick1986asymptotically,
  title={On asymptotically efficient estimation in semiparametric models},
  author={Schick, Anton},
  journal={The Annals of Statistics},
  pages={1139--1151},
  year={1986},
  publisher={JSTOR}
}

@article{li2021self,
  title={Self-supervised learning with kernel dependence maximization},
  author={Li, Yazhe and Pogodin, Roman and Sutherland, Danica J and Gretton, Arthur},
  journal={Advances in Neural Information Processing Systems},
  volume={34},
  pages={15543--15556},
  year={2021}
}

@article{park2020measure,
  title={A measure-theoretic approach to kernel conditional mean embeddings},
  author={Park, Junhyung and Muandet, Krikamol},
  journal={Advances in neural information processing systems},
  volume={33},
  pages={21247--21259},
  year={2020}
}

@article{li2024towards,
  title={Towards optimal sobolev norm rates for the vector-valued regularized least-squares algorithm},
  author={Li, Zhu and Meunier, Dimitri and Mollenhauer, Mattes and Gretton, Arthur},
  journal={Journal of Machine Learning Research},
  volume={25},
  number={181},
  pages={1--51},
  year={2024}
}

@book{aubin2000applied,
  title={Applied functional analysis},
  author={Aubin, Jean-Pierre},
  year={2000},
  publisher={John Wiley \& Sons}
}

@inproceedings{gretton2005measuring,
  title={Measuring statistical dependence with Hilbert-Schmidt norms},
  author={Gretton, Arthur and Bousquet, Olivier and Smola, Alex and Sch{\"o}lkopf, Bernhard},
  booktitle={International conference on algorithmic learning theory},
  pages={63--77},
  year={2005},
  organization={Springer}
}

@book{van2000asymptotic,
  title={Asymptotic statistics},author={ {van der Vaart}, A. W.},
  volume={3},
  year={2000},
  publisher={Cambridge university press},
  collection={Cambridge Series in Statistical and Probabilistic Mathematics},
  series={Cambridge Series in Statistical and Probabilistic Mathematics},
   place={Cambridge}
}

@article{morzywolek2025inference,
  title={Inference on Local Variable Importance Measures for Heterogeneous Treatment Effects},
  author={Morzywolek, Pawel and Gilbert, Peter B and Luedtke, Alex},
  journal={arXiv preprint arXiv:2510.18843},
  year={2025}
}

@book{steinwart2008support,
  title={Support vector machines},
  author={Steinwart, Ingo and Christmann, Andreas},
  year={2008},
  publisher={Springer Science \& Business Media}
}

@book{berlinet2011reproducing,
  title={{Reproducing Kernel {H}ilbert Spaces in Probability and Statistics}},
  author={Berlinet, Alain and Thomas-Agnan, Christine},
  year={2011},
  publisher={Springer}
}

@article{carmeli2006vector,
  title={Vector valued reproducing kernel {H}ilbert spaces of integrable functions and {M}ercer theorem},
  author={Carmeli, Claudio and De Vito, Ernesto and Toigo, Alessandro},
  journal={Analysis and Applications},
  volume={4},
  number={04},
  pages={377--408},
  year={2006},
  publisher={World Scientific}
}

@article{carmeli2010vector,
  title={Vector valued reproducing kernel {H}ilbert spaces and universality},
  author={Carmeli, Claudio and De Vito, Ernesto and Toigo, Alessandro and Umanit{\'a}, Veronica},
  journal={Analysis and Applications},
  volume={8},
  number={01},
  pages={19--61},
  year={2010},
  publisher={World Scientific}
}

@book{lee1990ustatistics,
  author    = {Lee, A. J.},
  title     = {U-Statistics: Theory and Practice},
  edition   = {1},
  publisher = {Routledge},
  year      = {1990},
  doi       = {10.1201/9780203734520},
  url       = {https://doi.org/10.1201/9780203734520}
}

@article{he2026on,
  title={On the Hardness of Conditional Independence Testing In Practice},
  author={He, Zheng and Pogodin, Roman and Li, Yazhe and Deka, Namrata and Gretton, Arthur and Sutherland, Danica J},
  journal={arXiv preprint arXiv:2512.14000},
  year={2025}
}

@inproceedings{wang2026practical,
  title={Practical Kernel Selection for Kernel-based Conditional Independence Test},
  author={Wang, Wenjie and Gong, Mingming and Huang, Biwei and Bailey, James and Han, Bo and Zhang, Kun and Liu, Feng},
  booktitle={The Thirty-ninth Annual Conference on Neural Information Processing Systems},
  year={2025}
}

@article{pogodin2024practical,
  title={Practical kernel tests of conditional independence},
  author={Pogodin, Roman and Schrab, Antonin and Li, Yazhe and Sutherland, Danica J and Gretton, Arthur},
  journal={arXiv preprint arXiv:2402.13196},
  year={2024}
}

@article{fischer2020sobolev,
  title={Sobolev norm learning rates for regularized least-squares algorithms},
  author={Fischer, Simon and Steinwart, Ingo},
  journal={Journal of Machine Learning Research},
  volume={21},
  number={205},
  pages={1--38},
  year={2020}
}

@article{caponnetto2007optimal,
  title={Optimal rates for the regularized least-squares algorithm},
  author={Caponnetto, Andrea and De Vito, Ernesto},
  journal={Foundations of Computational Mathematics},
  volume={7},
  number={3},
  pages={331--368},
  year={2007},
  publisher={Springer}
}

@inproceedings{lietal2022optimal,
 author = {Li, Zhu and Meunier, Dimitri and Mollenhauer, Mattes and Gretton, Arthur},
 booktitle = {Advances in Neural Information Processing Systems},
 pages = {4433--4445},
 title = {Optimal Rates for Regularized Conditional Mean Embedding Learning},
 volume = {35},
 year = {2022}
}

@inproceedings{lin2018optimal,
  title={Optimal distributed learning with multi-pass stochastic gradient methods},
  author={Lin, Junhong and Cevher, Volkan},
  booktitle={International Conference on Machine Learning},
  pages={3092--3101},
  year={2018},
  organization={PMLR}
}

@article{lin2020optimal,
  title={Optimal rates for spectral algorithms with least-squares regression over {H}ilbert spaces},
  author={Lin, Junhong and Rudi, Alessandro and Rosasco, Lorenzo and Cevher, Volkan},
  journal={Applied and Computational Harmonic Analysis},
  volume={48},
  number={3},
  pages={868--890},
  year={2020},
  publisher={Elsevier}
}

@book{bickel1993efficient,
  title={Efficient and adaptive estimation for semiparametric models},
  author={Bickel, Peter J and Klaassen, Chris AJ and Bickel, Peter J and Ritov, Ya’acov and Klaassen, J and Wellner, Jon A and Ritov, YA'Acov},
  volume={4},
  year={1993},
  publisher={Johns Hopkins University Press Baltimore}
}

@book{serfling2009approximation,
  title={Approximation theorems of mathematical statistics},
  author={Serfling, Robert J},
  year={2009},
  publisher={John Wiley \& Sons}
}

@article{sriperumbudur2011universality,
  title={Universality, characteristic kernels and {RKHS} embedding of measures},
  author={Sriperumbudur, Bharath K and Fukumizu, Kenji and Lanckriet, Gert RG},
  journal={Journal of Machine Learning Research},
  volume={12},
  number={Jul},
  pages={2389--2410},
  year={2011}
}

@article{peters2013causal,
author = {Peters, Jonas and Mooij, Joris and Janzing, Dominik and Schölkopf, Bernhard},
year = {2013},
month = {09},
pages = {},
title = {Causal Discovery with Continuous Additive Noise Models},
volume = {15},
journal = {Journal of Machine Learning Research},
doi = {10.15496/publikation-1672}
}

@article{mooij2016distinguishing,
  author  = {Mooij, J. M. and Peters, J. and Janzing, D. and Zscheischler, J. and Sch{\"o}lkopf, B.},
  title   = {Distinguishing Cause from Effect Using Observational Data: Methods and Benchmarks},
  journal = {Journal of Machine Learning Research},
  volume  = {17},
  number  = {32},
  pages   = {1--102},
  year    = {2016}
}

@inproceedings{schrab2022efficient,
  author    = {Schrab, Antonin and Kim, Ilmun and Guedj, Benjamin and Gretton, Arthur},
  title     = {Efficient Aggregated Kernel Tests Using Incomplete U-Statistics},
  booktitle = {Advances in Neural Information Processing Systems},
  year      = {2022}
}

\section{Guide to appendix}
\label{sec:guide}

In this section, we provide a guide to the contents of the appendix. At the top of each section, we provide a table listing the \blue{notation} used in that section, with the exception of \Cref{sec:notations}, where we provide a table for \blue{notation} used across the paper.

Sample splitting allows us to avoid restrictive Donsker-type complexity restrictions on \blue{the function class}, and significantly \blue{simplifies} proofs. We refer the reader to \citet[Section 4.2]{kennedy2024semiparametric} for an exposition on sample splitting and cross-fitting. The main drawback of \blue{the sample-splitting} estimator is that it only uses half of the total \blue{sample} for evaluation, \blue{so} its asymptotic variance is two times larger than the semiparametric efficiency bound for the full sample. The virtue of the \blue{sample-splitting} estimator is that it is significantly less notationally burdensome to present, as there is no need to perform bookkeeping for the fold indices. In order to strike a balance between pedagogical clarity and generality, we opted to present the \blue{sample-splitting} estimator, as it already reveals the essential features of our debiased procedure and the nuisance requirements (\Cref{asst: consistency,asst: suff_cond_an}).

As explained in \Cref{rem:cross_fitting}, by swapping the role of evaluation and fitting samples, we can recover full statistical efficiency by using a $K$-fold cross-fitted estimator. The prototypical example is the $2$-fold cross-fitted estimator. It is standard in theoretical works on semiparametric efficiency theory to present guarantees for the $2$-fold cross-fitted estimator, as the extension of the proofs to the $K$-fold cross-fitted one presents no essential difficulty while requiring more burdensome notation, see e.g. \citet[Section 2.5]{luedtke2024one}. Furthermore, the asymptotic variance \blue{of the} $2$-fold cross-fitted estimator already attains the semiparametric efficiency bound.

The contents of the appendix can therefore be summarized as follows. In \Cref{sec:notations}, we introduce additional background material.  In \Cref{sec:EIF_appendix} we prove \Cref{thm:cco_eif} in the main text, in the more general setting where \Cref{asst:nested_covariates} may not hold. Our results \blue{therein reduce} immediately to \Cref{thm:cco_eif} when \Cref{asst:nested_covariates} does in fact hold.

In \Cref{sec: appendix_one_step}, we generalize the guarantees in \citet{luedtke2024one} for generic Hilbert-valued parameters to a two-fold cross-fitted variant of the estimator presented in \Cref{eq: two_fold_os}. This allows us to exhibit the required rates on the nuisance functions, and show that they are exactly given by \Cref{asst: consistency,asst: suff_cond_an}, with a straightforward notational extension from the \blue{sample-splitting} to \blue{the cross-fitting} case.

The $K$-fold cross-fitted estimator often enjoys superior \blue{finite-sample} performance to the $2$-fold cross-fitted estimator, and as a result, in all of our empirical simulations we use \blue{a} $K$-fold cross-fitted estimator for $K\geq 2$. To support this and to ensure reproducibility of our research, we derive the \blue{closed-form} solution for the $K$-fold cross-fitted estimator in \Cref{sec: computation}, where we are fully explicit about how to perform tensor slice operations that \blue{let} us perform cross-fitting in full generality.

In \Cref{sec:inference}, we provide supplementary theoretical results to the two methods of constructing confidence intervals proposed in \Cref{sec:main_hyp_test} in the main text, for the general $K$-fold cross-fitting setting. In particular, we provide a \blue{closed-form} estimator $\widehat{\sigma}^2$ for the asymptotic variance in the delta method in the $K$-fold cross-fitting case, fully reusing expressions computed in \Cref{sec: computation}. We then establish the consistency \blue{of} $\widehat{\sigma}^2$. In \Cref{sec:finite_sample_delta_validity}, we construct the data-adaptive diagnostic for the trustworthiness of the delta method introduced in \Cref{sec:main_hyp_test}. In \Cref{app:statistical_learning_theory}, we provide a detailed introduction \blue{to} statistical learning theory for vector-valued kernel ridge regression, which allows us to obtain convergence rates for the nuisance $\widehat{v}_j$, and thereby show when \Cref{asst: consistency,asst: suff_cond_an} is satisfied by more fundamental assumptions on Sobolev smoothness.

In \Cref{sec:experiment_details}, we provide all experimental results not included in the main text and \blue{the experimental} configuration.

\section{Notation and background}
\label{sec:notations}

\begin{center}
\small
\setlength{\tabcolsep}{5pt}
\renewcommand{\arraystretch}{1.25}
\begin{tabularx}{\linewidth}{@{}>{\raggedright\arraybackslash}p{0.28\linewidth}>{\raggedright\arraybackslash}X@{}}
\toprule
Notation & Meaning \\
\midrule
$O=(X,W,Y), o_i=(x_i,w_i,y_i)$ & Observation and realized observation. \\
\midrule
$\mathcal M$, $P_0$, $P$ & Nonparametric model, true law, and generic law. Except for $H_0$, subscript $0$ always denotes an object indexed by the true data generating process. \\
\midrule
$m_P$, $\xi_P$ & Regression function $m_P(w)=\mathbb E_P[Y\mid W=w]$ and $\xi_P(w,y)=y-m_P(w)$. \\
\midrule
$k,l$; $\mathcal H_k,\mathcal H_l,\mathcal H_{kl}$ & Reproducing kernels on $\mathcal{X}$ and $\mathbb{R}^{d_y}$; corresponding reproducing kernel Hilbert spaces; tensor-product RKHS $\cH_{k}\otimes \cH_l$.\\
\midrule
$\varphi_k$, $\varphi_l$, $\varphi_{\xi,P}$ & Canonical feature maps, with $\varphi_{\xi,P}(w,y)=\varphi_l\{\xi_P(w,y)\}$. \\
\midrule
$\mu_{X,P}$, $\mu_{\xi,P}$ & Kernel mean embeddings $\mathbb E_P[\varphi_k(X)]$ and $\mathbb E_P[\varphi_{\xi,P}(W,Y)]$. \\
\midrule
$\widetilde\varphi_{k,P}$, $\widetilde\varphi_{\xi,P}$ & Centered feature maps obtained by subtracting the corresponding kernel mean embeddings. \\
\midrule
$\Psi(P)$, $\Psi_0$ & RKHS cross-covariance operator between $X$ and the residual $\xi_P(W,Y)$; $\Psi_0=\Psi(P_0)$. \\
\midrule
$\hsic$ & Signal magnitude $\theta_0=\|\Psi_0\|_{\mathcal H_{kl}}$; its square is HSIC. \\
\midrule
$\partial_j\varphi_l$ & Derivative of the residual feature map with respect to residual coordinate $j$. \\
\midrule
$\widehat{\cdot}$, $\check{\cdot}$, ${\cdot}^{(-k)}$ & Estimated object, cross-fitted object, and object fitted without fold $k$. \\
\bottomrule
\end{tabularx}
\end{center}

This table records only notation used across the paper. Section-specific notation is introduced again, locally, in the appendix section where it is used.

\subsection{Background on semiparametric efficiency theory}
\label{sec: review_sp_stat}

\textbf{Pathwise differentiability for Hilbert-valued parameter.} Let $\mathcal{P}$ be a collection of distributions defined on a common complete metric space $(\mathcal{Z}, \mathcal{B})$. Let $\mathcal{H}$ denote a reproducing kernel Hilbert space over a set $\mathcal{T}$ and $\nu : \mathcal{P} \rightarrow \mathcal{H}$ a parameter whose value is to be estimated. We let the feature map of $\mathcal{H}$ be denoted as $\phi$. The parameter $\nu$ is said to be pathwise differentiable at $P\in \mathcal{P}$ if and only if there exists a continuous linear operator $\dot{\nu}_P : \dot{\mathcal{P}}_P \rightarrow \mathcal{H}$ such that for all parametric submodels $\{P_{\epsilon}: \epsilon \in [0,\delta)\}$ with score function $s$ and centred at $P$ with $P_0 = P$, we have
\begin{align*}
    \left\|\nu(P_{\epsilon}) - \nu(P) - \epsilon \dot{\nu}_P(s)\right\|_{\mathcal{H}} = o(\epsilon). 
\end{align*}
The operator $\dot{\nu}_P$ is called the local parameter of $\nu$ at $P$ and its Hermitian adjoint $\dot{\nu}_P^{\ast}$ is called the efficient influence operator. We refer the reader to \citet[Section 2.2]{luedtke2024one} and \Cref{sec: review_sp_stat} in the Appendix for a full definition of parametric submodel and score function. 

\textbf{Hilbert-valued one-step estimator.} In the notation of the above paragraph, we define $\phi_P : \mathcal{Z} \to \mathcal{H}$ as follows for each $t\in \mathcal{T}$:
\begin{align*}
    \phi_P(z)(t) = \dot{\nu}_P^{\ast}(\phi(t))(z), \qquad \text{$P$-a.s. $z$}.
\end{align*}
If $\phi_P$ is square integrable in the sense of $\|\phi_P\|_{L^2(P;\mathcal{H})} < \infty$, then we say that $\phi_P$ is the EIF of $\nu$. In this case, we can then construct a one-step estimator of the form $\nu(\widehat{P}) + P_n \phi_{\widehat{P}}$, where $P_n$ denotes the sample average and $\widehat{P}$ is an initial data-driven estimate of the data-generating distribution $P_0$. 

\subsection{Background on reproducing kernel Hilbert spaces}
\begin{asst}[Regularity assumptions on $l$]
\label{ass:regularity_l_appendix}
The kernel $l$ satisfies the following conditions.
\begin{enumerate}[label=(K\arabic*), leftmargin=2.4em]
    \item $l$ is symmetric, positive definite, and bounded on the diagonal:
    \[
        \kappa_0^2 := \sup_{u\in\R^d} l(u,u)<\infty.
    \]
    \item For every pair of multi-indices $\alpha,\beta\in\N_0^d$ with $|\alpha|\le 2$ and $|\beta|\le 2$, the derivative $\partial_1^\alpha\partial_2^\beta l$ exists and is continuous on $\R^d\times\R^d$.
    \item The first and second diagonal derivative seminorms are finite:
    \begin{align*}
        \kappa_1^2
        &:= \sup_{u\in\R^d}\sup_{\normllm{a}=1}
        \partial_{1,a}\partial_{2,a}l(u,u)<\infty, \\
        \kappa_2^2
        &:= \sup_{u\in\R^d}\sup_{\normllm{a}=\normllm{b}=1}
        \partial_{1,a}\partial_{1,b}\partial_{2,a}\partial_{2,b}l(u,u)<\infty.
    \end{align*}
\end{enumerate}
\end{asst}

Derivatives with respect to the first and second arguments of $l$ are denoted by $\partial_1$ and $\partial_2$.  For multi-indices $\alpha,\beta\in\N_0^d$, write
\[
    \partial_1^\alpha \partial_2^\beta l(u,v)
    := \frac{\partial^{|\alpha|+|\beta|}}{\partial u_1^{\alpha_1}\cdots\partial u_d^{\alpha_d}\partial v_1^{\beta_1}\cdots\partial v_d^{\beta_d}}l(u,v),
\]
whenever this derivative exists.  For directions $a,b,c,e\in\R^d$, write, for example,
\[
    \partial_{1,a}\partial_{2,b}l(u,v)
    := \sum_{j,r=1}^d a_j b_r\,\partial_{1,j}\partial_{2,r}l(u,v).
\]

\begin{lemma}[Feature-map consequences]\label{lem:kernel-to-feature}
Let $\phi(u):=l(\cdot,u)$ be the canonical feature map into $\cH_l$.  
Under \Cref{ass:regularity_l_appendix}, $\phi$ is twice continuously Fr\'echet differentiable and, for all
$u,v\in\R^d$ and $a,b\in\R^d$,
\[
    \ipllm{D\phi(u)[a]}{D\phi(v)[b]}_{\cH_l}
    =
    \partial_{1,a}\partial_{2,b}l(u,v),
\]
and
\[
    \ipllm{D^2\phi(u)[a,b]}{D^2\phi(v)[a,b]}_{\cH_l}
    =
    \partial_{1,a}\partial_{1,b}\partial_{2,a}\partial_{2,b}l(u,v).
\]
Consequently,
\[
    \sup_u\|\phi(u)\|_{\cH_l}\le \kappa_0,\qquad
    \sup_u\|D\phi(u)\|_{\opllm}\le \kappa_1,\qquad
    \sup_u\|D^2\phi(u)\|_{\opllm}\le \kappa_2.
\]
Hence, for all $u,v,h\in\R^d$,
\[
    \|\phi(u)-\phi(v)\|_{\cH_l}\le \kappa_1\|u-v\|,
\]
\[
    \|D\phi(u)-D\phi(v)\|_{\opllm}\le \kappa_2\|u-v\|,
\]
and
\[
    \|\phi(u+h)-\phi(u)-D\phi(u)[h]\|_{\cH_l}
    \le \frac{\kappa_2}{2}\|h\|^2.
\]
In coordinates,
\[
    \ipllm{\partial_j\phi(u)}{\partial_r\phi(v)}_{\cH_l}
    =
    \partial_{1,j}\partial_{2,r}l(u,v).
\]
\end{lemma}

\begin{proof}
The differentiability and inner-product identities follow from the RKHS differentiability theorem for kernels with continuous mixed derivatives, applied up to order two. The bounds are obtained by evaluating the identities at $u=v$:
\[
\|D\phi(u)[a]\|_{\cH_l}^2
=
\partial_{1,a}\partial_{2,a}l(u,u)
\le \kappa_1^2\|a\|^2,
\]
and
\[
\|D^2\phi(u)[a,b]\|_{\cH_l}^2
=
\partial_{1,a}\partial_{1,b}\partial_{2,a}\partial_{2,b}l(u,u)
\le \kappa_2^2\|a\|^2\|b\|^2.
\]
The Lipschitz and Taylor bounds follow from the Banach-valued fundamental theorem of calculus.
\end{proof}

\begin{prop}[Gaussian and Sobolev kernels satisfy \Cref{ass:regularity_l_appendix}]
The following standard kernels satisfy \Cref{ass:regularity_l_appendix}.
\begin{enumerate}[label=(\roman*), leftmargin=2.2em]
    \item For the Gaussian kernel
    \[
        l_\sigma(u,v)=\exp\{-\|u-v\|^2/(2\sigma^2)\},
    \]
    one may take
    \[
        \kappa_0=1,\qquad
        \kappa_1=\sigma^{-1},\qquad
        \kappa_2=\sqrt{3}\,\sigma^{-2}.
    \]

    \item For the Sobolev/Bessel-potential kernel
    \[
        l_{s,\rho}(u,v)
        =
        \frac{1}{(2\pi)^d}
        \int_{\R^d}e^{i\omega^\top(u-v)}
        (1+\rho^2\|\omega\|^2)^{-s}\,\ddllm\omega,
    \]
    \Cref{ass:regularity_l_appendix} holds whenever
    \[
        s>d/2+2.
    \]
    Equivalently, Mat\'ern kernels with smoothness \(\nu>2\) satisfy
    \Cref{ass:regularity_l_appendix} when their spectral density is proportional to
    \[
        (\lambda^2+\|\omega\|^2)^{-(\nu+d/2)}.
    \]
\end{enumerate}
\end{prop}

\begin{proof}
For the Gaussian kernel, smoothness and boundedness are immediate. Direct differentiation gives, for unit $a,b$,
\[
    \partial_{1,a}\partial_{2,a}l_\sigma(u,u)=\sigma^{-2},
\]
and
\[
    \partial_{1,a}\partial_{1,b}\partial_{2,a}\partial_{2,b}l_\sigma(u,u)
    =
    \sigma^{-4}\{1+2(a^\top b)^2\}
    \le 3\sigma^{-4}.
\]

For the Sobolev kernel, positive definiteness follows from its nonnegative spectral density. Derivatives up to order two in each argument are obtained by differentiating under the integral. This is justified whenever
\[
    \int_{\R^d}\|\omega\|^4(1+\rho^2\|\omega\|^2)^{-s}\,\ddllm\omega<\infty,
\]
which is equivalent to $s>d/2+2$. The same fourth-moment bound gives $\kappa_1<\infty$ and $\kappa_2<\infty$. For the Mat\'ern family, $s=\nu+d/2$, so the condition is $\nu>2$.
\end{proof}

\section{Efficient influence function of the cross-covariance operator}
\label{sec:EIF_appendix}

We prove \Cref{thm:cco_eif} in the main text. In this work, this section is the only section where we do not impose \Cref{asst:nested_covariates}. This is because it is essentially for free to compute the EIF in the more general setting where $X$ does not have to be $W$-measurable. At the end of this section, we derive the simplification when $X$ is $W$-measurable, which recovers the statement of \Cref{thm:cco_eif}. 
\begin{proof}[Proof of \Cref{thm:cco_eif} in the general setting where we drop \Cref{asst:nested_covariates}]
Following a standard convention in semiparametric statistical theory, we use the subscript $0$ to denote any objects indexed by the true distribution $P_0$. We have 
\begin{align*}
    &\mathbb{E}_{P}\left[\widetilde{\phi}_{X,P}(X) \otimes \widetilde{\varphi}_{\xi,P}(W, Y)\right]\\
    &= \mathbb{E}_{P}[(\varphi_k(X) - \mathbb{E}_{P}[\varphi_k(X')])\otimes (\varphi_{\xi, P}(W, Y) - \mathbb{E}_{P}[\varphi_{\xi, P}(W', Y')])]\\
    &= \textcolor{purple}{\mathbb{E}_{P}[\varphi_k(X) \otimes \varphi_{\xi, P}(W, Y)]} - \textcolor{blue}{\mathbb{E}_{P}[\varphi_k(X)]} \otimes \textcolor{teal}{\mathbb{E}_{P}[\varphi_{\xi, P}(W,Y)]}. 
\end{align*}
We define the following RKHS-valued parameters on $\mathcal{M}$,
\begin{align*}
    \textcolor{purple}{\Psi_1(P)} &:= \mathbb{E}_P[\varphi_k(X)\otimes \varphi_l(Y - \mathbb{E}_P[Y\mid W])] = \mathbb{E}_P[\varphi_k(X)\otimes \varphi_{\xi, P}(W,Y)]\\
    \textcolor{blue}{\Psi_2(P)} &:= \mathbb{E}_P[\varphi_k(X)]\\
    \textcolor{teal}{\Psi_3(P)} &:= \mathbb{E}_P[\varphi_l(Y - \mathbb{E}_P[Y\mid W])] = \mathbb{E}_P[\varphi_{\xi, P}(W,Y)]\\
    \textcolor{violet}{\Psi(P)} &= \textcolor{purple}{\Psi_1(P)} - \textcolor{blue}{\Psi_2(P) \otimes \textcolor{teal}{\Psi_3(P)}}. 
\end{align*}
We write $\textcolor{purple}{\Psi_1(P)}$ using the computation graph of primitives
\begin{align*}
    \theta_1:  \mathcal{M}\times \emptyset &\rightarrow L^2(P_0 ; \mathbb{R}^{d_y})\\
    (P,) &\mapsto ((x,w,y)\mapsto y)\\
    \theta_2 : \mathcal{M}\times L^2(P_0;\mathbb{R}^{d_y}) &\rightarrow L^2(P_{0,W};\mathbb{R}^{d_y})\\
    (P, h_1)&\mapsto \mathbb{E}_P[h_1(X,W,Y)\mid W=w]\\
    \theta_3 : \mathcal{M}\times L^2(P_{0,W};\mathbb{R}^{d_y}) &\rightarrow L^2(P_{0};\mathbb{R}^{d_y})\\
    (P,h_2) &\mapsto ((x,w,y)\mapsto h_2(w))\\
    \theta_4 :  \mathcal{M}\times L^2(P_{0};\mathbb{R}^{d_y}) &\rightarrow L^2(P_{0};\mathbb{R}^{d_y})\\
    (P,h_3) &\mapsto ((x,w,y)\mapsto y - h_3(x,w,y))\\
    \theta_5 :  \mathcal{M}\times L^2(P_{0};\mathbb{R}^{d_y}) &\rightarrow L^2(P_{0};\mathcal{H}_{kl})\\
    (P,h_4) &\mapsto ((x,w,y)\mapsto \varphi_k(x) \otimes \varphi_{l}(h_4(x,w,y)))\\
    \theta_6 : \mathcal{M}\times L^2(P_{0};\mathcal{H}_{kl})  &\rightarrow \mathcal{H}_{kl}\\
    (P, h_5) &\mapsto \mathbb{E}_P[h_5(X,W,Y)].
\end{align*}
We let $h_1 = \theta_1(P,)$, $h_2 = \theta_2(P, h_1)$, $h_3 = \theta_3(P,h_2)$, $h_4 = \theta_4(P,h_3)$, $h_5 = \theta_5(P,h_4)$, $h_6 = \theta_6(P, h_5)$. Concretely, we have
\begin{align*}
    h_1(x,w,y) &= y\\
    h_2(w) &= \mathbb{E}_P[Y\mid W=w]\\
    h_3(x,w,y) &= h_2(w) = \mathbb{E}_P[Y\mid W=w]\\
    h_4(x,w,y) &= y - \mathbb{E}_P[Y\mid W=w]\\
    h_5(x,w,y) &= \varphi_k(x)\otimes \varphi_l (y - \mathbb{E}_P[Y\mid W=w])\\
    h_6 &= \mathbb{E}_P\left[\varphi_k(X)\otimes \varphi_l (Y - \mathbb{E}_P[Y\mid W])\right]. 
\end{align*}
We have $\textcolor{purple}{\Psi_1(P)} = h_6$. Note that all of the above primitives are contained in \citet[Table 1]{luedtke2025simplifying}. They are respectively \emph{constant map}, \emph{conditional mean}, \emph{lift to new domain}, \emph{pointwise operation}, \emph{pointwise operation} and \emph{marginal mean}. For $\textcolor{purple}{\Psi_1(P)}$, we initialize the adjoint variables $f_i = 0$ for $0\leq i\leq 5$ and let $f_6 \in \mathcal{H}_{kl}$ be user specified. We use the notation convention in \citet{luedtke2024one} to use subscript $0$ to denote objects associated with the true data generating process. We let $\dot{\mathcal{M}}_0$ denote the tangent space of the statistical model $\mathcal{M}$ at $P_0$, which we assume to be the space $L^2_0(P_0)$ of mean-zero square integrable functions. This corresponds to a nonparametric statistical model. We now describe the backpropagation procedure. 

\begin{enumerate}
    \item $j = 6$. $h_6$ is a \emph{marginal mean}. 
    We have that the pathwise derivative of $\theta_6$ at $(P_0, h_5)$ is given by
    \begin{align*}
        D[\theta_6](P_0, h_5) &: \dot{\mathcal{M}}_0 \oplus L^2(P_0;\mathcal{H}_{kl}) \rightarrow \mathcal{H}_{kl}. 
    \end{align*}
    We backpropagate $f_6$ through $\theta_6$ to $f_0, f_5$ as follows:
    \begin{align*}
        \begin{pmatrix}
            f_0\\ f_5
        \end{pmatrix} += D[\theta_6](P_0, h_5)^{\ast}(f_6) =  \begin{pmatrix}(x,w,y)\mapsto\langle f_6, h_5(x,w,y) - \mathbb{E}_0[h_5(X,W,Y)]\rangle_{\mathcal{H}_{kl}}\\
        (x,w,y)\mapsto f_6\end{pmatrix}. 
    \end{align*}
    \item $j = 5$. We have
    \begin{align*}
        D[\theta_5](P_0, h_4) &: \dot{\mathcal{M}}_0\oplus L^2(P_{0};\mathbb{R}^{d_y}) \rightarrow L^2(P_{0};\mathcal{H}_{kl}). 
    \end{align*}
    We backpropagate $f_5$ through $\theta_5$ to $f_4$ as follows:
    \begin{align*}
        \begin{pmatrix}
            f_0\\ f_4
        \end{pmatrix} +=  D[\theta_5](P_0, h_4)^{\ast}(f_5) = \begin{pmatrix}
            0\\
            (x,w,y)\mapsto (\langle f_5(x,w,y), \varphi_k(x)\otimes (\partial_j\varphi_{l})(h_4(x,w,y))\rangle_{\mathcal{H}_{kl}})_{j=1}^{d_y}
        \end{pmatrix},
    \end{align*}
    where $\partial_j\varphi_{l}$ refers to the $j$th-entry of the gradient of $\varphi_{l} : \mathbb{R}^{d_y}\to \mathcal{H}_l$. 
    \item $j = 4$. We have
    \begin{align*}
        D[\theta_4](P_0, h_3) : \dot{\mathcal{M}}_0 \oplus L^2(P_{0};\mathbb{R}^{d_y}) \to L^2(P_0;\mathbb{R}^{d_y}).
    \end{align*}
    We backpropagate $f_4$ through $\theta_4$ to $f_3$ as follows:
    \begin{align*}
        \begin{pmatrix}
            f_0\\ f_3
        \end{pmatrix} +=  D[\theta_4](P_0, h_3)^{\ast}(f_4) &=\begin{pmatrix}
            0\\
            (x,w,y) \mapsto -f_{4}(x,w,y)
        \end{pmatrix}.
    \end{align*}
    \item $j=3$. We have
    \begin{align*}
        D[\theta_3](P_0, h_2): \dot{\mathcal{M}}_0 \oplus L^2(P_{0,W};\mathbb{R}^{d_y}) \to L^2(P_0;\mathbb{R}^{d_y}). 
    \end{align*}
    We backpropagate $f_3$ through $\theta_3$ to $f_2$ as follows:
    \begin{align*}
        \begin{pmatrix}
            f_0\\ f_2
        \end{pmatrix} +=  D[\theta_3](P_0, h_2)^{\ast}(f_3) &= \begin{pmatrix}
            0\\
            w \mapsto \mathbb{E}_0[f_{3}(X,W,Y)\mid W=w]
        \end{pmatrix}.
    \end{align*}
    \item $j=2$. We have
    \begin{align*}
        D[\theta_2](P_0, h_1) : \dot{\mathcal{M}}_0 \oplus L^2(P_{0};\mathbb{R}^{d_y}) \to L^2(P_{0,W};\mathbb{R}^{d_y}).
    \end{align*}
    We backpropagate $f_2$ through $\theta_2$ to $f_1$ as follows:
    \begin{align*}
        \begin{pmatrix}
            f_0\\ f_1
        \end{pmatrix} +=  D[\theta_2](P_0, h_1)^{\ast}(f_2) = \begin{pmatrix}
            (x,w,y)\mapsto f_2(w)^{\top}(h_1(x,w,y) - \mathbb{E}_0[h_1(X,W,Y)\mid W=w])\\
            (x,w,y)\mapsto f_2(w)
        \end{pmatrix}. 
    \end{align*}
    \item $j=1$. We have 
    \begin{align*}
        D[\theta_1](P_0): \dot{\mathcal{M}}_0 \to L^2(P_0;\mathbb{R}^{d_y}).
    \end{align*}
    We backpropagate $f_1$ through $\theta_1$ to $f_0$ as follows:
    \begin{align*}
        f_0 += D[\theta_1](P_0)^{\ast}(f_1) =0. 
    \end{align*}
\end{enumerate}
We have 
\begin{align*}
    f_5(x,w,y) &= f_6\\
    f_4(x,w,y) &= (\langle f_6, \varphi_k(x) \otimes (\partial_j \varphi_l)(h_4(x,w,y))\rangle_{\cH_{kl}})_{j=1}^{d_y}\\
    f_3(x,w,y) &= -(\langle f_6, \varphi_k(x) \otimes (\partial_j \varphi_l)(h_4(x,w,y))\rangle_{\cH_{kl}})_{j=1}^{d_y}\\
    f_2(w) &= \left(-\mathbb{E}_0\left[\langle f_6, \varphi_k(X) \otimes (\partial_j \varphi_l)(h_4(X,W,Y))\rangle_{\cH_{kl}}\mid W=w\right]\right)_{j=1}^{d_y}\\
    f_1(x,w,y) &= \left(-\mathbb{E}_0\left[\langle f_6, \varphi_k(X) \otimes (\partial_j \varphi_l)(h_4(X,W,Y))\rangle_{\cH_{kl}}\mid W=w\right]\right)_{j=1}^{d_y}.
\end{align*}
Finally, summing the adjoint contributions for $f_0$ yields
\begin{align*}
    f_0(x,w,y) &= \langle f_6, h_5(x,w,y) - \mathbb{E}_0[h_5(X,W,Y)]\rangle_{\cH_{kl}}\\
                &+ f_2(w)^{\top}\left(h_1(x,w,y) - \mathbb{E}_0[h_1(X,W,Y)\mid W=w]\right)\\
                &= \langle f_6, \varphi_k(x)\otimes \varphi_l(y - \mathbb{E}_0[Y\mid W=w]) - \mathbb{E}_0\left[\varphi_k(X)\otimes \varphi_l(Y - \mathbb{E}_0[Y\mid W])\right]\rangle_{\cH_{kl}}\\
                &- \sum_{j=1}^{d_y}\mathbb{E}_0\left[\langle f_6, \varphi_k(X)\otimes (\partial_j \varphi_l)(Y - \mathbb{E}_0[Y\mid W])\rangle_{\cH_{kl}} \mid W=w\right](y_j - \mathbb{E}_0[Y_j\mid W=w])\\
                &= \langle f_6, \varphi_k(x)\otimes \varphi_l(y - \mathbb{E}_0[Y\mid W=w]) - \mathbb{E}_0\left[\varphi_k(X)\otimes \varphi_l(Y - \mathbb{E}_0[Y\mid W])\right]\rangle_{\cH_{kl}}\\
                &- \mathbb{E}_0\left[\langle f_6, \varphi_k(X)\otimes (\nabla \varphi_l)(Y - \mathbb{E}_0[Y\mid W])\rangle_{\cH_{kl}} \mid W=w\right]^{\top}(y - \mathbb{E}_0[Y\mid W=w]),
\end{align*}    
where $(\nabla \varphi_l)(Y - \mathbb{E}_0[Y\mid W])$ denotes a tuple in $\cH_l^{d_y}$, and the inner product with $f_6$ is understood componentwise.

Recall we defined $m_0(w)=\mathbb{E}_0[Y\mid W=w]$, $\xi_0(w,y)=y-m_0(w)$, and $\varphi_{\xi,0}(w,y)=\varphi_l\{\xi_0(w,y)\}$. We write $\mu_{X,0}:= \mathbb{E}_{0}[\varphi_k(X)]$ and $\mu_{\xi, 0}:= \mathbb{E}_0[\varphi_{\xi, 0}(W,Y)]$. For each $j = 1,\dots,d_y$, define
\begin{align*}
    v_j(w) &:= \mathbb{E}_0[(\partial_j\varphi_{l})(\xi_0(W,Y))\mid W=w]\in\mathcal{H}_l,\\
    a_j(w) &:= \mathbb{E}_0[\varphi_k(X)\otimes(\partial_j\varphi_{l})(\xi_0(W,Y))\mid W=w]\in\mathcal{H}_{kl}.
\end{align*}
By \citet[Theorem 1]{luedtke2025simplifying}, $\textcolor{purple}{\Psi_1(P)}$ is pathwise differentiable at $P_0$, and its efficient influence operator $\dot{\textcolor{purple}{\Psi}}_{1,0}^{\ast} : \mathcal{H}_{kl} \to \dot{\mathcal{M}}_0$ is given by 
\begin{align*}
    \dot{\textcolor{purple}{\Psi}}_{1,0}^{\ast}(f)(x,w,y) &= \left\langle f, \varphi_k(x)\otimes \varphi_{\xi, 0}(w,y) - \mathbb{E}_0[\varphi_k(X)\otimes \varphi_{\xi, 0}(W,Y)]\right\rangle_{\mathcal{H}_{kl}}\\
    &\quad - \sum_{j=1}^{d_y}\xi_{0,j}(w,y)\langle f,a_j(w)\rangle_{\mathcal{H}_{kl}}\\
    &= \langle f, \Gamma_1(x,w,y)\rangle_{\mathcal{H}_{kl}},
\end{align*}
where we define
\begin{align*}
    \Gamma_1(x,w,y) := \varphi_k(x)\otimes \varphi_{\xi, 0}(w,y) - \mathbb{E}_0[\varphi_k(X)\otimes \varphi_{\xi, 0}(W,Y)] - \sum_{j=1}^{d_y}\xi_{0,j}(w,y)a_j(w).
\end{align*}

Note $\textcolor{blue}{\Psi_2(P)}$ is a \emph{kernel mean embedding} (see \citet[C.2.5]{luedtke2025simplifying}), since 
\begin{align*}
    \textcolor{blue}{\Psi_2(P)} = \int_{\mathcal{X}} k_X(x,\cdot)P_X(dx).
\end{align*} 
By \citet[Lemma S11]{luedtke2025simplifying}, $\textcolor{blue}{\Psi_2(P)}$ is pathwise differentiable at $P_0$, and its efficient influence operator $\dot{\textcolor{blue}{\Psi}}^{\ast}_{2,0} : \mathcal{H}_{k}\to \dot{\mathcal{M}}_0$ is given by
\begin{align*}
    \dot{\textcolor{blue}{\Psi}}^{\ast}_{2,0}(f)(x,w,y) = f(x) - \mathbb{E}_0[f(X)] = \langle f, \Gamma_2(x)\rangle_{\mathcal{H}_k},
\end{align*}
where we define
\begin{align*}
    \Gamma_2(x):= \varphi_k(x) - \mu_{X,0}. 
\end{align*}
Note that $\textcolor{teal}{\Psi_3(P)}$ can be obtained by modifying the primitive $\theta_5$ of $\textcolor{purple}{\Psi_1(P)}$ to 
\begin{align*}
    \widetilde\theta_5 :  \mathcal{M}\times L^2(P_{0};\mathbb{R}^{d_y}) &\rightarrow L^2(P_{0};\mathcal{H}_{l})\\
    (P,h_4) &\mapsto ((x,w,y)\mapsto  \varphi_{l}(h_4(x,w,y))).
\end{align*}

We conclude that $\textcolor{teal}{\Psi_3(P)}$ is pathwise differentiable at $P_0$, with efficient influence operator $\dot{\textcolor{teal}{\Psi}}_{3,0}^{\ast} : \mathcal{H}_{l}\to \dot{\mathcal{M}}_0$ given by
\begin{align*}
    \dot{\textcolor{teal}{\Psi}}_{3,0}^{\ast}(f)(x,w,y) &= \langle f, \varphi_{\xi, 0}(w,y) - \mu_{\xi,0}\rangle_{\mathcal{H}_{l}} - \sum_{j=1}^{d_y}\xi_{0,j}(w,y)\langle f,v_j(w)\rangle_{\mathcal{H}_{l}}\\
    &= \left\langle f, \Gamma_3(x,w,y)\right\rangle_{\mathcal{H}_{l}},
\end{align*}
where we define
\begin{align*}
    \Gamma_3(x,w,y):=\varphi_{\xi, 0}(w,y) - \mu_{\xi,0} - \sum_{j=1}^{d_y}\xi_{0,j}(w,y)v_j(w). 
\end{align*}
By product rule, for $s\in L^2_0(P_0)$, we have
\begin{align*}
    \dot{\textcolor{violet}{\Psi}}_{0}[s] = \dot{\textcolor{purple}{\Psi}}_{1,0}[s] - \dot{\textcolor{blue}{\Psi}}_{2,0}[s] \otimes \mu_{\xi,0} - \mu_{X,0} \otimes \dot{\textcolor{teal}{\Psi}}_{3,0}[s] \in \mathcal{H}_{kl}.
\end{align*}
Let $f\in \mathcal{H}_{kl}$ be arbitrary. We have
\begin{align*}
    \langle f, \dot{\textcolor{violet}{\Psi}}_{0}[s]\rangle_{\mathcal{H}_{kl}} &= \langle \dot{\textcolor{violet}{\Psi}}_{0}^{\ast}f, s\rangle_{L^2(P_{0})}\\
    &= \langle f, \dot{\textcolor{purple}{\Psi}}_{1,0}[s] - \dot{\textcolor{blue}{\Psi}}_{2,0}[s] \otimes \mu_{\xi,0} - \mu_{X,0} \otimes \dot{\textcolor{teal}{\Psi}}_{3,0}[s]\rangle_{\mathcal{H}_{kl}}.
\end{align*}
We have, for all elementary tensors $u\otimes v\in \mathcal{H}_{kl}$, 
\begin{align*}
    \langle u\otimes v, \dot{\textcolor{blue}{\Psi}}_{2,0}[s] \otimes \mu_{\xi,0}\rangle_{\mathcal{H}_{kl}} &= \langle u, \dot{\textcolor{blue}{\Psi}}_{2,0}[s]\rangle_{\mathcal{H}_k} \langle v, \mu_{\xi,0}\rangle_{\mathcal{H}_l}\\
    &= \langle \dot{\textcolor{blue}{\Psi}}_{2,0}^{\ast}u,s\rangle_{L^2(P_{0})}\langle v, \mu_{\xi,0}\rangle_{\mathcal{H}_l}\\
    &= \langle \dot{\textcolor{blue}{\Psi}}_{2,0}^{\ast}(\mathrm{Id}_{\mathcal{H}_k}\otimes_{\mathrm{op}} \mu_{\xi,0}^{\ast})(u\otimes v), s\rangle_{L^2(P_0)}.
\end{align*}
Hence, extending by linearity and continuity, we have for all $f\in \mathcal{H}_{kl}$,
\begin{align*}
    \langle f, \dot{\textcolor{blue}{\Psi}}_{2,0}[s] \otimes \mu_{\xi,0}\rangle_{\mathcal{H}_{kl}} &= \langle \dot{\textcolor{blue}{\Psi}}_{2,0}^{\ast}(\mathrm{Id}_{\mathcal{H}_k}\otimes_{\mathrm{op}} \mu_{\xi,0}^{\ast})f, s\rangle_{L^2(P_0)}\\
    \langle f,\mu_{X,0} \otimes \dot{\textcolor{teal}{\Psi}}_{3,0}[s]\rangle_{\mathcal{H}_{kl}} &= \langle \dot{\textcolor{teal}{\Psi}}_{3,0}^{\ast}( \mu_{X,0}^{\ast}\otimes_{\mathrm{op}}\mathrm{Id}_{\mathcal{H}_l})f, s\rangle_{L^2(P_0)}.
\end{align*}
Therefore, 
\begin{align*}
    \dot{\textcolor{violet}{\Psi}}_{0}^{\ast}(f)(x,w,y) &= \dot{\textcolor{purple}{\Psi}}_{1,0}^{\ast}(f)(x,w,y) - \dot{\textcolor{blue}{\Psi}}_{2,0}^{\ast}(\mathrm{Id}_{\mathcal{H}_k}\otimes_{\mathrm{op}} \mu_{\xi,0}^{\ast})f(x,w,y) - \dot{\textcolor{teal}{\Psi}}_{3,0}^{\ast}( \mu_{X,0}^{\ast}\otimes_{\mathrm{op}}\mathrm{Id}_{\mathcal{H}_l})f(x,w,y)\\
    &= \langle f, \Gamma_1(x,w,y)\rangle_{\mathcal{H}_{kl}} - \langle (\mathrm{Id}_{\mathcal{H}_k}\otimes_{\mathrm{op}} \mu_{\xi,0}^{\ast})f, \Gamma_2(x)\rangle_{\mathcal{H}_k} - \langle ( \mu_{X,0}^{\ast}\otimes_{\mathrm{op}}\mathrm{Id}_{\mathcal{H}_l})f, \Gamma_3(x,w,y)\rangle_{\mathcal{H}_l}\\
    &= \langle f, \Gamma_1(x,w,y) - (\mathrm{Id}_{\mathcal{H}_k}\otimes_{\mathrm{op}} \mu_{\xi,0}^{\ast})^{\ast}\Gamma_2(x) - ( \mu_{X,0}^{\ast}\otimes_{\mathrm{op}}\mathrm{Id}_{\mathcal{H}_l})^{\ast}\Gamma_3(x,w,y)\rangle_{\mathcal{H}_{kl}}. 
\end{align*}
Assuming that $\Gamma_1(x,w,y) - (\mathrm{Id}_{\mathcal{H}_k}\otimes_{\mathrm{op}} \mu_{\xi,0}^{\ast})^{\ast}\Gamma_2(x) - ( \mu_{X,0}^{\ast}\otimes_{\mathrm{op}}\mathrm{Id}_{\mathcal{H}_l})^{\ast}\Gamma_3(x,w,y)$ belongs to $L^2_0(P_0;\mathcal{H}_{kl})$, the parameter $\textcolor{violet}{\Psi}$ has an EIF at $P_0$ in the sense of \citet{luedtke2024one}, given by
\begin{align*}
    \phi_0(x,w,y) = \Gamma_1(x,w,y) - (\mathrm{Id}_{\mathcal{H}_k}\otimes_{\mathrm{op}} \mu_{\xi,0}^{\ast})^{\ast}\Gamma_2(x) - ( \mu_{X,0}^{\ast}\otimes_{\mathrm{op}}\mathrm{Id}_{\mathcal{H}_l})^{\ast}\Gamma_3(x,w,y). 
\end{align*}
We compute
\begin{align*}
    (\mathrm{Id}_{\mathcal{H}_k}\otimes_{\mathrm{op}} \mu_{\xi,0}^{\ast})^{\ast}\Gamma_2(x) &= (\mathrm{Id}_{\mathcal{H}_k}\otimes_{\mathrm{op}} \mu_{\xi,0}^{\ast})^{\ast}(\varphi_k(x) - \mu_{X,0})\\
    &= (\varphi_k(x) - \mu_{X,0})\otimes \mu_{\xi,0}\\
    &= \widetilde{\phi}_{X,0}(x)\otimes \mu_{\xi,0},
\end{align*}
where the second line can be verified by a simple adjoint identity. Similarly, we have
\begin{align*}
    ( \mu_{X,0}^{\ast}\otimes_{\mathrm{op}}\mathrm{Id}_{\mathcal{H}_l})^{\ast}\Gamma_3(x,w,y) &= ( \mu_{X,0}^{\ast}\otimes_{\mathrm{op}}\mathrm{Id}_{\mathcal{H}_l})^{\ast}\left(\varphi_{\xi, 0}(w,y) - \mu_{\xi,0} - \sum_{j=1}^{d_y}\xi_{0,j}(w,y)v_j(w)\right)\\
    &= \mu_{X,0}\otimes \left(\varphi_{\xi, 0}(w,y) - \mu_{\xi,0} - \sum_{j=1}^{d_y}\xi_{0,j}(w,y)v_j(w)\right). 
\end{align*}
Hence we have
\begin{align*}
    \mathcal{H}_{kl}\ni\phi_0(x,w,y) &= \varphi_k(x)\otimes \varphi_{\xi, 0}(w,y) - \mathbb{E}_0[\varphi_k(X)\otimes \varphi_{\xi, 0}(W,Y)] - \sum_{j=1}^{d_y}\xi_{0,j}(w,y)a_j(w)\\
    &\quad - \widetilde{\phi}_{X,0}(x)\otimes \mu_{\xi,0} - \mu_{X,0}\otimes \left(\varphi_{\xi, 0}(w,y) - \mu_{\xi,0} - \sum_{j=1}^{d_y}\xi_{0,j}(w,y)v_j(w)\right).
\end{align*}
Define
\begin{align*}
    F_j(w) &:= a_j(w)-\mu_{X,0}\otimes v_j(w)\\
    &= \mathbb{E}_0[\{\varphi_k(X)-\mu_{X,0}\}\otimes(\partial_j\varphi_l)(\xi_0(W,Y))\mid W=w]\in\mathcal{H}_{kl}.
\end{align*}
Since
\begin{align*}
    \textcolor{violet}{\Psi}_0 = \mathbb{E}_0[\varphi_k(X)\otimes \varphi_{\xi,0}(W,Y)]-\mu_{X,0}\otimes\mu_{\xi,0},
\end{align*}
we obtain
\begin{align*}
    \phi_0(x,w,y) = \widetilde{\phi}_{X,0}(x)\otimes\left(\varphi_{\xi, 0}(w,y)-\mu_{\xi,0}\right)-\sum_{j=1}^{d_y}\xi_{0,j}(w,y)F_j(w)-\textcolor{violet}{\Psi}_0. 
\end{align*}
\end{proof}

\section{Asymptotic linearity of the two-fold cross-fitted estimator}
\label{sec: appendix_one_step}
Fix $P\in\mathcal M$. Recall that
\[
m_P(w)=\mathbb E_P[Y\mid W=w],
\qquad
\xi_P(w,y)=y-m_P(w),
\]
and
\[
\mu_{X,P}=\mathbb E_P[\varphi_k(X)],
\qquad
\mu_{\xi,P}=\mathbb E_P[\varphi_l\{\xi_P(W,Y)\}].
\]
For $j=1,\ldots,d_y$, define
\[
v_{P,j}(w)
\doteq
\mathbb E_P[
(\partial_j\varphi_l)\{\xi_P(W,Y)\}
\mid W=w]
\in\mathcal H_l,
\qquad
v_P=(v_{P,1},\ldots,v_{P,d_y}).
\]
Set
\[
\eta_P=(m_P,v_P,\mu_{X,P},\mu_{\xi,P}).
\]
For $\eta=(m,v,\mu_X,\mu_\xi)$, define
\[
D_\eta(x,w,y)
\doteq
\{\varphi_k(x)-\mu_X\}
\otimes
\left[
\varphi_l\{y-m(w)\}-\mu_\xi
-
\sum_{j=1}^{d_y}\{y_j-m_j(w)\}v_j(w)
\right].
\]
Under \Cref{asst:nested_covariates}, the EIF in \Cref{thm:cco_eif} satisfies
\[
\phi_0(o)=D_{\eta_0}(o)-\Psi_0.
\]

Let $I_1,I_2$ be an equal split of $[n]$, and define
\[
\mathbb P_{n,r}\doteq \frac2n\sum_{i\in I_r}\delta_{o_i},
\qquad r=1,2.
\]
Let
\[
\widehat\eta_r
=
(\widehat m_r,\widehat v_r,\widehat\mu_{X,r},\widehat\mu_{\xi,r}),
\]
be fitted on $I_r$, where
\[
\widehat\mu_{X,r}
=
\mathbb P_{n,r}\varphi_k(X),
\qquad
\widehat\mu_{\xi,r}
=
\mathbb P_{n,r}\varphi_l\{Y-\widehat m_r(W)\}.
\]
The two-fold cross-fitted estimator is
\[
\check\Psi_n
\doteq
\frac12\{
\mathbb P_{n,1}D_{\widehat\eta_2}
+
\mathbb P_{n,2}D_{\widehat\eta_1}
\}.
\]

\begin{lemma}[Foldwise $L^2$-consistency of the estimated non-centered EIF]
\label{lem:foldwise_D_consistency}
Suppose \Cref{asst:kernel_boundedness,asst:asst_1,asst:nested_covariates,asst: differentiability_kernel_y,asst: boundedness_reg_fn} hold. Let
$\widehat\eta_r=(\widehat m_r,\widehat v_r,\widehat\mu_{X,r},\widehat\mu_{\xi,r})$, $r=1,2$, and define
\[
D_\eta(x,w,y)
\doteq
\{\varphi_k(x)-\mu_X\}
\otimes
\left[
\varphi_l\{y-m(w)\}-\mu_\xi
-\sum_{j=1}^{d_y}\{y_j-m_j(w)\}v_j(w)
\right].
\]
If, for $r=1,2$,
\[
\|\widehat m_r-m_0\|_{L^\infty(P_{0,W})}
+
\|\widehat m_r-m_0\|_{L^2(P_{0,W})}
+
\|\widehat v_r-v_0\|_{L^2(P_{0,W})}
=o_p(1),
\]
then
\[
\|\widehat\mu_{X,r}-\mu_{X,0}\|_{\mathcal H_k}=O_p(n^{-1/2}),
\qquad
\|\widehat\mu_{\xi,r}-\mu_{\xi,0}\|_{\mathcal H_l}=o_p(1),
\]
and
\[
\|D_{\widehat\eta_r}-D_{\eta_0}\|_{L^2(P_0;\mathcal H_{kl})}=o_p(1).
\]
\end{lemma}
\begin{proof}
The bound for $\widehat\mu_{X,r}$ follows from \Cref{asst:kernel_boundedness}. For $\widehat\mu_{\xi,r}$,
\[
\begin{aligned}
\|\widehat\mu_{\xi,r}-\mu_{\xi,0}\|_{\mathcal H_l}
&\leq
\left\|
\mathbb P_{n,r}
\{\varphi_l(Y-\widehat m_r(W))-\varphi_l(Y-m_0(W))\}
\right\|_{\mathcal H_l} \\
&\quad+
\left\|
\mathbb P_{n,r}\varphi_l(Y-m_0(W))
-
P_0\varphi_l(Y-m_0(W))
\right\|_{\mathcal H_l}.
\end{aligned}
\]
The first term is bounded by
$L\|\widehat m_r-m_0\|_{L^\infty(P_{0,W})}$, and the second is $o_p(1)$
by \Cref{asst:kernel_boundedness}. Hence
$\|\widehat\mu_{\xi,r}-\mu_{\xi,0}\|_{\mathcal H_l}=o_p(1)$.

Write
\[
A_{\eta}(X)\doteq \varphi_k(X)-\mu_X,
\qquad
M_{\eta}(O)\doteq
\varphi_l\{Y-m(W)\}-\mu_\xi
-\sum_{j=1}^{d_y}\{Y_j-m_j(W)\}v_j(W).
\]
Then
\[
D_{\widehat\eta_r}-D_{\eta_0}
=
(A_{\widehat\eta_r}-A_{\eta_0})\otimes M_{\eta_0}
+
A_{\widehat\eta_r}\otimes(M_{\widehat\eta_r}-M_{\eta_0}).
\]
By \Cref{asst:kernel_boundedness,asst:asst_1,asst: boundedness_reg_fn},
$M_{\eta_0}\in L^2(P_0;\mathcal H_l)$. Therefore
\[
\|(A_{\widehat\eta_r}-A_{\eta_0})\otimes M_{\eta_0}\|_{L^2(P_0)}
\lesssim
\|\widehat\mu_{X,r}-\mu_{X,0}\|_{\mathcal H_k}
=o_p(1).
\]
Moreover, $\|A_{\widehat\eta_r}\|_{\mathcal H_k}\leq 2\kappa_k$, and
\[
\begin{aligned}
\|M_{\widehat\eta_r}-M_{\eta_0}\|_{L^2(P_0)}
&\lesssim
\|\widehat m_r-m_0\|_{L^2(P_{0,W})}
+
\|\widehat\mu_{\xi,r}-\mu_{\xi,0}\|_{\mathcal H_l} \\
&\quad+
\|\xi_0^\top(\widehat v_r-v_0)\|_{L^2(P_0)}
+
\|(\widehat m_r-m_0)^\top \widehat v_r\|_{L^2(P_0)} .
\end{aligned}
\]
The third term is $o_p(1)$ by \Cref{asst: boundedness_reg_fn} and
$\|\widehat v_r-v_0\|_{L^2(P_{0,W})}=o_p(1)$. The last term is $o_p(1)$, since
$\widehat v_r=v_0+o_p(1)$ in $L^2(P_{0,W})$ and
$v_0\in L^\infty(P_{0,W};\mathcal H_l^{d_y})$ by \Cref{asst:asst_1}.
Thus $\|M_{\widehat\eta_r}-M_{\eta_0}\|_{L^2(P_0)}=o_p(1)$, proving the claim.
\end{proof}
\begin{lemma}[Upper bound for the remainder term in each fold]
\label{lem:foldwise_drift_bound}
Suppose \Cref{asst:kernel_boundedness,asst:asst_1,asst:nested_covariates,asst: differentiability_kernel_y} hold. For $\eta=(m,v,\mu_X,\mu_\xi)$, define
\[
R_0(\eta)\doteq P_0D_\eta-\Psi_0.
\]
Then, for $r=1,2$,
\[
\begin{aligned}
\|R_0(\widehat\eta_r)\|_{\mathcal H_{kl}}
&\lesssim
\|\widehat\mu_{X,r}-\mu_{X,0}\|_{\mathcal H_k}
\|\widehat\mu_{\xi,r}-\mu_{\xi,0}\|_{\mathcal H_l} \\
&\quad+
\|\widehat m_r-m_0\|_{L^2(P_{0,W})}^2 \\
&\quad+
\|\widehat m_r-m_0\|_{L^2(P_{0,W})}
\|\widehat v_r-v_0\|_{L^2(P_{0,W})}.
\end{aligned}
\]
Consequently, if
\[
\|\widehat m_r-m_0\|_{L^2(P_{0,W})}^2
+
\|\widehat m_r-m_0\|_{L^2(P_{0,W})}
\|\widehat v_r-v_0\|_{L^2(P_{0,W})}
=o_p(n^{-1/2}),
\]
then $R_0(\widehat\eta_r)=o_p(n^{-1/2})$.
\end{lemma}\begin{proof}
Let $\delta=m-m_0$, $\xi_0=Y-m_0(W)$, and
$A_{\mu_X}(W)=\varphi_k\{\pi(W)\}-\mu_X$. By
\Cref{asst:nested_covariates}, $A_{\mu_X}(X)=A_{\mu_X}(W)$ almost surely.
Using $\mathbb E_0[\xi_0\mid W]=0$,
\[
\begin{aligned}
R_0(\eta)
&=
(\mu_{X,0}-\mu_X)\otimes(\mu_{\xi,0}-\mu_\xi) \\
&\quad+
P_0\left[
A_{\mu_X}(W)
\otimes
\{\varphi_l(\xi_0-\delta(W))-\varphi_l(\xi_0)\}
\right] \\
&\quad+
P_0\left[
A_{\mu_X}(W)
\otimes
\sum_{j=1}^{d_y}\delta_j(W)v_j(W)
\right].
\end{aligned}
\]
Add and subtract
$\sum_{j=1}^{d_y}\delta_j(W)(\partial_j\varphi_l)(\xi_0)$. Since
\[
v_{0,j}(W)
=
\mathbb E_0[(\partial_j\varphi_l)(\xi_0)\mid W],
\]
the linear terms combine to
\[
P_0\left[
A_{\mu_X}(W)
\otimes
\sum_{j=1}^{d_y}\delta_j(W)\{v_j(W)-v_{0,j}(W)\}
\right].
\]
The Taylor remainder is bounded by \Cref{asst: differentiability_kernel_y}:
\[
\left\|
\varphi_l(\xi_0-\delta)
-
\varphi_l(\xi_0)
+
\sum_{j=1}^{d_y}\delta_j(\partial_j\varphi_l)(\xi_0)
\right\|_{\mathcal H_l}
\leq
L_2\|\delta\|_{\mathbb R^{d_y}}^2.
\]
Since $\|A_{\mu_X}(W)\|_{\mathcal H_k}\leq 2\kappa_k$, Cauchy-Schwarz gives
\[
\|R_0(\eta)\|_{\mathcal H_{kl}}
\lesssim
\|\mu_{X,0}-\mu_X\|_{\mathcal H_k}\|\mu_{\xi,0}-\mu_\xi\|_{\mathcal H_l}
+
\|m-m_0\|_{L^2(P_{0,W})}^2
+
\|m-m_0\|_{L^2(P_{0,W})}\|v-v_0\|_{L^2(P_{0,W})}.
\]
Taking $\eta=\widehat\eta_r$ gives the result. The final statement follows from
$\|\widehat\mu_{X,r}-\mu_{X,0}\|_{\mathcal H_k}=O_p(n^{-1/2})$ and
$\|\widehat\mu_{\xi,r}-\mu_{\xi,0}\|_{\mathcal H_l}=o_p(1)$, as shown in
\Cref{lem:foldwise_D_consistency}.
\end{proof}

\textbf{Analysis of the cross-fitted one-step estimator.}
\begin{theorem}[Two-fold cross-fitted asymptotic linearity]
\label{thm:appendix_two_fold_cf_asymp_linearity}
Let $o_1,\ldots,o_n\sim P_0^n$, and let $\check\Psi_n$ be defined above. Suppose \Cref{asst:kernel_boundedness,asst:asst_1,asst:nested_covariates,asst: differentiability_kernel_y,asst: boundedness_reg_fn} hold. In place of \Cref{asst: consistency,asst: suff_cond_an}, assume that, for $r=1,2$,
\[
\|\widehat m_r-m_0\|_{L^\infty(P_{0,W})}
+
\|\widehat m_r-m_0\|_{L^2(P_{0,W})}
+
\|\widehat v_r-v_0\|_{L^2(P_{0,W})}
=o_p(1),
\]
and
\[
\|\widehat m_r-m_0\|_{L^2(P_{0,W})}^2
+
\|\widehat m_r-m_0\|_{L^2(P_{0,W})}
\|\widehat v_r-v_0\|_{L^2(P_{0,W})}
=o_p(n^{-1/2}).
\]
Then
\[
\check\Psi_n-\Psi_0
=
\mathbb P_n\phi_0+o_p(n^{-1/2}),
\]
where $\phi_0=D_{\eta_0}-\Psi_0$ is the EIF in \Cref{thm:cco_eif}. Hence
\[
\sqrt n(\check\Psi_n-\Psi_0)
=
\frac1{\sqrt n}\sum_{i=1}^n\phi_0(o_i)+o_p(1)
\rightsquigarrow \mathbb H
\quad\text{in }\mathcal H_{kl},
\]
where $\mathbb H$ is a mean-zero Gaussian element in $\cH_{kl}$ with
\[
\mathrm{Var}\{\langle \mathbb H,h\rangle_{\mathcal H_{kl}}\}
=
\mathbb E_0[\langle \phi_0(O),h\rangle_{\mathcal H_{kl}}^2],
\qquad h\in\mathcal H_{kl}.
\]
\end{theorem}

\begin{proof}[Proof of \Cref{thm:appendix_two_fold_cf_asymp_linearity}]
Set $R_0(\eta)\doteq P_0D_\eta-\Psi_0$. Since $D_{\eta_0}-\Psi_0=\phi_0$,
\[
\begin{aligned}
\check\Psi_n-\Psi_0
&=
\mathbb P_n\phi_0 \\
&\quad+
\frac12(\mathbb P_{n,1}-P_0)(D_{\widehat\eta_2}-D_{\eta_0})
+
\frac12(\mathbb P_{n,2}-P_0)(D_{\widehat\eta_1}-D_{\eta_0}) \\
&\quad+
\frac12\{R_0(\widehat\eta_1)+R_0(\widehat\eta_2)\}.
\end{aligned}
\]
By \Cref{lem:foldwise_D_consistency} and sample splitting, we can control the stochastic equicontinuity term as follows
\[
\begin{aligned}
&\frac12(\mathbb P_{n,1}-P_0)(D_{\widehat\eta_2}-D_{\eta_0})
+
\frac12(\mathbb P_{n,2}-P_0)(D_{\widehat\eta_1}-D_{\eta_0}) \\
&\qquad =
o_p(n^{-1/2}).
\end{aligned}
\]
By \Cref{lem:foldwise_drift_bound},
\[
R_0(\widehat\eta_1)+R_0(\widehat\eta_2)=o_p(n^{-1/2}).
\]
Therefore
\[
\check\Psi_n-\Psi_0
=
\mathbb P_n\phi_0
+
o_p(n^{-1/2}).
\]
The weak convergence to $\mathbb{H}$ follows from apply a Hilbert-valued CLT, using the fact that $\cH_{kl}$ is separable and $\phi_0\in L_0^2(P_0;\mathcal H_{kl})$.
\end{proof}

\section{Computation}
\label{sec: computation}

\begin{center}
\small
\setlength{\tabcolsep}{5pt}
\renewcommand{\arraystretch}{1.25}
\begin{tabularx}{\linewidth}{@{}>{\raggedright\arraybackslash}p{0.27\linewidth}>{\raggedright\arraybackslash}X@{}}
\toprule
Notation & Meaning \\
\midrule
$\widehat A_i,\widehat C_i,\widehat M_i$ & Centered $X$ feature, centered residual feature, and corrected residual feature in the sample-split formulas. \\
\midrule
$\widehat D_i$, $G_{ii'}$ & Tensor-product one-step summand and its Gram entry. \\
\midrule
$\mathcal D^k,\mathcal D^{-k},n_k$ & Fold $k$, its complement, and fold size in the cross-fitted computation. \\
\midrule
$\mathbb P_k,\mathbb P_k^{(b)}$ & Empirical law on fold $k$ and its bootstrap resample. \\
\midrule
$z=(w,y)$ & Abbreviation used for the residual-side arguments. \\
\midrule
$\mg K$, $\mg L^{(-k_1,-k_2)}$ & Stored Gram matrices for the $X$ kernel and fold-specific residual kernel. \\
\midrule
$\mg r_X,\mg r_\xi,\mg r_{\xi,v},\mg r_v$ & Stored centered $X$ term, residual-feature term, residual--$v$ cross term, and $v$-diagonal term. \\
\midrule
$\partial L$, $\Delta L$ & Stored first- and mixed-second-derivative tensors for the residual kernel. \\
\midrule
$\Xi^{-k}$, $W^{-k}$ & Stored residual-coordinate tensor and vvKRR weight tensor. \\
\midrule
$\mathfrak w_j^{-k}$ & Scalar vvKRR weight for label coordinate $j$ trained outside fold $k$. \\
\bottomrule
\end{tabularx}
\end{center}

We first explain the computations on the simpler setting instantiated in \Cref{sec: estimator}. Then we extend this example to the cross-fitted estimators, where the computations
are described in additional detail. Since our experiments focus on the setting where $X$ is $W$-measurable (\Cref{asst:nested_covariates}), we describe this procedure, to facilitate
the reproduction of our results, which follow this exact computational scheme. The extension to general $W$ primarily differs in the need to estimate the $\cH_{kl}$-valued $F_j$ instead of the $\cH_l$-valued $v_j$, which further leads to slightly different expansions of the inner products analysed below. 
\subsection{Sample-split computations}
 For \(i\in I_1\cup I_2\), define
\(
    \widehat\xi_i
    \doteq
    y_i-\widehat m(w_i).
\)
The residuals \(\widehat\xi_i\) are evaluated for all \(i\in I_1\cup I_2\). The mean embeddings are computed on \(I_2\):
\[
    \widehat\mu_X
    \doteq
    \frac{2}{n}\sum_{t\in I_2}\varphi_k(x_t),
    \qquad
    \widehat\mu_\xi
    \doteq
    \frac{2}{n}\sum_{t\in I_2}\varphi_l(\widehat\xi_t).
\]
Thus \(\widehat\xi_i\) is used both at evaluation points \(i\in I_1\) and at centering points \(t\in I_2\).
Let \(\widehat v_j\), \(j=1,\ldots,d_y\), be the remaining nuisance estimates fitted on \(I_2\). For \(i\in I_1\), define
\[
    \widehat A_i
    \doteq
    \varphi_k(x_i)-\widehat\mu_X,
    \qquad
    \widehat C_i
    \doteq
    \varphi_l(\widehat\xi_i)-\widehat\mu_\xi,
    \qquad
    \widehat M_i
    \doteq
    \widehat C_i-\sum_{j=1}^{d_y}\widehat\xi_{i,j}\widehat v_j(w_i).
\]
Then
\[
\widehat D_i\doteq\widehat D(o_i)=\widehat A_i\otimes\widehat M_i.
\]
Hence, the tensor-product inner product
factorizes:
\[
     G_{ii'}
     \doteq 
     \langle
        \widehat D_i,\widehat D_{i'}
    \rangle_{\mathcal H_{kl}}
    =
    \langle
        \widehat A_i,\widehat A_{i'}
    \rangle_{\mathcal H_k}
    \langle
        \widehat M_i,\widehat M_{i'}
    \rangle_{\mathcal H_l}.
\]

Thus \(\widehat Q_V= \|\widehat\Psi_n\|_{\mathcal H_{kl}}^2 = \frac4{n^2}\sum_{i,i'\in I_1}G_{ii'}\) can be computed using only inner products in
\(\mathcal H_k\) and \(\mathcal H_l\). The residual-side inner product is
\[
\begin{aligned}
    \left\langle
        \widehat M_i,\widehat M_{i'}
    \right\rangle_{\mathcal H_l}
    &=
    \left\langle
        \widehat C_i,\widehat C_{i'}
    \right\rangle_{\mathcal H_l}
    -
    \sum_{j=1}^{d_y}
    \widehat\xi_{i',j}
    \left\langle
        \widehat C_i,\widehat v_j(w_{i'})
    \right\rangle_{\mathcal H_l}
    \\
    &\quad
    -
    \sum_{j=1}^{d_y}
    \widehat\xi_{i,j}
    \left\langle
        \widehat v_j(w_i),\widehat C_{i'}
    \right\rangle_{\mathcal H_l}
    +
    \sum_{j,r=1}^{d_y}
    \widehat\xi_{i,j}\widehat\xi_{i',r}
    \left\langle
        \widehat v_j(w_i),\widehat v_r(w_{i'})
    \right\rangle_{\mathcal H_l}.
\end{aligned}
\]
The first factor is the centered \(k\)-Gram entry
\[
\begin{aligned}
    \left\langle
        \widehat A_i,\widehat A_{i'}
    \right\rangle_{\mathcal H_k}
    =
    k(x_i,x_{i'})
    -
    \frac{2}{n}\sum_{t\in I_2}k(x_i,x_t)
    -
    \frac{2}{n}\sum_{t\in I_2}k(x_{i'},x_t)
    +
    \frac{4}{n^2}
    \sum_{s,t\in I_2}
    k(x_s,x_t).
\end{aligned}
\]
Similarly,
\[
\begin{aligned}
    \left\langle
        \widehat C_i,\widehat C_{i'}
    \right\rangle_{\mathcal H_l}
    =
    l(\widehat\xi_i,\widehat\xi_{i'})
    -
    \frac{2}{n}\sum_{t\in I_2}l(\widehat\xi_i,\widehat\xi_t)
    -
    \frac{2}{n}\sum_{t\in I_2}l(\widehat\xi_{i'},\widehat\xi_t)
    +
    \frac{4}{n^2}
    \sum_{s,t\in I_2}
    l(\widehat\xi_s,\widehat\xi_t).
\end{aligned}
\]
It remains to compute the terms involving \(\widehat v_j\). Let
\(
    R_{j,t}
    \doteq
    (\partial_j\varphi_l)(\widehat\xi_t), t\in I_2.
\)
The vector-valued KRR estimator has the representer form \citep{carmeli2006vector}
\[
    \widehat v_j(w)
    =
    \sum_{t\in I_2}
    \omega_{j,t}(w)R_{j,t},
\]
where, writing \(Q=(q(w_s,w_t))_{s,t\in I_2}\) and
\(q_w=(q(w,w_t))_{t\in I_2}\),
\(
    \omega_j(w)^\top
    =
    q_w^\top
    \left(Q+\frac n2\lambda_{j,n}I\right)^{-1}.
\)
Therefore,
\[
\begin{aligned}
    \left\langle
        \widehat C_i,\widehat v_j(w)
    \right\rangle_{\mathcal H_l}
    &=
    \sum_{t\in I_2}
    \omega_{j,t}(w)
    \left[
        \partial_{2,j}l(\widehat\xi_i,\widehat\xi_t)
        -
        \frac{2}{n}
        \sum_{s\in I_2}
        \partial_{2,j}l(\widehat\xi_s,\widehat\xi_t)
    \right],
\end{aligned}
\]
where \(\partial_{2,j}l\) denotes the derivative of \(l\) with respect to the \(j\)th
coordinate of its second argument. Finally,
\(
    \left\langle
        \widehat v_j(w),\widehat v_r(w')
    \right\rangle_{\mathcal H_l}
    =
    \sum_{s,t\in I_2}
    \omega_{j,s}(w)\omega_{r,t}(w')
    \partial_{1,j}\partial_{2,r}
    l(\widehat\xi_s,\widehat\xi_t).
\)
Thus the statistic can be computed from evaluations of \(k\), \(l\), first derivatives
\(\partial_{2,j}l\), mixed second derivatives
\(\partial_{1,j}\partial_{2,r}l\), and the scalar KRR weight matrices associated with
the kernel \(q\).

\subsection{Cross-fitted estimators and confidence sets.}
We move on to explain how to compute the point estimates and confidence sets in full-generality for $K$-fold cross-fitted estimators.

\textbf{Usage of superscripts and subscripts.} For an object which depends on estimated nuisances, we always use \emph{superscripts} to denote the fold used to fit the nuisances. In addition, we always use \emph{subscripts} to denote data points at which an object is evaluated, or indices.

\textbf{Usage of colour.} \mb{blue} denotes input/data, \mr{red} denotes intermediate mathematical objects, and \mg{green} denotes quantities that are stored and actually computed. As a rule, we never express objects in \mr{red} in terms of objects in \mg{green}. For objects in \mr{red}, we denote evaluation at data points as function evaluation, namely we use $(\cdot)$ in normal size. In contrast, we view objects in \mg{green} as tensors, therefore we consider points at which they are evaluated as indices, which are written inside parentheses in subscript. When \mb{blue} and \mg{green} clash, \mg{green} takes precedence.

Data $\mb{\mathcal{D}}=(X_i, Z_i := (W_i, Y_i))_{i\in [n]}$, split into $K$ folds $\mb{\mathcal{D}^k}$ of sizes $n_k$ and with complements $\mb{\mathcal{D}^{-k}}$.
\begin{asst}[Equal fold sizes]
\label{asst: equal fold size}
    Assume $K$ divides $n$ and all folds are of size $n_k=n/K$.
\end{asst}
\begin{rem}
    This is assumed purely for simplicity of exposition. It suffices that fold sizes are such that, there exist constants $c<C$ for which $c<\frac{n_k}{n} <C$
    for all sufficiently large $n$.
\end{rem}
The cross-fitted estimator of $\Psi_0$ is,
\begin{align*}
    \mr{\check\Psi_n} \doteq K\inv \sum_{k=1}^{K} \P_{k}\left(\mr{\widehat{D}^{-k}}\right).
\end{align*}
where $\P_k$ is the empirical law of $\mb{\mathcal{D}^k}$ and $\mr{\widehat D^{-k}}$ is fitted entirely using $\mb{\mathcal{D}^{-k}}$.
Further, we will denote $\P_k^{(b)}$ to be empirical measures obtained by bootstrapping $\P_k$ \citep{efron1992bootstrap}.
We shall denote $\P_k^{(0)}\doteq \P_k$.
A minimal objective is to compute quantities of the form
\(\mg{\left\langle \P_{k_1}^{(b_1)}\widehat D^{-k_1},\P_{k_2}^{(b_2)}\widehat D^{-k_2}\right\rangle_{\cH_{kl}}}\), since all
quantities that we need to compute are expressed in terms of inner products of (possibly bootstrapped) estimators of $\Psi_0$.

WLOG we will work with $b_1=b_2=0$ for the rest of the derivations.
Write \(z=(w,y)\).
Recall that for a set of nuisances $\eta$,
\[
A_\eta(x)\doteq \varphi_k(x)-\mu_X,\qquad
M_\eta(z)\doteq \varphi_l(\xi(z))-\mu_\xi-\sum_{j=1}^{d_y}\xi_j(z)v_j(w),
\qquad
D_\eta (x,z) = A_\eta(x)\otimes M_\eta(z).
\]
Our estimator should be a plug-in version of these. That is,
\[
\mr{\widehat A^{-k}}(x)\doteq \varphi_k(x)-\mr{\widehat \mu_X^{-k}},
\qquad
\mr{\widehat M^{-k}}(z)\doteq
\varphi_l\left(\mr{\widehat\xi^{-k}}(z)\right)-\mr{\widehat{\mu}_\xi^{-k}}
-\sum_{j=1}^{d_y}\mr{\widehat{\xi}^{-k}_j}(z)\mr{\widehat{v}^{-k}_j}(w).
\]
\[
\mr{\widehat D^{-k}}(x,z)=\mr{\widehat A^{-k}}(x)\otimes\mr{\widehat M^{-k}}(z).
\]
Quantities, such as,
\begin{align*}
    \mg{\|\check\Psi_n\|_{\cH_{kl}}^2}
    = \frac{1}{K^2}\sum_{k_1, k_2}\mg{\lra{\P_{k_1}\widehat D^{-k_1},\P_{k_2}\widehat D^{-k_2}}_{\cH_{kl}}},
\end{align*}
expand as averages of the following V-statistics,
\begin{align}
\label{eq: v_stat_ip_decomposition}
    \mg{\lra{\P_{k_1}\widehat D^{-k_1},\P_{k_2}\widehat D^{-k_2}}_{\cH_{kl}}}
    &= \frac1{n_{k_1}n_{k_2}}\sum_{\mb{(x_1,z_1)\in\mathcal{D}^{k_1}}}\sum_{\mb{(x_2,z_2)\in\mathcal{D}^{k_2}}} \lra{\mr{\widehat D^{-k_1}}(\mb{x_1},\mb{z_1}),\mr{\widehat D^{-k_2}}(\mb{x_2},\mb{z_2})}_{\cH_{kl}}.
\end{align}
We will explain how to compute those summands.
Since $\langle f_1\otimes g_1, f_2\otimes g_2\rangle_{\cH_{kl}} = \lra{f_1,f_2}_{\cH_k}\lra{g_1,g_2}_{\cH_l}$ \citep{aubin2000applied}, we isolate two kinds of terms out of these inner products.
\subsection{$X$-terms}
We define the empirical kernel mean embedding
\begin{align*}
    \mr{\widehat\mu_X^{-k}}\doteq \frac1{n-n_k}\sum_{\mb{x\in\mathcal{X}^{-k}}} \varphi_k(\mb{x}).
\end{align*}
Thus, for any $x_1, x_2$, we have
\begin{align*}
    \mg{r_{X, (x_1, x_2)}^{(-k_1, -k_2)}} &\doteq \lra{\varphi_k(x_1)-\mr{\widehat\mu_X^{-k_1}},\varphi_k(x_2)-\mr{\widehat\mu_X^{-k_2}}}_{\cH_k}\\
    &= k(x_1,x_2)- \frac1{n-n_{k_1}}\sum_{\mb{x\in\mathcal{X}^{-{k_1}}}}k(\mb{x},x_2) -\frac1{n-n_{k_2}}\sum_{\mb{x\in\mathcal{X}^{-{k_2}}}}k(\mb{x},x_1) \\
    &+ \frac 1{(n-n_{k_1})(n-n_{k_2})}\sum_{\substack{\mb{x\in\mathcal{X}^{-{k_1}}}\\\mb{x'\in\mathcal{X}^{-k_2}}}} k(\mb{x}, \mb{x'})\\
    &= \mg{K_{(x_1, x_2)}} - \frac{1}{n-n_{k_1}}1^{\top}_{n-n_{k_1}}\mg{K_{(-k_1, x_2)}} - \frac{1}{n- n_{k_2}}\mg{K_{(x_1, -k_2)}}1_{n-n_{k_2}}\\ &+ \frac{1}{(n-n_{k_1})(n - n_{k_2})}1_{n-n_{k_1}}^{\top}\mg{K_{(-k_1, -k_2)}}1_{n-n_{k_2}},
\end{align*}
where $\mg{K}\in \mathbb{R}^{n\times n}$ denote the full-data kernel matrix on $X$, namely $\mg{K_{(x_1, x_2)}} = k(x_1, x_2)$, and $1_{d}$ denote the $d\times 1$ vector of all ones, for any $d\geq 1$. $\mg{K_{(-k_1, x_2)}} \in \mathbb{R}^{(n-n_{k_1}) \times 1}$ denotes the matrix of $k(x, x_2)$, where $x$ are the $X$ observations in fold \(\mathcal{X}^{-k_1}\). The remaining notations are defined analogously.

\subsection{$\xi$-terms}
We continue to write \(z=(w,y)\), e.g. \(z_1=(w_1,y_1)\).

Analogously to the $X$-setting, we define \[
\widehat\mu_\xi^{-k} \doteq \frac1{n-n_k}\sum_{z\in\cZ^{-k}}\varphi_l\left\{\widehat\xi^{-k}(z)\right\}.
\]
For any $z$, we partition \(\mr{\widehat M^{-k}}(z)\in \cH_l\) into two summands,
\begin{align*}
    \varphi_l(\mr{\widehat\xi^{-k}}(z))-\mr{\widehat{\mu}_\xi^{-k}}, \qquad -\sum_{j=1}^{d_y}\underbrace{\mr{\widehat{\xi}^{-k}_j}(z)}_{\mathbb{R}}\underbrace{\mr{\widehat{v}^{-k}_j}(w)}_{\cH_l}.
\end{align*}
That yields two kinds of ``diagonal terms'' and one type of ``cross-term'' in the inner product decomposition
\begin{align*}
    \left\langle \mr{\widehat M^{-k_1}}(z_1), \mr{\widehat M^{-k_2}}(z_2)
    \right\rangle_{\cH_l}.
\end{align*}

\textbf{$\varphi_l$-diagonal terms.} Let $\mg{L^{(-k_1,-k_2)}}\in\R^{n\times n}$,
have entries
\[
\mg{L^{(-k_1,-k_2)}_{(z_1, z_2)}} \doteq\lra{\varphi_l\left(\mr{\widehat\xi^{-k_1}}(z_1)\right),\varphi_l\left(\mr{\widehat\xi^{-k_2}}(z_2)\right)}_{\cH_l}=l\left(\mr{\widehat\xi^{-k_1}}(z_1),\mr{\widehat\xi^{-k_2}}(z_2)\right).
\]
Then exactly like for the $X$-terms, we have
\begin{align*}
    \mg{r_{\xi, (z_1, z_2)}^{(-k_1, -k_2)}} &\doteq \lra{\varphi_l\left(\mr{\widehat\xi^{-k_1}}(z_1)\right)-\mr{\widehat\mu_\xi^{-k_1}},\varphi_l\left(\mr{\widehat\xi^{-k_2}}(z_2)\right)-\mr{\widehat\mu_\xi^{-k_2}}}_{\cH_l} \\
    &=\mg{L^{(-k_1,-k_2)}_{(z_1,z_2)}}-\frac1{n-n_{k_1}} \1_{n-n_{k_1}}\Tt \mg{L^{(-k_1,-k_2)}_{(-k_1,z_2)}} - \frac1{n-n_{k_2}} \mg{L^{(-k_1,-k_2)}_{(z_1,-k_2)}}\1_{n-n_{k_2}}\\
    &+ \frac1{(n-n_{k_1})(n-n_{k_2})} \1_{n-n_{k_1}}\Tt\mg{L^{(-k_1,-k_2)}_{(-k_1,-k_2)}}\1_{n-n_{k_2}},
\end{align*}
where $\mg{r_{\xi}^{(-k_1, -k_2)}}\in \mathbb{R}^{n_{k_1}\times n_{k_2}}$.

\textbf{Cross-terms.} We define
\begin{align}
\label{eq: r_xi_v_defn}
    \mg{r_{\xi,v, (z_1,z_2) }^{(-k_1,-k_2)}}
    \doteq
    \lra{\varphi_l\left(\mr{\widehat\xi^{-k_1}}(z_1)\right)-\mr{\widehat\mu_{\xi}^{-k_1}},\sum_{j=1}^{d_y}\mr{\widehat\xi_j^{-k_2}}(z_2)\mr{\widehat v^{-k_2}_j}(w_2)}_{\cH_l},
\end{align}
and an analogous term with $1\leftrightarrow 2$.
Suppose that $\mr{\widehat{v}^{-k}_j}$ is the solution to a vector-valued kernel ridge regression with separable kernel \citep{li2024towards}
\[
    Q(w_1,w_2)=q(w_1,w_2)\mathrm{Id}_{\cH_l}.
\]
Then it can be written as
\begin{align}
\label{eq: vkrr_sol_gen}
    \mr{\widehat{v}_j^{-k}}(w)
    =
    \sum_{\mb{z'\in \mathcal{Z}^{-k}}}
    \mr{\mathfrak{w}_j^{-k}}(\mb{z'},w)
    (\partial_j\varphi_l)\!\left(\mr{\widehat{\xi}^{-k}}(\mb{z'})\right),
\end{align}
where $\mb{z'}=(\mb{w'},\mb{y'})$ and
\[
 \mr{\mathfrak{w}_j^{-k}}(\mb{z'},w)
 =
 \left[
   \left(Q^{-k}+(n-n_k)\lambda_{j,n}I\right)^{-1}
   {q}^{-k}(w)
 \right]_{\mb{z'}},
\]
with
\[
({Q}^{-k})_{\mb{z'},\mb{z''}}=q(\mb{w'},\mb{w''}),
\qquad
({q}^{-k}(w))_{\mb{z''}}=q(\mb{w''},w).
\]

\begin{rem}
The vvKRR fit above is only one way to estimate $v_j^{-k}$. Any regression method can be used
provided that the inner products appearing below can be computed. A particularly useful variant is to
center the regression targets. Let
\[
    S_j^{-k}(\mb{z'})
    :=
    (\partial_j\varphi_l)\!\left(\mr{\widehat{\xi}^{-k}}(\mb{z'})\right),
    \qquad
    \bar S_j^{-k}
    :=
    \frac{1}{n-n_k}\sum_{\mb{z'}\in\cZ^{-k}}S_j^{-k}(\mb{z'}).
\]
If vvKRR is fit to the centered labels $S_j^{-k}(\mb{z'})-\bar S_j^{-k}$, we add the mean back:
\[
    \mr{\widehat v_j^{-k}}(w)
    =
    \bar S_j^{-k}
    +
    \sum_{\mb{z'}\in\cZ^{-k}}
    \mr{\widetilde{\mathfrak w}_j^{-k}}(\mb{z'},w)
    \{S_j^{-k}(\mb{z'})-\bar S_j^{-k}\}.
\]
Equivalently,
\[
    \mr{\widehat v_j^{-k}}(w)
    =
    \sum_{\mb{z'}\in\cZ^{-k}}
    \mr{\mathfrak w_{j,\mathrm{cen}}^{-k}}(\mb{z'},w)S_j^{-k}(\mb{z'}),
\]
where
\[
    \mr{\mathfrak w_{j,\mathrm{cen}}^{-k}}(\mb{z'},w)
    =
    \mr{\widetilde{\mathfrak w}_j^{-k}}(\mb{z'},w)
    +
    \frac{1-\sum_{\mb{u}\in\cZ^{-k}}
    \mr{\widetilde{\mathfrak w}_j^{-k}}(\mb{u},w)}
    {n-n_k}.
\]
Thus the formulas below are unchanged after replacing $\mathfrak w_j^{-k}$ by
$\mathfrak w_{j,\mathrm{cen}}^{-k}$. Centering is particulary helpful under the $\Psi_0=0$ null in a setting where $W=X$. There $v_{0,j}$ is constant. The constant component is then estimated by the empirical
mean and is not penalized by the ridge regression.
\end{rem}

We define
\begin{align*}
    \mg{\partial L^{(-k_1,-k_2)}_{j, (z_1, z_2)}} = \left(\partial^1_j l\right)\left(\mr{\widehat{\xi}^{-k_1}}(z_1), \mr{\widehat{\xi}^{-k_2}}(z_2)\right),
\end{align*}
where $\mg{\partial L^{(-k_1, -k_2)}} \in \mathbb{R}^{d_y \times n \times n}$ is a tensor of rank $3$, and $\partial_j^{1}$ denotes the partial derivative with respect to the $j$-th component of the first $\mathbb{R}^{d_y}$-valued argument. It is not in general symmetric in its second and third coordinates. We also emphasize that $\partial L$ should be understood as a single symbol. Substituting Eq.~\eqref{eq: vkrr_sol_gen} into Eq.~\eqref{eq: r_xi_v_defn} yields
\begin{align*}
    &\mg{r_{\xi, v, (z_1, z_2)}^{(-k_1, -k_2)}}\\
    &= \lra{\varphi_l\left(\mr{\widehat\xi^{-k_1}}(z_1)\right)-\mr{\widehat\mu_{\xi}^{-k_1}}, \sum_{j=1}^{d_y}\sum_{\mb{z'\in \mathcal{Z}^{-k_2}}}\mr{\mathfrak{w}_j^{-k_2}}(\mb{z'},w_2)(\partial_j\varphi_l)\left(\mr{\widehat{\xi}^{-k_2}}(\mb{z'})\right)\mr{\widehat{\xi}_j^{-k_2}}(z_2)}_{\cH_l}\\
    &= \sum_{j=1}^{d_y}\sum_{\mb{z'\in \mathcal{Z}^{-k_2}}} \mr{\mathfrak{w}_j^{-k_2}}(\mb{z'},w_2)\mr{\widehat{\xi}_j^{-k_2}}(z_2) \lra{(\partial_j\varphi_l)\left(\mr{\widehat{\xi}^{-k_2}}(\mb{z'})\right), \varphi_l\left(\mr{\widehat\xi^{-k_1}}(z_1)\right)-\mr{\widehat\mu_{\xi}^{-k_1}}}_{\cH_l}\\
    &= \sum_{j=1}^{d_y}\sum_{\mb{z'\in \mathcal{Z}^{-k_2}}} \mr{\mathfrak{w}_j^{-k_2}}(\mb{z'},w_2)\mr{\widehat{\xi}_j^{-k_2}}(z_2) \left\{\mg{\partial L^{(-k_2,-k_1)}_{j, (\mb{z'}, z_1)}} - \frac{1}{n-n_{k_1}}\mg{\partial L^{(-k_2, -k_1)}_{j, (z', -k_1)}}^{\top}1_{n-n_{k_1}}\right\}.
\end{align*}
next, we define tensors storing the quantities needed to $\mr{\mathfrak{w}_j^{-k_2}}(\mb{z'},w_2)\mr{\widehat{\xi}_j^{-k_2}}(z_2)$. Define a rank $3$-tensor $\mg{W^{-k}}\in \mathbb{R}^{d_y\times (n-n_{k})\times n}$ and rank $2$-tensor $\mg{\Xi^{-k}}\in \mathbb{R}^{d_y\times n}$ by
\begin{align*}
    \mg{\Xi^{-k}_{j, z}} &\doteq \mr{\widehat{\xi}^{-k}_{j}}(z)\\
    \mg{W^{-k}_{j, (z',w)}} &\doteq \mr{\mathfrak{w}_j^{-k}}(\mb{z'},w).
\end{align*}
We can therefore write
\begin{align*}
    &\mg{r_{\xi, v, (z_1, z_2)}^{(-k_1, -k_2)}}\\
    &= \sum_{j=1}^{d_y}\sum_{\mb{z'\in \mathcal{Z}^{-k_2}}} \mg{W^{-k_2}_{j, (\mb{z'}, w_2)}}\mg{\Xi^{-k_2}_{j, z_2}}\left\{\mg{\partial L^{(-k_2,-k_1)}_{j, (z', z_1)}} - \frac{1}{n-n_{k_1}}\mg{\partial L^{(-k_2, -k_1)}_{j, (z', -k_1)}}^{\top}1_{n-n_{k_1}}\right\}.
\end{align*}
The above tensor contraction can be performed using one line of Einstein summation in NumPy or JAX.

\textbf{$v$-diagonal terms.} We define the rank-$4$ tensor $\mg{\Delta L^{(-k_1, -k_2)}} \in \mathbb{R}^{d_y\times d_y \times (n - n_{k_1})\times (n - n_{k_2})}$ as follows
\begin{align*}
    \mg{\Delta L^{(-k_1, -k_2)}_{j_1, j_2, (z', z'')}} = \partial_{j_1}^{1}\partial_{j_2}^2 l \left(\mr{\widehat{\xi}^{-k_1}(z')}, \mr{\widehat{\xi}^{-k_2}(z'')}\right).
\end{align*}
Therefore we have
\begin{align*}
    &\mg{r_{v, (z_1, z_2)}^{(-k_1, -k_2)}}\\
    &= \left\langle\sum_{j_1=1}^{d_y}\mr{\widehat\xi_{j_1}^{-k_1}}(z_1)\mr{\widehat v^{-k_1}_{j_1}}(w_1), \sum_{j_2=1}^{d_y}\mr{\widehat\xi_{j_2}^{-k_2}}(z_2)\mr{\widehat v^{-k_2}_{j_2}}(w_2) \right\rangle_{\mathcal{H}_l}\\
    &= \sum_{j_1=1}^{d_y}\sum_{j_2 = 1}^{d_y} \mg{\Xi_{j_1, z_1}^{-k_1} \Xi_{j_2, z_2}^{-k_2}}\sum_{\substack{\mb{z'\in \mathcal{Z}^{-k_1}}\\ \mb{z''\in \mathcal{Z}^{-k_2}}}}\mg{W^{-k_1}_{j_1, (z', w_1)} W^{-k_2}_{j_2, (z'', w_2)} \Delta L_{j_1, j_2, (z', z'')}^{(-k_1,-k_2)}},
\end{align*}
which can again be computed using Einstein summation.

For consistency with \Cref{sec: estimator}, define the fold-pair Gram entry
\[
\mg{G_{((x_1,z_1),(x_2,z_2))}^{(-k_1,-k_2)}}
\doteq
\mg{r_{X, (x_1, x_2)}^{(-k_1, -k_2)}}
\left\{
\mg{r_{\xi, (z_1, z_2)}^{(-k_1, -k_2)}}
- \mg{r_{\xi, v, (z_1, z_2)}^{(-k_1, -k_2)}}
- \mg{r_{\xi, v, (z_2, z_1)}^{(-k_2, -k_1)}}
+ \mg{r_{v, (z_1, z_2)}^{(-k_1, -k_2)}}
\right\}.
\]
Putting everything together, we have
\begin{align}
\label{eq: v_stat_all_green}
    &\mg{\lra{\P_{k_1}\widehat D^{-k_1},\P_{k_2}\widehat D^{-k_2}}_{\cH_{kl}}}\\
    &= \frac{1}{n_{k_1}n_{k_2}}\sum_{\substack{\mb{(x_1,z_1)\in \mathcal{D}^{k_1}}\\ \mb{(x_2,z_2)\in \mathcal{D}^{k_2}}}}\mg{G_{((x_1,z_1),(x_2,z_2))}^{(-k_1,-k_2)}}.
\end{align}

\section{Inference}
\label{sec:inference}

\begin{center}
\small
\setlength{\tabcolsep}{5pt}
\renewcommand{\arraystretch}{1.25}
\begin{tabularx}{\linewidth}{@{}>{\raggedright\arraybackslash}p{0.27\linewidth}>{\raggedright\arraybackslash}X@{}}
\toprule
Notation & Meaning \\
\midrule
$\mathcal O$, $\Omega_n$, $\Omega_0$ & Class of positive self-adjoint operators; empirical and population choices for the operator used to construct a confidence set. \\
\midrule
$w(h;\Omega)$, $\mathcal C_n(\zeta)$ & Quadratic form and resulting Hilbert-space confidence set. \\
\midrule
$\mathbb H$, $\mathbb H_n^{(b)}$ & Gaussian limit and bootstrap empirical-process draw. \\
\midrule
$\zeta_{1-\alpha}$, $\widehat\zeta_n$ & Population and bootstrap-estimated quantiles for the chosen confidence-set geometry. \\
\midrule
$\sigma^2$, $\widehat\sigma^2$ & Delta-method variance component and its plug-in estimator. \\
\midrule
$\check\Psi_n$ & Cross-fitted one-step estimator of \(\Psi_0\). \\
\midrule
$\widehat D^{-k}$ & Uncentered EIF estimate fitted without fold \(k\). \\
\midrule
$D_{\eta_0}$, $\phi_0$ & Population uncentered and centered EIFs, with \(\phi_0=D_{\eta_0}-\Psi_0\). \\
\midrule
$G_{ij}$ & Cross-fitted Gram entry $\langle\widehat D^{-k(i)}(O_i),\widehat D^{-k(j)}(O_j)\rangle_{\mathcal H_{kl}}$, where \(k(i)\) is the validation fold containing \(O_i\). \\
\midrule
$\widehat Q_V$, $\widehat Q_U$ & V-statistic and U-statistic center $\|\check\Psi_n\|_{\mathcal H_{kl}}^2$ and its diagonal-term-free analogue. \\
\bottomrule
\end{tabularx}
\end{center}

We build on \citet[Section 4.2]{luedtke2024one} to construct \((1-\alpha)\)-confidence sets centered at the one-step estimator. Let \(\mathcal O\) denote the set of continuous, self-adjoint, positive-definite linear operators from \(\mathcal H_{kl}\) to \(\mathcal H_{kl}\). For \(\Omega\in\mathcal O\), define
\[
    w(h;\Omega):=\langle h,\Omega h\rangle_{\mathcal H_{kl}}.
\]
Let \(\Omega_n\in\mathcal O\) be an empirical estimator of some possibly \(P_0\)-dependent operator \(\Omega_0\in\mathcal O\). For a threshold \(\zeta\geq 0\), define
\[
    \mathcal C_n(\zeta):=\left\{h\in\mathcal H_{kl}:w(\check\Psi_n-h;\Omega_n)\leq \frac{\zeta}{n}\right\}.
\]
A consistent estimator of the \((1-\alpha)\)-quantile \(\zeta_{1-\alpha}\) of \(w(\mathbb H;\Omega_0)\) is obtained by Efron's bootstrap \citep{efron1992bootstrap}. We instantiate the bootstrap in our notation. Here \(\mathbb P_k:=n_k^{-1}\sum_{i:k(i)=k}\delta_{O_i}\) is the empirical law of fold \(k\), and \(n_k=n/K\). For \(b\in[B]\) and \(k\in[K]\), draw
\[
    O_{1,k}^{(b)},\dots,O_{n_k,k}^{(b)}\overset{\mathrm{i.i.d.}}{\sim}\mathbb P_k.
\]
The bootstrap observations are sampled independently of everything else. Let \(\mathbb P_k^{(b)}\) be the empirical law of \(O_{1,k}^{(b)},\dots,O_{n_k,k}^{(b)}\). For \(b\in[B]\), define
\[
    \mathbb H_n^{(b)}:=\sqrt n\,\frac{1}{K}\sum_{k=1}^K(\mathbb P_k^{(b)}-\mathbb P_k)\widehat D^{-k}\in\mathcal H_{kl}.
\]
We define \(\widehat\zeta_n\) as the empirical \((1-\alpha)\)-quantile of \(w(\mathbb H_n^{(b)};\Omega_n)\) over \(b=1,\dots,B\), conditionally on the original sample. The uncentered EIF \(\widehat D^{-k}\) is not refitted inside each bootstrap iteration. The following proposition is a direct application of \citet[Theorem 4]{luedtke2024one}.

\begin{prop}
\label{prop: luedtke_bootstrap_consistency}
Suppose the conditions of \Cref{thm:appendix_two_fold_cf_asymp_linearity} hold. Further suppose that \(\Omega_n,\Omega_0\in\mathcal O\), \(\|\Omega_n-\Omega_0\|_{\mathrm{op}}=o_p(1)\), and
\[
    \max_{k\in[K]}\left\|\widehat D^{-k}-D_{\eta_0}\right\|_{L^2(P_0;\mathcal H_{kl})}=o_p(1).
\]
Then \(\widehat\zeta_n\overset{p}{\to}\zeta_{1-\alpha}\).
\end{prop}

In this paper, we focus on the case \(\Omega_n=\Omega_0=\mathrm{Id}_{\mathcal H_{kl}}\), for which
\[
    w(\mathbb H_n^{(b)};\Omega_n)=\|\mathbb H_n^{(b)}\|_{\mathcal H_{kl}}^2.
\]
This quantity is computed using the V-statistic in \Cref{sec: computation}, with \(\mathbb P_k\) replaced by \(\mathbb P_k^{(b)}\) in the validation folds. Equivalently, the same foldwise summands \(\langle\widehat D^{-k_1}(o_1),\widehat D^{-k_2}(o_2)\rangle_{\mathcal H_{kl}}\) are used, with bootstrap resampling applied only to the outer validation sums. The nuisance estimators \(\widehat m^{-k}\), \(\widehat\mu_X^{-k}\), \(\widehat\mu_\xi^{-k}\), and \(\widehat v_j^{-k}\) are not refitted inside bootstrap draws.

Our primary objective is to benchmark the proposed method against naive methods which are not robust to nuisance estimation error, and which may therefore be biased when flexible machine learning estimators are used. It remains an open question whether the procedure can be improved by data-driven choices of \(\Omega_n\). A particularly natural choice is the regularized Wald-type confidence set described in \citet{luedtke2024one}. We defer these investigations to future work.

\textbf{Confidence intervals based on triangle inequality.}
Let \(k(i)\) denote the validation fold containing \(O_i\), and define
\[
    G_{ij}
    :=
    \left\langle
        \widehat D^{-k(i)}(O_i),
        \widehat D^{-k(j)}(O_j)
    \right\rangle_{\mathcal H_{kl}} .
\]
Then the plug-in squared norm of the cross-fitted estimator is the V-statistic
\begin{equation}
\label{eq:qv_crossfit}
    \widehat Q_V
    :=
    \|\check\Psi_n\|_{\mathcal H_{kl}}^2
    =
    \frac{1}{n^2}
    \sum_{i,j=1}^n
    G_{ij}.
\end{equation}
The quantities \(G_{ij}\) are precisely the foldwise Gram entries computed in
\Cref{sec: computation}.

If \(N_i^{(b)}\) denotes the number of times \(O_i\) appears in the bootstrap draw of
its validation fold, then, for \(\Omega_n=\mathrm{Id}_{\mathcal H_{kl}}\),
\[
    \|\mathbb H_n^{(b)}\|_{\mathcal H_{kl}}^2
    =
    \frac{1}{n}
    \sum_{i,j=1}^n
    \{N_i^{(b)}-1\}\{N_j^{(b)}-1\}
    G_{ij}.
\]

For \(\Omega_n=\mathrm{Id}_{\mathcal H_{kl}}\), the confidence set is
\[
    \mathcal C_n(\zeta)
    =
    \left\{
        h\in\mathcal H_{kl}:
        \|\check\Psi_n-h\|_{\mathcal H_{kl}}^2
        \leq \frac{\zeta}{n}
    \right\}.
\]
Hence, by the reverse triangle inequality, it induces the conservative interval
\[
    \mathcal I_{\zeta}
    :=
    \left[
        \left\{
            \sqrt{\widehat Q_V}
            -
            \sqrt{\zeta/n}
        \right\}_+,
        \sqrt{\widehat Q_V}
        +
        \sqrt{\zeta/n}
    \right].
\]
In implementation we take \(\zeta=\widehat\zeta_n\). The induced test of
\(H_0:\Psi_0=0\) rejects when \(0\notin\mathcal I_{\widehat\zeta_n}\), equivalently when
\(n\widehat Q_V>\widehat\zeta_n\).

\subsection{CI based on the delta method}
\label{sec: delta_method}

Beyond testing \(\Psi_0=0\), we may be interested in the magnitude of the signal, namely \(\|\Psi_0\|_{\mathcal H_{kl}}\) or \(\|\Psi_0\|_{\mathcal H_{kl}}^2\). When \(\Psi_0\neq 0\), the functional delta method \citep[Chapter 3]{van2000asymptotic} gives an asymptotic normal approximation. Under the assumptions of \Cref{thm:appendix_two_fold_cf_asymp_linearity},
\[
    \sqrt n(\check\Psi_n-\Psi_0)\rightsquigarrow\mathbb H.
\]
Since \(h\mapsto\|h\|_{\mathcal H_{kl}}^2\) is Fréchet differentiable with derivative \(u\mapsto 2\langle h,u\rangle_{\mathcal H_{kl}}\), we have
\[
    \sqrt n\left(\|\check\Psi_n\|_{\mathcal H_{kl}}^2-\|\Psi_0\|_{\mathcal H_{kl}}^2\right)\rightsquigarrow 2\langle\Psi_0,\mathbb H\rangle_{\mathcal H_{kl}}.
\]
The limiting variance is \(4\sigma^2\), where
\[
    \sigma^2:=\mathbb E_0\left[\langle\Psi_0,\phi_0(O)\rangle_{\mathcal H_{kl}}^2\right]
    =
    \mathbb E_0\left[\langle\Psi_0,D_{\eta_0}(O)\rangle_{\mathcal H_{kl}}^2\right]-\|\Psi_0\|_{\mathcal H_{kl}}^4.
\]
This motivates the estimator
\begin{equation}
    \label{eq:sigma_hat_estimator_sq}
    \begin{aligned}
        \widehat\sigma^2
    &:=\frac{1}{K}\sum_{l=1}^K\frac{1}{n_l}\sum_{\nocolourb{o'\in\mathcal D^l}}\left\langle\check\Psi_n,\widehat D^{-l}(\nocolourb{o'})\right\rangle_{\mathcal H_{kl}}^2-\|\check\Psi_n\|_{\mathcal H_{kl}}^4\\
    &=\frac{1}{K}\sum_{l=1}^K\frac{1}{n_l}\sum_{\nocolourb{o'\in\mathcal D^l}}\left(\frac{1}{K}\sum_{k=1}^K\frac{1}{n_k}\sum_{\nocolourb{o\in\mathcal D^k}}\left\langle\widehat D^{-k}(\nocolourb{o}),\widehat D^{-l}(\nocolourb{o'})\right\rangle_{\mathcal H_{kl}}\right)^2\\
    &\quad-\left(\frac{1}{K^2}\sum_{k=1}^K\sum_{l=1}^K\frac{1}{n_kn_l}\sum_{\substack{\nocolourb{o\in\mathcal D^k}\\ \nocolourb{o'\in\mathcal D^l}}}\left\langle\widehat D^{-k}(\nocolourb{o}),\widehat D^{-l}(\nocolourb{o'})\right\rangle_{\mathcal H_{kl}}\right)^2.
    \end{aligned}
\end{equation}
The inner products \(\langle\widehat D^{-k}(\nocolourb{o}),\widehat D^{-l}(\nocolourb{o'})\rangle_{\mathcal H_{kl}}\) are precisely the quantities computed in \Cref{sec: computation}.

\begin{prop}[Consistency of the estimator of asymptotic variance]
\label{prop: sigma_consistency}
Suppose the assumptions of \Cref{thm:appendix_two_fold_cf_asymp_linearity} hold, and assume \(K\) is fixed and \(n_k=n/K\) for all \(k\in[K]\). Suppose further that
\[
    \max_{k\in[K]}\|\widehat D^{-k}-D_{\eta_0}\|_{L^2(P_0;\mathcal H_{kl})}=o_p(1),
\]
and that
\[
    \max_{k\in[K]}\int \|\widehat D^{-k}(o)\|_{\cH_{kl}}^4\,dP_0(o)=O_p(1),\qquad \int \|D_{\eta_0}(o)\|_{\cH_{kl}}^4\,dP_0(o)<\infty.
\]
Define
\begin{align*}
    \widehat\sigma^2
    :=
    \frac1K\sum_{l=1}^K\frac1{n_l}\sum_{\nocolourb{o'\in\mathcal D^l}}\langle \check\Psi_n,\widehat D^{-l}(\nocolourb{o'})\rangle_{\cH_{kl}}^2-\|\check\Psi_n\|_{\cH_{kl}}^4.
\end{align*}
Then
\[
    \widehat\sigma^2\overset{p}{\to}\sigma^2,
\]
where
\[
    \sigma^2:=\mathbb E_0[\langle \Psi_0,\phi_0(O)\rangle_{\cH_{kl}}^2]
    =
    \mathbb E_0[\langle \Psi_0,D_{\eta_0}(O)\rangle_{\cH_{kl}}^2]-\|\Psi_0\|_{\cH_{kl}}^4.
\]
\end{prop}
\begin{proof}
By \Cref{thm:appendix_two_fold_cf_asymp_linearity}, \(\check\Psi_n\overset{p}{\to}\Psi_0\). Hence, by the continuous mapping theorem,
\[
    \|\check\Psi_n\|_{\cH_{kl}}^4\overset{p}{\to}\|\Psi_0\|_{\cH_{kl}}^4.
\]
It remains to show that
\[
    \frac1K\sum_{l=1}^K\frac1{n_l}\sum_{\nocolourb{o'\in\mathcal D^l}}\langle \check\Psi_n,\widehat D^{-l}(\nocolourb{o'})\rangle_{\cH_{kl}}^2
    \overset{p}{\to}
    \mathbb E_0[\langle \Psi_0,D_{\eta_0}(O)\rangle_{\cH_{kl}}^2].
\]
We consider the decomposition
\begin{align}
    &\frac1K\sum_{l=1}^K\frac1{n_l}\sum_{\nocolourb{o'\in\mathcal D^l}}\langle \check\Psi_n,\widehat D^{-l}(\nocolourb{o'})\rangle_{\cH_{kl}}^2-\mathbb E_0[\langle \Psi_0,D_{\eta_0}(O)\rangle_{\cH_{kl}}^2]\nonumber\\
    &=\frac1K\sum_{l=1}^K\frac1{n_l}\sum_{\nocolourb{o'\in\mathcal D^l}}\left\{\langle \check\Psi_n,\widehat D^{-l}(\nocolourb{o'})\rangle_{\cH_{kl}}^2-\langle \Psi_0,\widehat D^{-l}(\nocolourb{o'})\rangle_{\cH_{kl}}^2\right\}\tag{A}\label{eq: sigma_consistency_term_A}\\
    &\quad+\frac1K\sum_{l=1}^K\frac1{n_l}\sum_{\nocolourb{o'\in\mathcal D^l}}\left\{\langle \Psi_0,\widehat D^{-l}(\nocolourb{o'})\rangle_{\cH_{kl}}^2-\langle \Psi_0,D_{\eta_0}(\nocolourb{o'})\rangle_{\cH_{kl}}^2\right\}\tag{B}\label{eq: sigma_consistency_term_B}\\
    &\quad+\frac1K\sum_{l=1}^K\frac1{n_l}\sum_{\nocolourb{o'\in\mathcal D^l}}\langle \Psi_0,D_{\eta_0}(\nocolourb{o'})\rangle_{\cH_{kl}}^2-\mathbb E_0[\langle \Psi_0,D_{\eta_0}(O)\rangle_{\cH_{kl}}^2].\tag{C}\label{eq: sigma_consistency_term_C}
\end{align}
For term \(\eqref{eq: sigma_consistency_term_A}\), we use \(a^2-b^2=(a-b)(a+b)\). We have
\begin{align*}
    |\eqref{eq: sigma_consistency_term_A}|
    &\leq \frac1K\sum_{l=1}^K\frac1{n_l}\sum_{\nocolourb{o'\in\mathcal D^l}}\left|\langle \check\Psi_n-\Psi_0,\widehat D^{-l}(\nocolourb{o'})\rangle_{\cH_{kl}}\right|\left|\langle \check\Psi_n+\Psi_0,\widehat D^{-l}(\nocolourb{o'})\rangle_{\cH_{kl}}\right|\\
    &\leq \|\check\Psi_n-\Psi_0\|_{\cH_{kl}}\|\check\Psi_n+\Psi_0\|_{\cH_{kl}}\frac1K\sum_{l=1}^K\frac1{n_l}\sum_{\nocolourb{o'\in\mathcal D^l}}\|\widehat D^{-l}(\nocolourb{o'})\|_{\cH_{kl}}^2.
\end{align*}
We now show that the final factor is \(O_p(1)\). Conditionally on the training sample \(\mathcal D^{-l}\), the observations in \(\mathcal D^l\) are i.i.d. under \(P_0\) and independent of \(\widehat D^{-l}\). Therefore, by Chebyshev's inequality,
\begin{align*}
    \frac1{n_l}\sum_{\nocolourb{o'\in\mathcal D^l}}\|\widehat D^{-l}(\nocolourb{o'})\|_{\cH_{kl}}^2-\int\|\widehat D^{-l}(o)\|_{\cH_{kl}}^2\,dP_0(o)=o_p(1),
\end{align*}
because
\[
    \int\|\widehat D^{-l}(o)\|_{\cH_{kl}}^4\,dP_0(o)=O_p(1),
\]
and \(n_l\to\infty\). Since \(K\) is fixed, averaging over \(l\) gives
\[
    \frac1K\sum_{l=1}^K\left\{\frac1{n_l}\sum_{\nocolourb{o'\in\mathcal D^l}}\|\widehat D^{-l}(\nocolourb{o'})\|_{\cH_{kl}}^2-\int\|\widehat D^{-l}(o)\|_{\cH_{kl}}^2\,dP_0(o)\right\}=o_p(1).
\]
On the other hand,
\begin{align*}
    \frac1K\sum_{l=1}^K\int\|\widehat D^{-l}(o)\|_{\cH_{kl}}^2\,dP_0(o)
    &\leq \frac2K\sum_{l=1}^K\|\widehat D^{-l}-D_{\eta_0}\|_{L^2(P_0;\cH_{kl})}^2+2\|D_{\eta_0}\|_{L^2(P_0;\cH_{kl})}^2\\
    &=O_p(1).
\end{align*}
Therefore
\[
    \frac1K\sum_{l=1}^K\frac1{n_l}\sum_{\nocolourb{o'\in\mathcal D^l}}\|\widehat D^{-l}(\nocolourb{o'})\|_{\cH_{kl}}^2=O_p(1).
\]
Since \(\check\Psi_n\overset{p}{\to}\Psi_0\), we have
\[
    |\eqref{eq: sigma_consistency_term_A}|=o_p(1).
\]

For term \(\eqref{eq: sigma_consistency_term_B}\), we again use \(a^2-b^2=(a-b)(a+b)\). We have
\begin{align*}
    |\eqref{eq: sigma_consistency_term_B}|
    &\leq \frac1K\sum_{l=1}^K\frac1{n_l}\sum_{\nocolourb{o'\in\mathcal D^l}}\|\Psi_0\|_{\cH_{kl}}^2\|\widehat D^{-l}(\nocolourb{o'})-D_{\eta_0}(\nocolourb{o'})\|_{\cH_{kl}}\|\widehat D^{-l}(\nocolourb{o'})+D_{\eta_0}(\nocolourb{o'})\|_{\cH_{kl}}\\
    &\leq \|\Psi_0\|_{\cH_{kl}}^2
    \left\{\frac1K\sum_{l=1}^K\frac1{n_l}\sum_{\nocolourb{o'\in\mathcal D^l}}\|\widehat D^{-l}(\nocolourb{o'})-D_{\eta_0}(\nocolourb{o'})\|_{\cH_{kl}}^2\right\}^{1/2}\\
    &\quad\times
    \left\{\frac1K\sum_{l=1}^K\frac1{n_l}\sum_{\nocolourb{o'\in\mathcal D^l}}\|\widehat D^{-l}(\nocolourb{o'})+D_{\eta_0}(\nocolourb{o'})\|_{\cH_{kl}}^2\right\}^{1/2}.
\end{align*}
Conditionally on \(\mathcal D^{-l}\), another application of Chebyshev's inequality gives
\begin{align*}
    &\frac1K\sum_{l=1}^K\left\{\frac1{n_l}\sum_{\nocolourb{o'\in\mathcal D^l}}\|\widehat D^{-l}(\nocolourb{o'})-D_{\eta_0}(\nocolourb{o'})\|_{\cH_{kl}}^2-\int\|\widehat D^{-l}(o)-D_{\eta_0}(o)\|_{\cH_{kl}}^2\,dP_0(o)\right\}=o_p(1),
\end{align*}
and similarly with \(+\) in place of \(-\). Hence
\begin{align*}
    \frac1K\sum_{l=1}^K\frac1{n_l}\sum_{\nocolourb{o'\in\mathcal D^l}}\|\widehat D^{-l}(\nocolourb{o'})-D_{\eta_0}(\nocolourb{o'})\|_{\cH_{kl}}^2
    &=\frac1K\sum_{l=1}^K\|\widehat D^{-l}-D_{\eta_0}\|_{L^2(P_0;\cH_{kl})}^2+o_p(1)=o_p(1),
\end{align*}
while
\begin{align*}
    \frac1K\sum_{l=1}^K\frac1{n_l}\sum_{\nocolourb{o'\in\mathcal D^l}}\|\widehat D^{-l}(\nocolourb{o'})+D_{\eta_0}(\nocolourb{o'})\|_{\cH_{kl}}^2=O_p(1).
\end{align*}
Therefore
\[
    |\eqref{eq: sigma_consistency_term_B}|=o_p(1).
\]

Finally, for term \(\eqref{eq: sigma_consistency_term_C}\), since \(D_{\eta_0}\in L^2(P_0;\cH_{kl})\), we have
\[
    \mathbb E_0[\langle \Psi_0,D_{\eta_0}(O)\rangle_{\cH_{kl}}^2]\leq \|\Psi_0\|_{\cH_{kl}}^2\mathbb E_0[\|D_{\eta_0}(O)\|_{\cH_{kl}}^2]<\infty.
\]
Therefore, by the weak law of large numbers,
\[
    \eqref{eq: sigma_consistency_term_C}
    =
    \mathbb P_n\{\langle \Psi_0,D_{\eta_0}(O)\rangle_{\cH_{kl}}^2\}-\mathbb E_0[\langle \Psi_0,D_{\eta_0}(O)\rangle_{\cH_{kl}}^2]=o_p(1).
\]
Combining the bounds for \(\eqref{eq: sigma_consistency_term_A}\), \(\eqref{eq: sigma_consistency_term_B}\), and \(\eqref{eq: sigma_consistency_term_C}\), we obtain
\[
    \frac1K\sum_{l=1}^K\frac1{n_l}\sum_{\nocolourb{o'\in\mathcal D^l}}\langle \check\Psi_n,\widehat D^{-l}(\nocolourb{o'})\rangle_{\cH_{kl}}^2
    \overset{p}{\to}
    \mathbb E_0[\langle \Psi_0,D_{\eta_0}(O)\rangle_{\cH_{kl}}^2].
\]
Together with \(\|\check\Psi_n\|_{\cH_{kl}}^4\overset{p}{\to}\|\Psi_0\|_{\cH_{kl}}^4\), this proves
\[
    \widehat\sigma^2\overset{p}{\to}\mathbb E_0[\langle \Psi_0,D_{\eta_0}(O)\rangle_{\cH_{kl}}^2]-\|\Psi_0\|_{\cH_{kl}}^4=\sigma^2.
\]
\end{proof}

\begin{theorem}[Validity and efficiency of the delta method]
\label{thm: delta_method_validity}
Suppose the assumptions of \Cref{thm:appendix_two_fold_cf_asymp_linearity} hold. Then, in \(\mathcal H_{kl}\),
\[
    \sqrt n(\check\Psi_n-\Psi_0)\rightsquigarrow \mathbb H,
\]
where \(\mathbb H\) is the centered Gaussian element in \Cref{thm:appendix_two_fold_cf_asymp_linearity}. Consequently,
\[
    \sqrt n\left(\|\check\Psi_n\|_{\cH_{kl}}^2-\|\Psi_0\|_{\cH_{kl}}^2\right)\rightsquigarrow 2\langle \Psi_0,\mathbb H\rangle_{\cH_{kl}}.
\]
Equivalently,
\[
    \sqrt n\left(\|\check\Psi_n\|_{\cH_{kl}}^2-\|\Psi_0\|_{\cH_{kl}}^2\right)\rightsquigarrow N(0,4\sigma^2),
\]
where
\[
    \sigma^2:=\mathbb E_0[\langle \Psi_0,\phi_0(O)\rangle_{\cH_{kl}}^2].
\]
The scalar efficient influence function of \(P\mapsto \|\Psi(P)\|_{\cH_{kl}}^2\) at \(P_0\) is
\[
    2\langle \Psi_0,\phi_0(O)\rangle_{\cH_{kl}}.
\]
Therefore \(\|\check\Psi_n\|_{\cH_{kl}}^2\) is regular and efficient for \(\|\Psi_0\|_{\cH_{kl}}^2\). If \(\Psi_0\neq 0\), then
\[
    \sqrt n\left(\|\check\Psi_n\|_{\cH_{kl}}-\|\Psi_0\|_{\cH_{kl}}\right)\rightsquigarrow N\left(0,\frac{\sigma^2}{\|\Psi_0\|_{\cH_{kl}}^2}\right),
\]
and the scalar efficient influence function of \(P\mapsto \|\Psi(P)\|_{\cH_{kl}}\) at \(P_0\) is
\[
    \frac{\langle \Psi_0,\phi_0(O)\rangle_{\cH_{kl}}}{\|\Psi_0\|_{\cH_{kl}}}.
\]
\end{theorem}
\begin{proof}
The first display is exactly \Cref{thm:appendix_two_fold_cf_asymp_linearity}. The map \(T:\mathcal H_{kl}\to\mathbb R\) defined by \(T(h)=\|h\|_{\cH_{kl}}^2\) is Fréchet differentiable at every \(h\in\mathcal H_{kl}\), with derivative
\[
    \dot T_h[u]=2\langle h,u\rangle_{\cH_{kl}}.
\]
The functional delta method therefore gives
\[
    \sqrt n\left(\|\check\Psi_n\|_{\cH_{kl}}^2-\|\Psi_0\|_{\cH_{kl}}^2\right)\rightsquigarrow 2\langle \Psi_0,\mathbb H\rangle_{\cH_{kl}}.
\]
Since \(\mathbb H\) is centered Gaussian and \(\langle\Psi_0,\cdot\rangle_{\cH_{kl}}\) is a continuous linear functional,
\[
    2\langle \Psi_0,\mathbb H\rangle_{\cH_{kl}}\sim N(0,4\sigma^2),
\]
where
\[
    \sigma^2=\mathbb E_0[\langle \Psi_0,\phi_0(O)\rangle_{\cH_{kl}}^2].
\]
The efficient influence function follows by the chain rule for pathwise differentiable Hilbert-valued parameters:
\[
    \phi_0^\Theta(O)=\dot T_{\Psi_0}[\phi_0(O)]=2\langle \Psi_0,\phi_0(O)\rangle_{\cH_{kl}}.
\]
Since \(\phi_0\) is the efficient influence function for \(\Psi\), this scalar influence function is efficient for \(P\mapsto \|\Psi(P)\|_{\cH_{kl}}^2\).

If \(\Psi_0\neq 0\), the map \(S:\mathcal H_{kl}\to\mathbb R\) defined by \(S(h)=\|h\|_{\cH_{kl}}\) is Fréchet differentiable at \(\Psi_0\), with derivative
\[
    \dot S_{\Psi_0}[u]=\frac{\langle \Psi_0,u\rangle_{\cH_{kl}}}{\|\Psi_0\|_{\cH_{kl}}}.
\]
Another application of the functional delta method gives
\[
    \sqrt n\left(\|\check\Psi_n\|_{\cH_{kl}}-\|\Psi_0\|_{\cH_{kl}}\right)\rightsquigarrow \frac{\langle \Psi_0,\mathbb H\rangle_{\cH_{kl}}}{\|\Psi_0\|_{\cH_{kl}}},
\]
which is centered Gaussian with variance \(\sigma^2/\|\Psi_0\|_{\cH_{kl}}^2\). The efficient influence function for \(P\mapsto \|\Psi(P)\|_{\cH_{kl}}\) follows by the same chain rule.
\end{proof}

\textbf{Delta-method implementation with reduced finite-sample bias.}
We contrast the plug-in V-statistic defined in \Cref{eq:qv_crossfit} with the corresponding U-statistic $\widehat{Q}_U$ by removing the
self-pairs \citep{serfling2009approximation}.
\[\widehat Q_U
	\doteq  \frac{1}{n(n-1)}\sum_{i\neq j}
	G_{ij}.\]
Under the assumptions of \Cref{thm:appendix_two_fold_cf_asymp_linearity}, we have $\widehat Q_U - \widehat Q_V =
	O_p(n^{-1})$. Therefore, when $\Psi_0\neq 0$, both centers have the same delta-method limit since $\sqrt{n}(\widehat Q_U - \widehat Q_V)\tod 0$. In
terms of asymptotic validity of the confidence sets, we can use whichever one we prefer. Since the positive bias of $\widehat Q_V$ is not negligible
in small sample settings, we suggest to center the delta-method confidence sets on $\widehat{Q}_U$, which leads to empirically better performance. However, the same argument does not allow for a substitution of $\widehat{Q}_U$ into the triangle-inequality-based interval. This is because it relies on the approximation
\[
	n\lr{\sqrt{\widehat{Q}_V}-\|\Psi_0\|_{\mathcal H_{kl}}}^2\tod \|\mathbb H\|_{\mathcal H_{kl}}^2.
\]
The $O_p(n^{-1})$ error is not negligible at this scale. We therefore propose to rely on $\widehat{Q}_V$ in the construction of the triangle-inequality-based interval. We
	believe that constructing analogous confidence sets with better small-sample behaviour and analysing them, is an interesting direction for future
	work.
\section{Experiment details}
\label{sec:experiment_details}

\begin{center}
\small
\setlength{\tabcolsep}{5pt}
\renewcommand{\arraystretch}{1.25}
\begin{tabularx}{\linewidth}{@{}>{\raggedright\arraybackslash}p{0.27\linewidth}>{\raggedright\arraybackslash}X@{}}
\toprule
Notation & Meaning \\
\midrule
$n$ & Sample size used in experiment figures and tables. \\
\midrule
$s_m$, $s_\epsilon$ & Fourier-decay parameters controlling smoothness of the mean and noise-scale functions. \\
\midrule
$\rho$ & Strength of residual heterogeneity; $\rho=0$ is homogeneous noise. \\
\midrule
$a_k,b_k,\phi_{m,k},\phi_{\sigma,k}$ & Random Fourier coefficients and phases in the synthetic data-generating mechanisms. \\
\midrule
$c_m,c_\epsilon,\sigma_0$ & Scaling constants for the mean/noise series and the baseline noise scale. \\
\midrule
$g(x),h(x),c$ & Auxiliary functions and scale constant used in heteroscedastic noise examples. \\
\midrule
$m_{Y\mid X}$, $m_{X\mid Y}$ & Direction-specific regressions in the causal-arrow experiment. \\
\midrule
$T$, $\sigma(0)$, $\sigma(1)$ & Extra covariate and group-specific residual scales in the covariate-adjusted experiment. \\
\bottomrule
\end{tabularx}
\end{center}

Throughout our experiments, the core nuisance functions $\hat m$ and $\hat v$ were obtained via $5$-fold cross-fitting,
using kernel ridge regression (vector-valued where appropriate). We found that Mat\'ern kernels perform slightly better than smoother
alternatives at learning $v$. Permutation tests are calibrated using 1000 permutations. Bootstrap-based procedures use 1000 bootstrap samples. 
All experiments can be performed on a single machine with a single NVIDIA RTX 4500 GPU and 64GB of RAM. To accelerate the evaluation of our 
method across a range of settings and random seeds, we parallelised the computations across multiple machines. Each one had 64GB of RAM, and multiple different GPUs were used, but
the GPU capacity is not a core bottleneck. We did not evaluate how quickly our procedures can be run on a CPU. On the aforementioned RTX 4500 GPU, fitting a 5-fold cross-fitted confidence set 
in $d_x=d_y=1$ and with 1000 data points took on the order of 1 to 2 minutes. The core complexity of the procedures scales quadratically with sample size, as is typical for kernel methods.
We did not investigate ways of mitigating this empirically; however, since the implementation is expressed through Gram matrices and their derivative analogues, standard kernel-acceleration 
ideas such as random Fourier features, Nystr{\"o}m approximations, and incomplete or block $U$-statistics are natural candidates for future scalability work \citep{rahimi2007random,williams2001using,zhang2018large,schrab2022efficient}. 

\subsection{Datasets used in the synthetic calibration-evaluation experiments}
\label{sec:experiment_details_synthetic_data}

This subsection gives the scalar data-generating mechanism used for the reported calibration-evaluation experiments.
The same random Fourier construction can be extended componentwise to multivariate $X$ or $Y$, but the reported main experiments keep $d_x=d_y=1$ to focus on residual-estimation difficulty.

The data-generating process is\[
Y=m(X) + \sigma(X)Z,
\]
where $Z\sim \cN(0,1)$ is sampled independently of $X$.
The structural function used to construct the synthetic datasets is \[
 m(x) = c_m \sum_{k=1}^{20} a_k k^{-s_m} \sin(k x + \phi_{m,k}),
\]
with $a_k\stackrel{\iid}{\sim}\cN(0,1)$, $\phi_{m,k}\stackrel{\iid}{\sim}\mathrm{Uniform}[-\pi,\pi]$, and $c_m$ chosen post-hoc so that $\Var(m(X))=1$.
Similarly,
\begin{align*}
    g(x)& = \sum_{k=1}^{20} b_k k^{-s_\epsilon} \sin(k x + \phi_{\sigma, k}),\\
    \sigma(x) &= \sigma_0 \exp (\rho c_\epsilon g(x)),
\end{align*}
where $\rho\ge0$ (a value of 0 corresponds to homogeneous noise with standard deviation $\sigma_0$). Weights $b_k$ and phases $\phi_{\sigma,k}$ are sampled as above.
$c_\epsilon$ is chosen to ensure $\Var(g(X)c_\epsilon)=1$, and in all our experiments $\sigma_0=0.35$.

The HSIC kernels are Gaussian with bandwidths tuned by the median heuristic, applied to full $X$ for the $\cH_k$ kernel,
and $\hat\xi$, the estimated residuals, on the $\cH_l$ kernel.
Example datasets are shown in \Cref{fig:synthetic_dataset_examples}.
We use the classical permutation HSIC test \citep{gretton2007kernel} as the plug-in baseline because it is the canonical
residual-independence diagnostic in ANM experiments.
Improvements to the underlying raw independence test, such as \citet{schrab2022efficient}, are unrelated to the debiasing problem studied here, though it would be an interesting
direction for future work to extend our estimation methodology to functionals that aggregate over cross-covariance operators in multiple RKHSs.
\begin{figure}[t]
    \centering
    \begin{subfigure}[t]{0.49\linewidth}
        \centering
        \includegraphics[width=\linewidth,trim=0mm 0mm 0mm 7mm,clip]{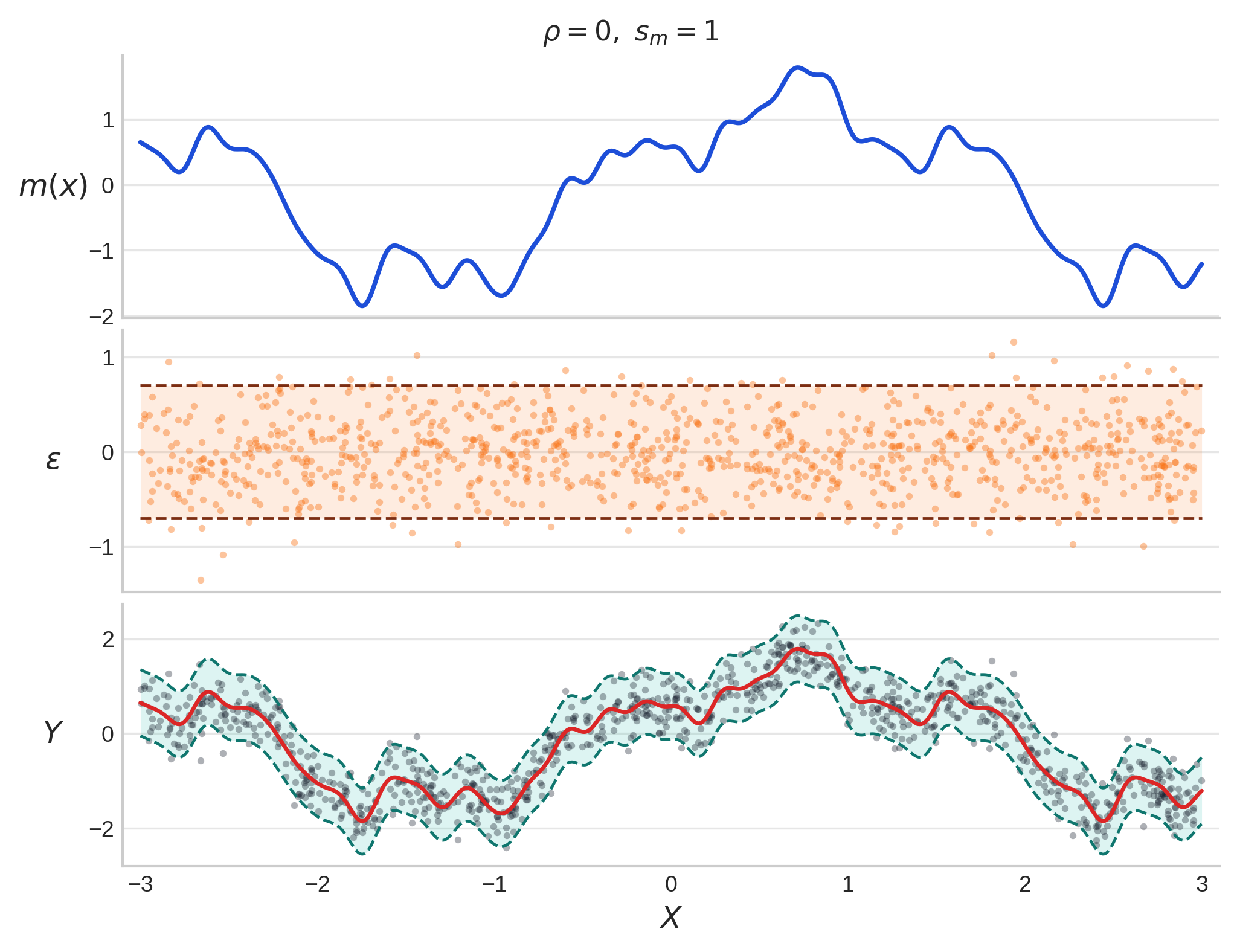}
        \caption{$s_m=1$, $\rho=0$}
        \label{fig:synthetic_dataset_example_sm1_null}
    \end{subfigure}
    \hfill
    \begin{subfigure}[t]{0.49\linewidth}
        \centering
        \includegraphics[width=\linewidth,trim=0mm 0mm 0mm 7mm,clip]{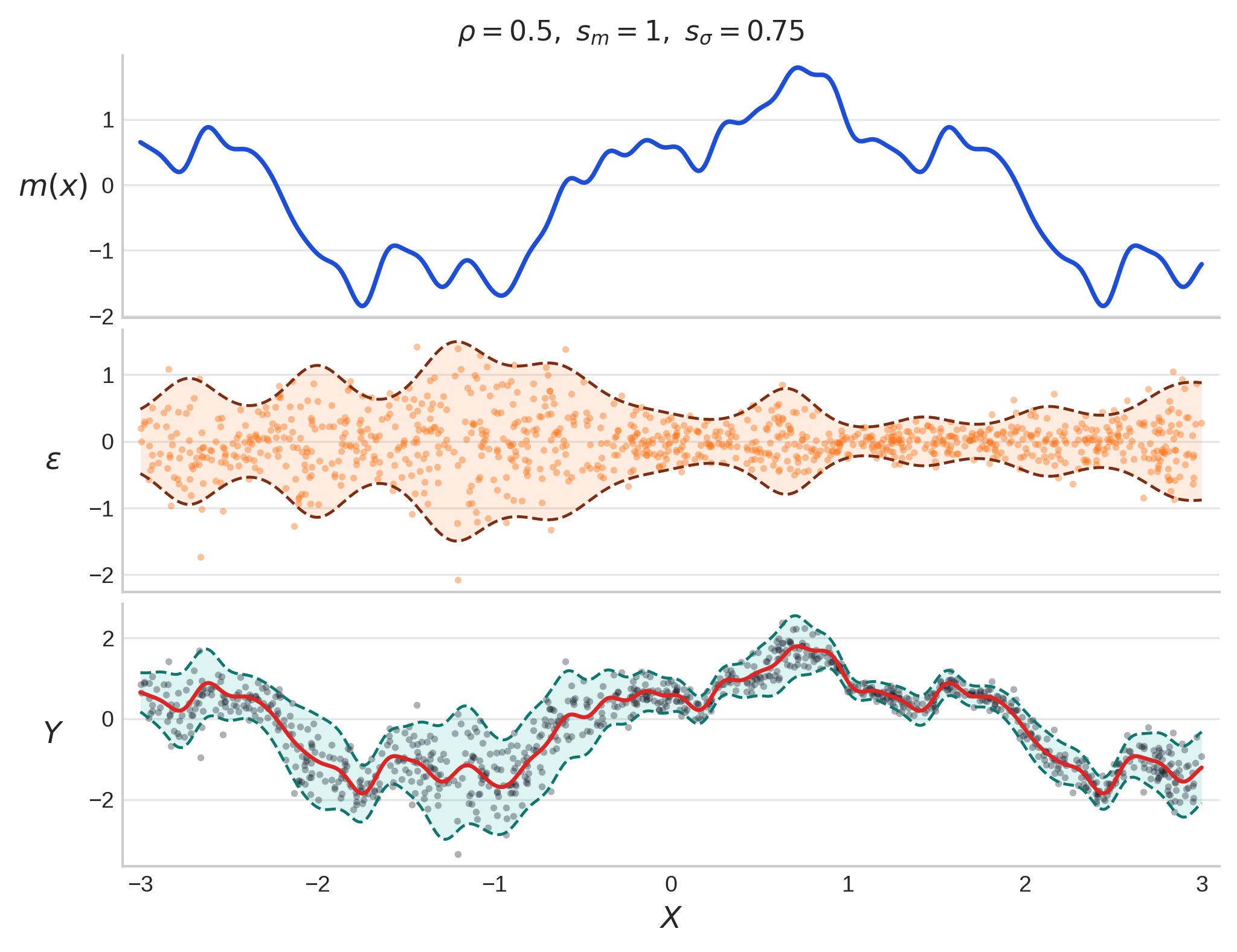}
        \caption{$s_m=1$, $\rho=0.5$, $s_\epsilon=0.75$}
        \label{fig:synthetic_dataset_example_sm1_alt}
    \end{subfigure}
    \par\smallskip
    \begin{subfigure}[t]{0.49\linewidth}
        \centering
        \includegraphics[width=\linewidth,trim=0mm 0mm 0mm 7mm,clip]{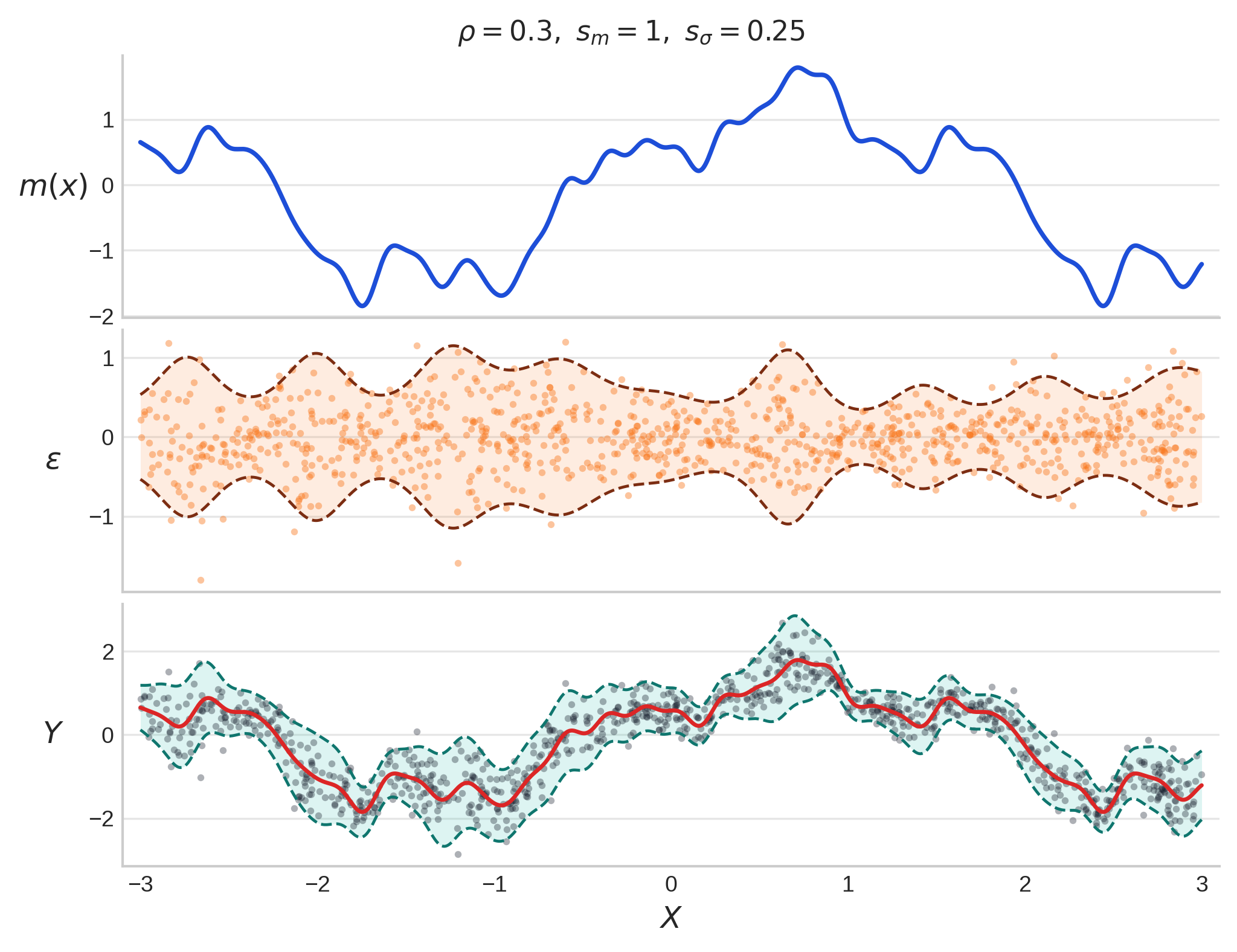}
        \caption{$s_m=1$, $\rho=0.3$, $s_\epsilon=0.25$}
        \label{fig:synthetic_dataset_example_sm2_alt_rough}
    \end{subfigure}
    \hfill
    \caption{Example synthetic datasets. Panel (a) is a homogeneous-noise null setting; panels (b)--(c) are alternatives with $X$-dependent noise scale.}
    \label{fig:synthetic_dataset_examples}
    \label{fig:synthetic_dataset_examples_sm1}
    \label{fig:synthetic_dataset_examples_sm2}
\end{figure}

\subsection{Normal convergence and coverage of the triangle inequality confidence sets.}
\label{sec:experiment_details_coverage}
We report empirical coverage for the nominal $95\%$ confidence intervals in the same synthetic setting used for the main calibration plot, namely $s_m=1$ and $s_\epsilon=0.75$.
The reverse-triangle-inequality intervals are valid at the null and remain conservative under alternatives.
This is visible in \Cref{tab:experiment_details_reverse_triangle_coverage_main}: coverage is close to nominal when $\rho=0$,
but becomes essentially one as the signal increases. The delta-method intervals are not guaranteed to be valid at $\rho=0$, where the delta method does not apply, but under alternatives they are much less conservative.
For the debiased estimator, \Cref{tab:experiment_details_delta_coverage_main} shows coverage generally close to nominal, with the weakest performance at the smallest sample size and weakest alternative.

The QQ plots in \Cref{fig:experiment_details_qq_calibration} tell the same story. Away from $\rho=0$, the debiased statistic tracks the Gaussian reference line reasonably well, especially in the central quantiles.
The plug-in statistic exhibits visible bias and slope distortions, while the debiased statistic is substantially better aligned with the target.
At $\rho=0$, the normal approximation is not expected to hold because of the norm functional's vanishing derivative; the departures in that panel are therefore a diagnostic of the nonregular null rather than a failure of the alternative-regime delta method.

\begin{figure}[p]
    \centering
    \includegraphics[
        width=\textwidth,
        height=0.72\textheight,
        keepaspectratio
    ]{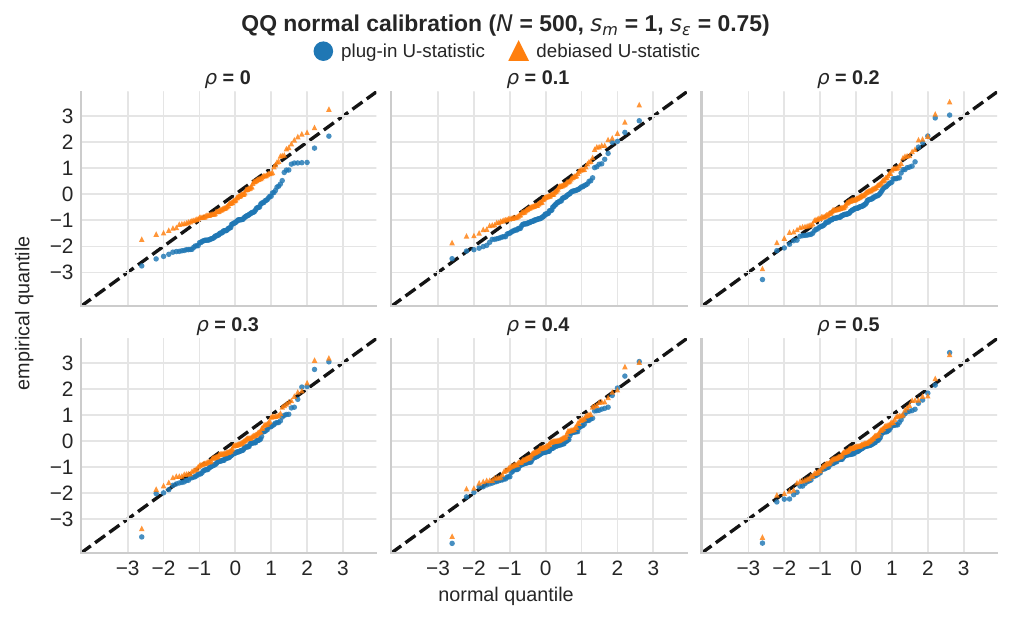}

    \par\vspace{0.5em}
    {\small\emph{Note.} At $\rho=0$ there is no asymptotic normality under the null, so the QQ plot is not expected to fit the Gaussian reference line in that panel.}

    \caption{QQ normal calibration of the plug-in and debiased estimators for the $X$-dependent residual experiment at $n=500$, across noise-dependence levels $\rho\in\{0,0.1,0.2,0.3,0.4,0.5\}$. The estimators are centered at the population-level $\|\Psi\|$ and scaled to have unit second moment.}
    \label{fig:experiment_details_qq_calibration}
\end{figure}

\vspace{1.25\baselineskip}

\begin{center}
    \begin{minipage}{\linewidth}
        \centering
        \captionof{table}{Coverage of the reverse-triangle-inequality confidence intervals for $s_m=1$ and $s_\epsilon=0.75$. Parentheses give 95\% error bars.}
        \label{tab:experiment_details_reverse_triangle_coverage_main}
        \scriptsize
        \resizebox{0.82\linewidth}{!}{%
        \begin{tabular}{@{}cccc@{}}
            \toprule
            $\rho$ & $n=250$ & $n=500$ & $n=1000$ \\
            \midrule
            0 & 0.930 (-0.050, +0.050) & 0.955 (-0.039, +0.039) & 0.916 (-0.053, +0.053) \\
            0.1 & 0.970 (-0.033, +0.030) & 0.982 (-0.025, +0.018) & 0.991 (-0.018, +0.009) \\
            0.2 & 0.990 (-0.020, +0.010) & 0.991 (-0.018, +0.009) & 1.000 (-0.000, +0.000) \\
            0.3 & 1.000 (-0.000, +0.000) & 1.000 (-0.000, +0.000) & 1.000 (-0.000, +0.000) \\
            0.4 & 1.000 (-0.000, +0.000) & 1.000 (-0.000, +0.000) & 1.000 (-0.000, +0.000) \\
            0.5 & 1.000 (-0.000, +0.000) & 1.000 (-0.000, +0.000) & 1.000 (-0.000, +0.000) \\
            \bottomrule
        \end{tabular}}
    \end{minipage}
\end{center}

\vspace{1.25\baselineskip}

\begin{center}
    \begin{minipage}{\linewidth}
        \centering
        \captionof{table}{Coverage of the debiased delta-method confidence intervals for $s_m=1$ and $s_\epsilon=0.75$. The delta-method interval is only reported for $\rho>0$. Parentheses give 95\% error bars.}
        \label{tab:experiment_details_delta_coverage_main}
        \scriptsize
        \resizebox{0.82\linewidth}{!}{%
        \begin{tabular}{@{}cccc@{}}
            \toprule
            $\rho$ & $n=250$ & $n=500$ & $n=1000$ \\
            \midrule
            0.1 & 0.911 (-0.083, +0.083) & 0.973 (-0.037, +0.027) & 0.929 (-0.055, +0.055) \\
            0.2 & 0.935 (-0.061, +0.061) & 0.970 (-0.034, +0.030) & 0.952 (-0.041, +0.041) \\
            0.3 & 0.938 (-0.053, +0.053) & 0.981 (-0.026, +0.019) & 0.943 (-0.044, +0.044) \\
            0.4 & 0.953 (-0.045, +0.045) & 0.981 (-0.026, +0.019) & 0.943 (-0.044, +0.044) \\
            0.5 & 0.933 (-0.052, +0.052) & 0.963 (-0.036, +0.036) & 0.945 (-0.042, +0.042) \\
            \bottomrule
        \end{tabular}}
    \end{minipage}
\end{center}

\subsection{Causal arrow experiment details}
\label{sec:experiment_details_causal_arrow_violin}

The evaluated dataset is a simple example similar to the ones evaluated in the synthetic experiments of \emph{$X$-dependent residuals} in  \Cref{sec:experiments}
The structural function is \[
m(x) = 0.55\,x + 0.95\,\sin(1.65\,x) +0.38\,\sin(3.6\,x) + 0.1 \, x^3.
\]
The conditional standard deviation of the normally distributed noise is
\begin{align*}
    \sigma(x) & = 0.45 \exp(c\, \rho\{h(x)-\E[h(X)]\}),\\
    h(x) & = \sin(1.6\, x) + 0.35\cos(2.4\, x),\\
    c &= \sqrt{1/\Var(h(X))}.
\end{align*}

\begin{figure}[ht]
    \centering
    \includegraphics[width=\linewidth,trim=0mm 0mm 0mm 0mm,clip]{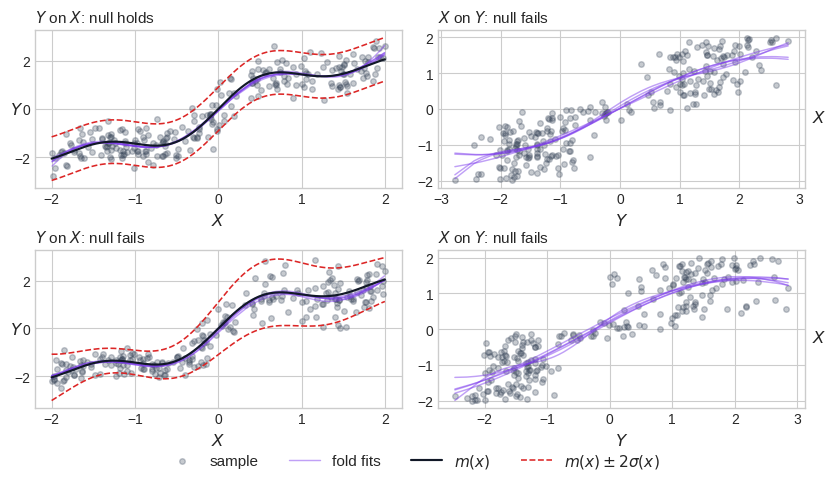}
    \caption{Example data from the causal ANM experiment, where the additive-noise direction is $X\to Y$.}
    \label{fig:causal_arrow_data}
\end{figure}

We observe that the only procedures that are properly calibrated under the null are the debiased bootstrap procedure and
the full-data cross-fitted permutation test, which was in fact very conservative, with its $p$-value distribution
far from the uniform law that would correspond to perfect calibration. Among the tests that correctly fail to reject the null,
only ours and the full-sample cross-fitted permutation baseline attain non-trivial power, with the former performing significantly better.
The split-data tests that were calibrated attain very poor power under the alternative because the testing sample is too small.
\begin{figure*}[t]
    \centering
    \includegraphics[
        width=0.98\textwidth,
        height=0.42\textheight,
        keepaspectratio,
        trim=3mm 2mm 3mm 1mm,
        clip
    ]{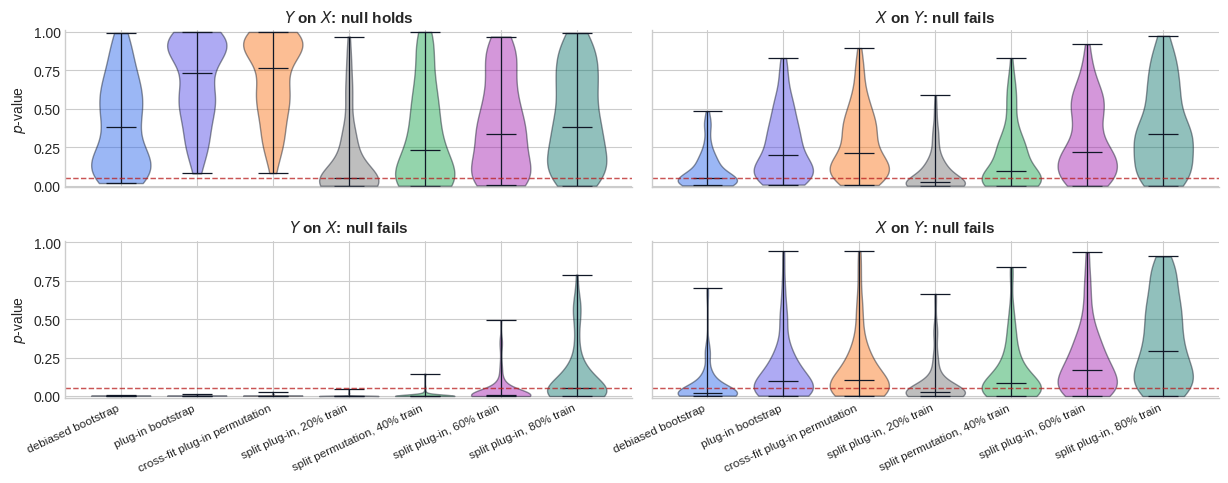}
    \caption{Distribution of causal-arrow scores across repetitions of the two-way goodness-of-fit experiment.}
    \label{fig:experiment_details_causal_arrow_violin}
\end{figure*}

\subsection{Covariate-adjusted residual diagnostics}
\label{sec:experiment_details_covariate_diagnostics}
Each simulated dataset has \(n=120\) observations, with \(X\sim\mathrm{Bernoulli}(1/2)\), \(T\sim\mathrm{Unif}[-\pi,\pi]\), and

\[
Y=m(X,T)+\sigma(X)\varepsilon,\qquad \varepsilon\sim \cN(0,1).
\]
The mean is

\[
m(x,t)=1.5\left\{0.2x+\sum_{k=1}^5\left[
\frac{0.30}{k}\sin(kt)+\frac{0.25}{k}\cos(kt)
+\frac{0.22}{k}x\sin(kt)+\frac{0.18}{k}x\cos(kt)
\right]\right\}.
\]
The null uses \(\sigma(0)=\sigma(1)=0.45\), while the alternative uses \(\sigma(0)=0.35,\sigma(1)=0.55\).

The experiment uses 4-fold cross-fitting, with nuisances obtained via kernel ridge regression. Kernel bandwidths and ridge terms are obtained by cross-validation
within the training folds.
\begin{center}
    \begin{minipage}{\linewidth}
        \centering
        \includegraphics[width=\linewidth,trim=0mm 0mm 0mm 0mm,clip]{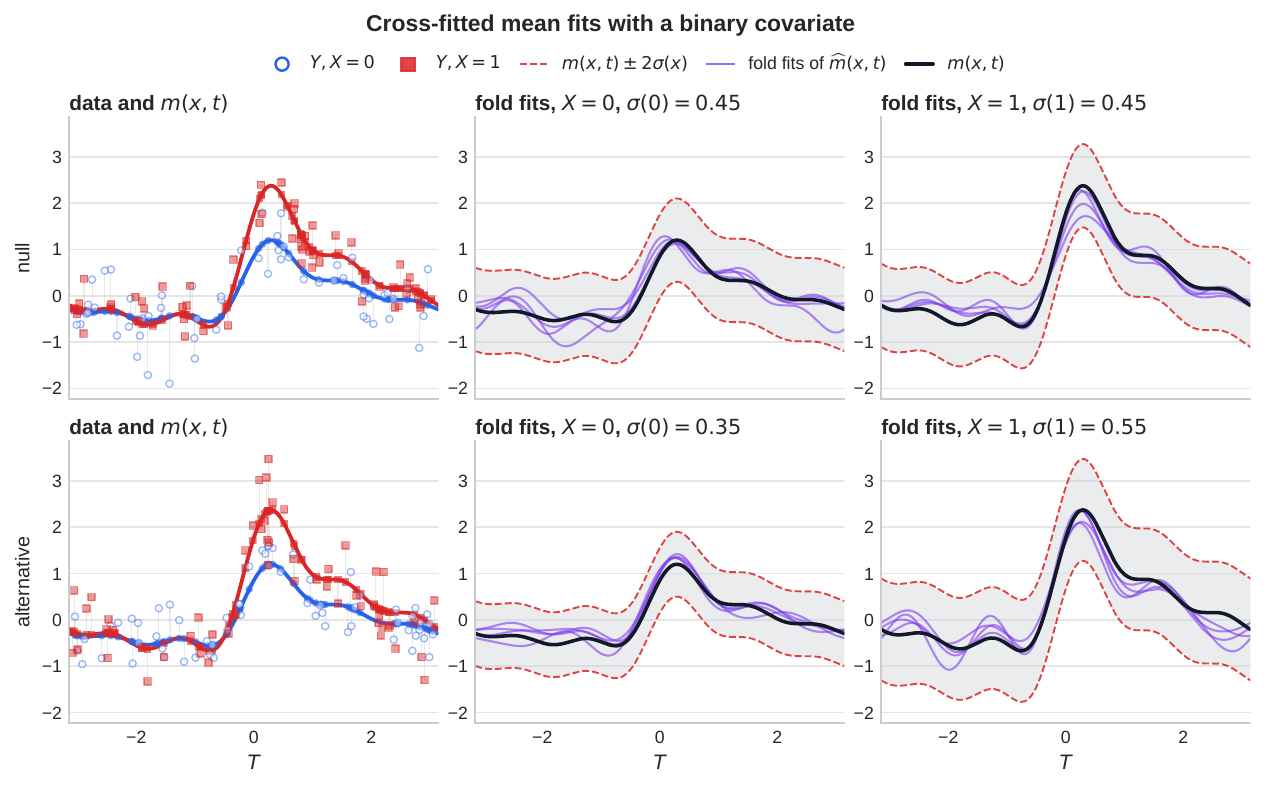}
        \captionof{figure}{Fold-specific structural function fits in the covariate-adjusted residual experiment.}
        \label{fig:experiment_details_covariate_structural_fold_fits}
    \end{minipage}
\end{center}

\begin{center}
    \begin{minipage}{\linewidth}
        \centering
        \includegraphics[width=\linewidth,trim=0mm 0mm 0mm 0mm,clip]{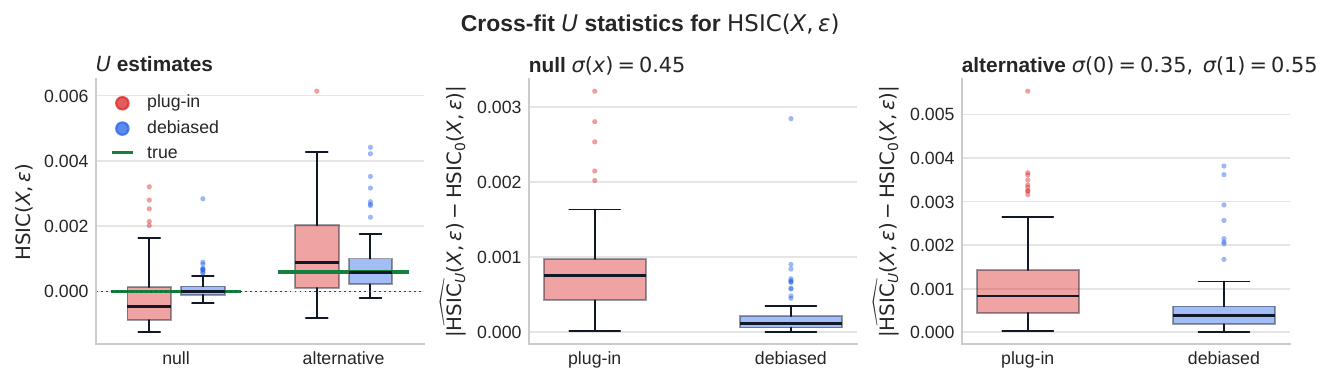}
        \captionof{figure}{Comparison of the feasible $U$-statistic estimator with the oracle target in the covariate-adjusted residual experiment.}
        \label{fig:experiment_details_covariate_u_statistic_oracle_comparison}
    \end{minipage}
\end{center}

\subsection{Ranges of empirical coverage point estimates}
\label{sec:experiment_details_additional_coverage_tables}
We report ranges of observed empirical coverage means across the indicated noise-dependence levels. 
\begin{center}
    \begin{minipage}{\linewidth}
        \centering
        \renewcommand{\arraystretch}{1.08}
        \captionof{table}{Ranges of empirical coverage point estimates over $\rho$ for the remaining synthetic calibration settings at $n=250$, excluding the main setting $s_m=1$ and $s_\epsilon=0.75$. The reverse-triangle range is over $\rho\in\{0,0.1,\ldots,0.5\}$; the delta-method range is over $\rho\in\{0.1,\ldots,0.5\}$.}
        \label{tab:experiment_details_additional_coverage_summary_N250}
        \begin{tabular}{@{}cccc@{}}
            \toprule
            $s_m$ & $s_\epsilon$ & Reverse-triangle & Debiased delta method \\
            \midrule
            1.0 & 0.25 & 0.940--0.990 & 0.838--0.941 \\
            1.0 & 0.5 & 0.930--1.000 & 0.881--0.951 \\
            1.5 & 0.25 & 0.920--1.000 & 0.886--0.987 \\
            1.5 & 0.5 & 0.890--1.000 & 0.900--0.989 \\
            1.5 & 0.75 & 0.890--1.000 & 0.940--0.989 \\
            2.0 & 0.25 & 0.920--1.000 & 0.887--0.972 \\
            2.0 & 0.5 & 0.910--0.990 & 0.893--0.975 \\
            2.0 & 0.75 & 0.910--1.000 & 0.931--0.973 \\
            \bottomrule
        \end{tabular}
    \end{minipage}
\end{center}

\vspace{0.75\baselineskip}

\begin{center}
    \begin{minipage}{\linewidth}
        \centering
        \renewcommand{\arraystretch}{1.08}
        \captionof{table}{Ranges of empirical coverage point estimates over $\rho$ for the remaining synthetic calibration settings at $n=500$. The entries are ranges of observed means, not confidence intervals.}
        \label{tab:experiment_details_additional_coverage_summary_N500}
        \begin{tabular}{@{}cccc@{}}
            \toprule
            $s_m$ & $s_\epsilon$ & Reverse-triangle & Debiased delta method \\
            \midrule
            1.0 & 0.25 & 0.938--1.000 & 0.966--0.989 \\
            1.0 & 0.5 & 0.955--1.000 & 0.969--0.991 \\
            1.5 & 0.25 & 0.933--1.000 & 0.959--0.982 \\
            1.5 & 0.5 & 0.964--1.000 & 0.943--0.985 \\
            1.5 & 0.75 & 0.964--1.000 & 0.953--0.974 \\
            2.0 & 0.25 & 0.929--1.000 & 0.960--0.978 \\
            2.0 & 0.5 & 0.982--1.000 & 0.944--0.989 \\
            2.0 & 0.75 & 0.982--1.000 & 0.943--0.986 \\
            \bottomrule
        \end{tabular}
    \end{minipage}
\end{center}

\vspace{0.75\baselineskip}

\begin{center}
    \begin{minipage}{\linewidth}
        \centering
        \renewcommand{\arraystretch}{1.08}
        \captionof{table}{Ranges of empirical coverage point estimates over $\rho$ for the remaining synthetic calibration settings at $n=1000$. The entries are ranges of observed means, not confidence intervals.}
        \label{tab:experiment_details_additional_coverage_summary_N1000}
        \begin{tabular}{@{}cccc@{}}
            \toprule
            $s_m$ & $s_\epsilon$ & Reverse-triangle & Debiased delta method \\
            \midrule
            1.0 & 0.25 & 0.918--0.991 & 0.938--0.970 \\
            1.0 & 0.5 & 0.918--1.000 & 0.920--0.963 \\
            1.5 & 0.25 & 0.914--0.991 & 0.900--0.955 \\
            1.5 & 0.5 & 0.896--1.000 & 0.919--0.943 \\
            1.5 & 0.75 & 0.907--1.000 & 0.915--0.951 \\
            2.0 & 0.25 & 0.895--1.000 & 0.917--0.960 \\
            2.0 & 0.5 & 0.879--1.000 & 0.933--0.945 \\
            2.0 & 0.75 & 0.882--1.000 & 0.925--0.943 \\
            \bottomrule
        \end{tabular}
    \end{minipage}
\end{center}

\clearpage
\subsection{\texorpdfstring{Additional QQ calibration plots at $n=500$}{Additional QQ calibration plots at n=500}}
\label{sec:experiment_details_additional_qq_plots}

\begin{center}
    \begin{minipage}{\linewidth}
        \centering
        \captionsetup{font=small,skip=2pt}
        \includegraphics[
            width=0.92\linewidth,
            keepaspectratio,
            trim=0mm 0mm 0mm 0mm,
            clip
        ]{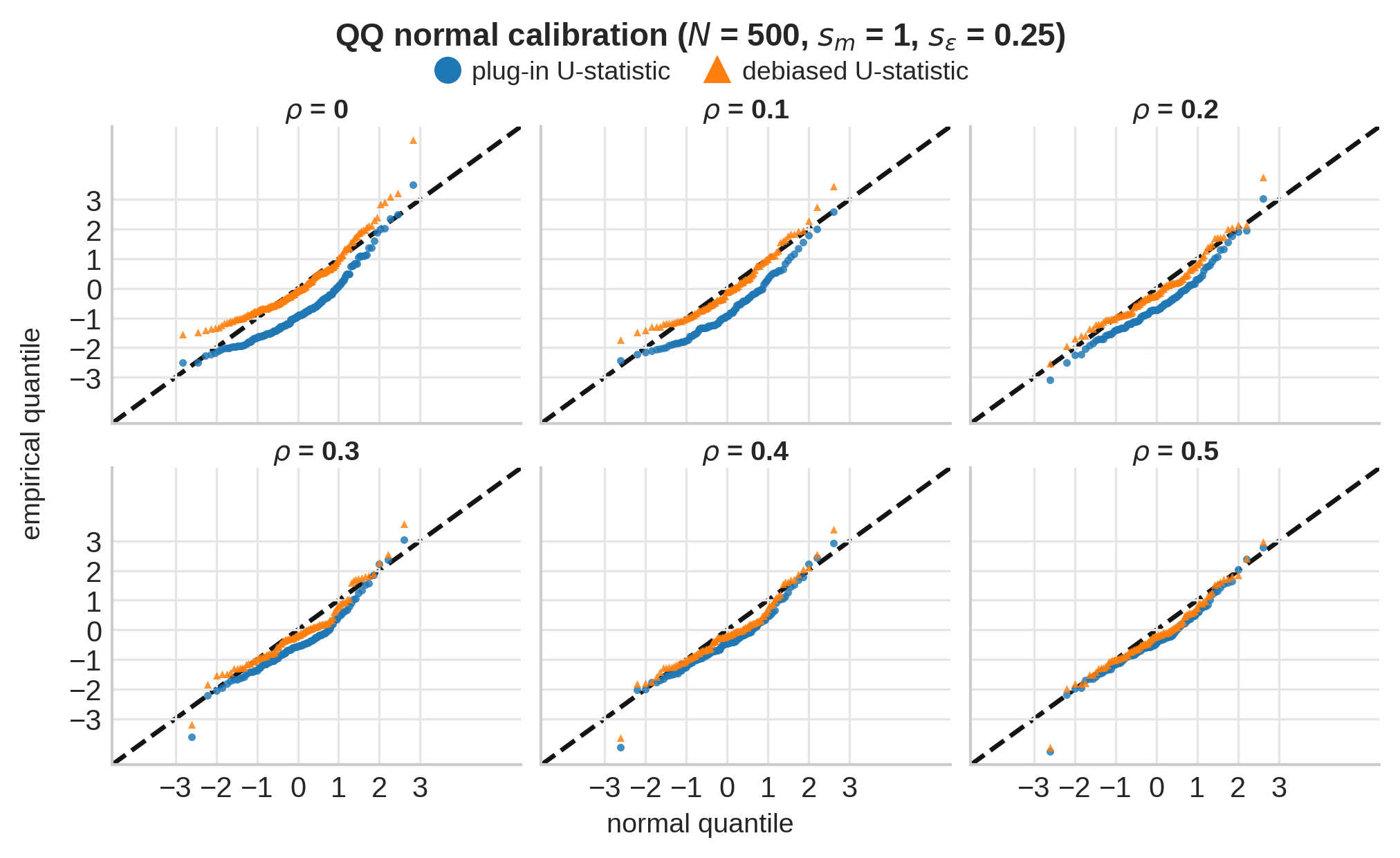}
        \captionof{figure}{QQ calibration at $n=500$, $s_m=1$, and $s_\epsilon=0.25$.}
        \label{fig:experiment_details_qq_N500_sm1_seps025}
        \par\vspace{0.45\baselineskip}
        \includegraphics[
            width=0.92\linewidth,
            keepaspectratio,
            trim=0mm 0mm 0mm 0mm,
            clip
        ]{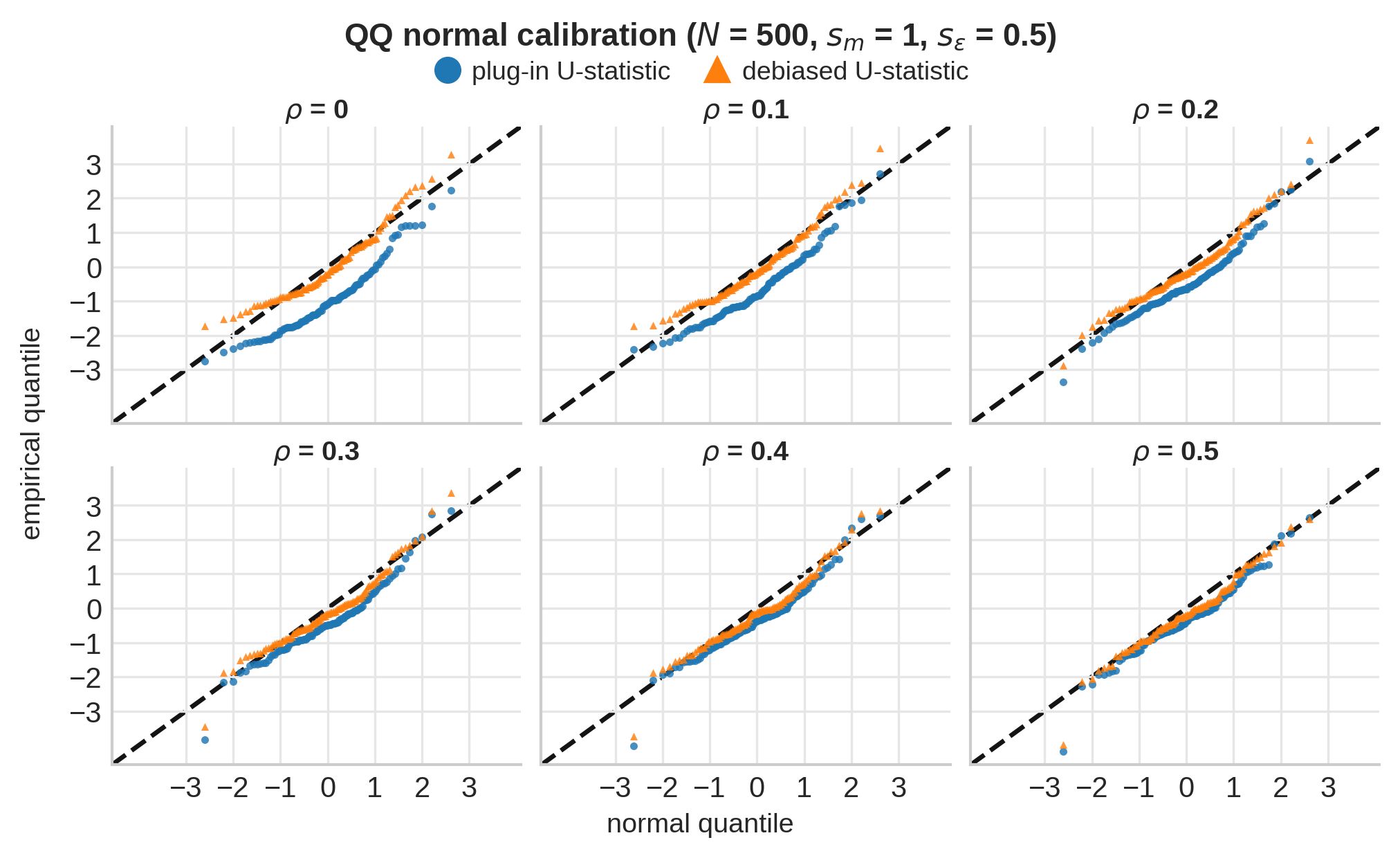}
        \captionof{figure}{QQ calibration at $n=500$, $s_m=1$, and $s_\epsilon=0.5$.}
        \label{fig:experiment_details_qq_N500_sm1_seps05}
    \end{minipage}
\end{center}

\begin{figure}[p]
    \centering
    \includegraphics[width=0.92\linewidth,trim=0mm 0mm 0mm 0mm,clip]{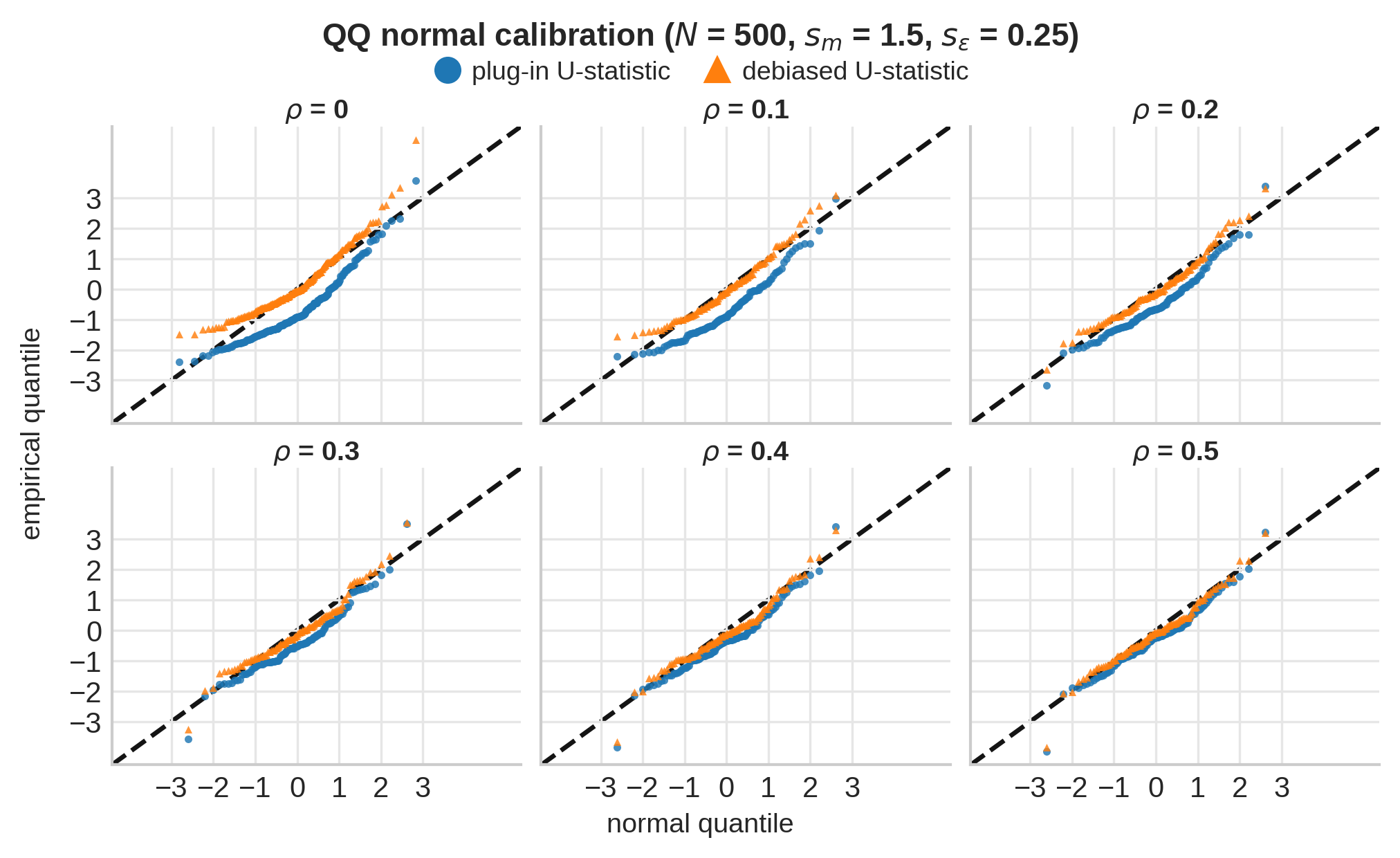}
    \caption{QQ calibration at $n=500$, $s_m=1.5$, and $s_\epsilon=0.25$.}
    \label{fig:experiment_details_qq_N500_sm15_seps025}
\end{figure}

\begin{figure}[p]
    \centering
    \includegraphics[width=0.92\linewidth,trim=0mm 0mm 0mm 0mm,clip]{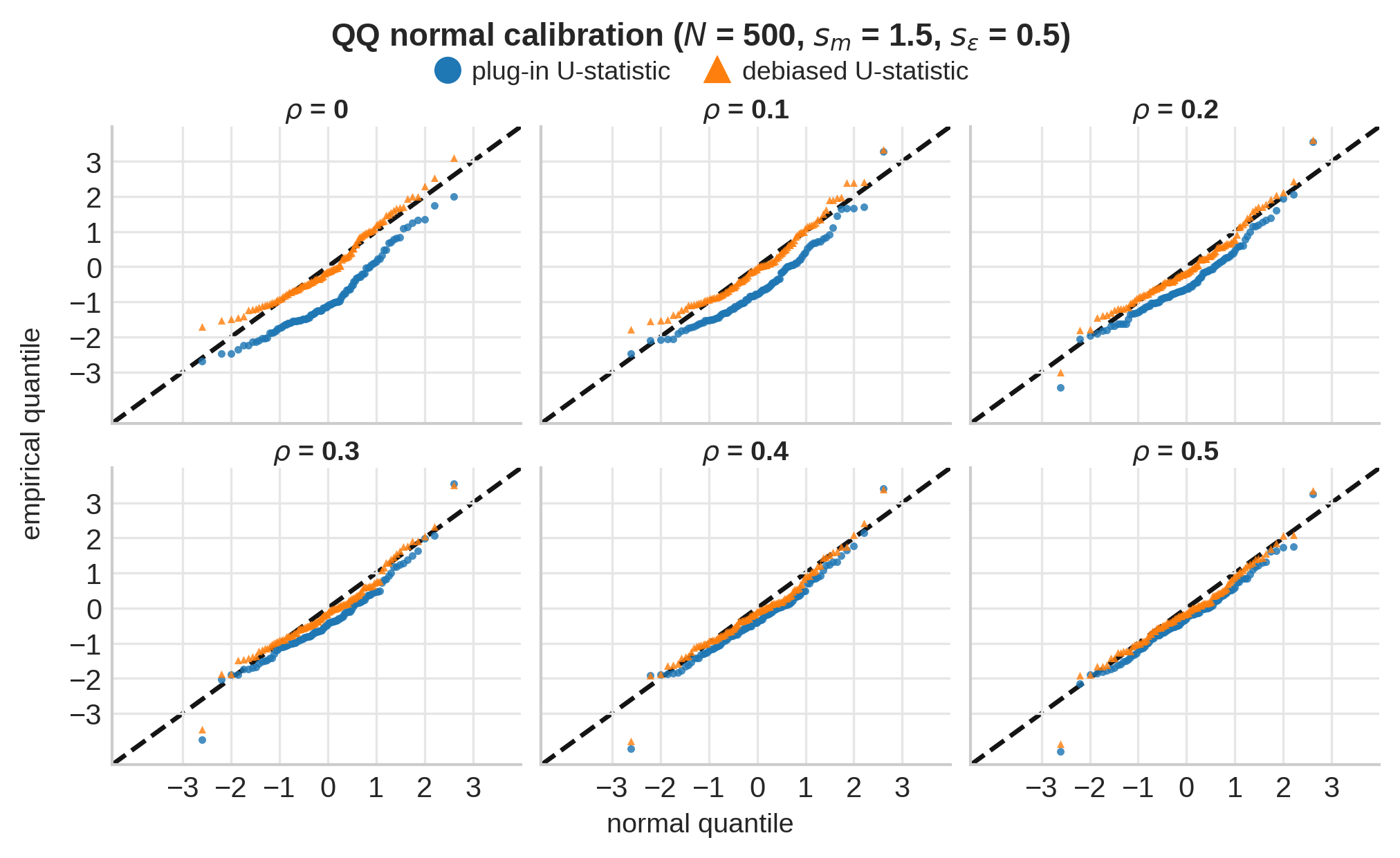}
    \caption{QQ calibration at $n=500$, $s_m=1.5$, and $s_\epsilon=0.5$.}
    \label{fig:experiment_details_qq_N500_sm15_seps05}
\end{figure}

\begin{figure}[p]
    \centering
    \includegraphics[width=0.92\linewidth,trim=0mm 0mm 0mm 0mm,clip]{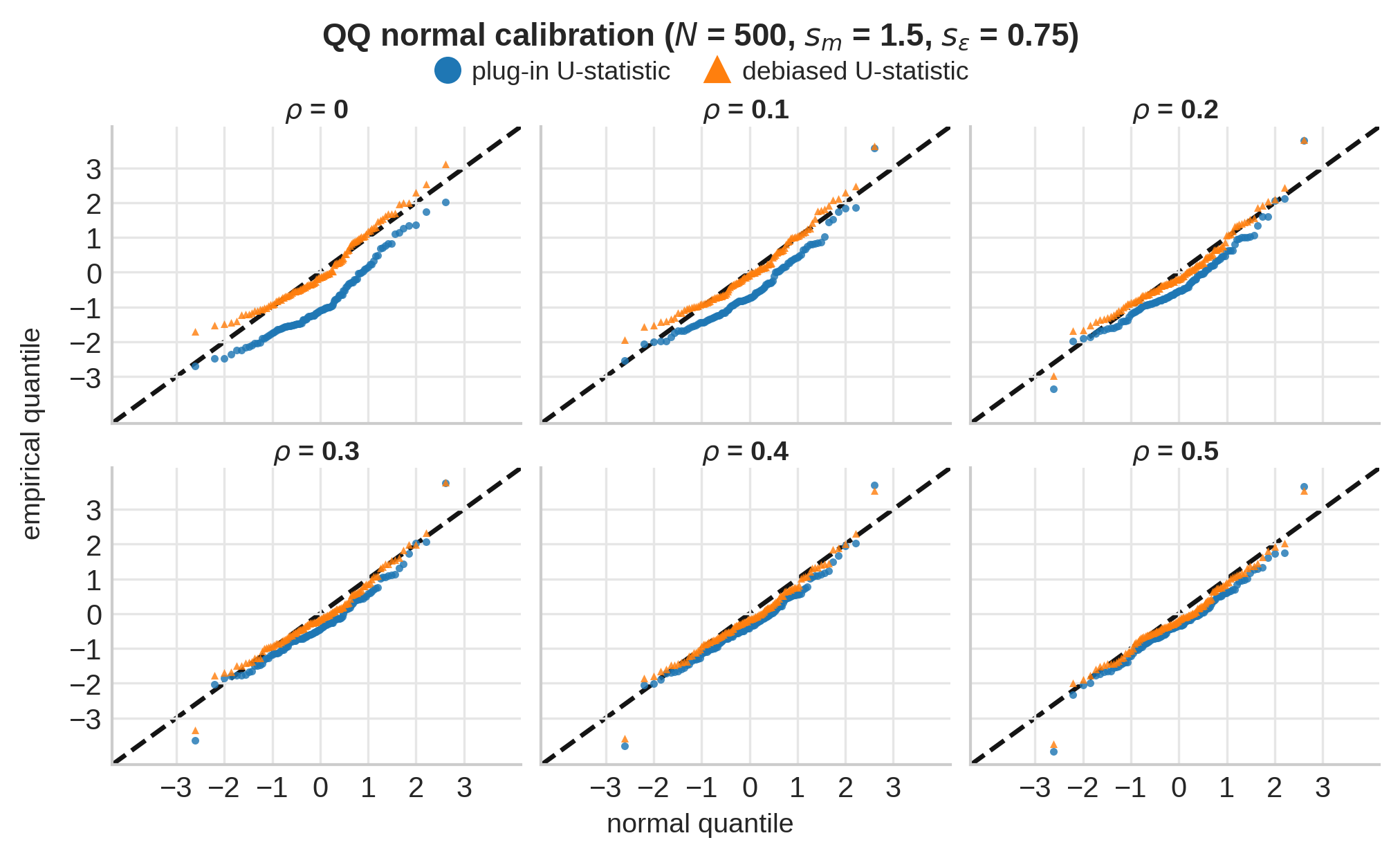}
    \caption{QQ calibration at $n=500$, $s_m=1.5$, and $s_\epsilon=0.75$.}
    \label{fig:experiment_details_qq_N500_sm15_seps075}
\end{figure}

\begin{figure}[p]
    \centering
    \includegraphics[width=0.92\linewidth,trim=0mm 0mm 0mm 0mm,clip]{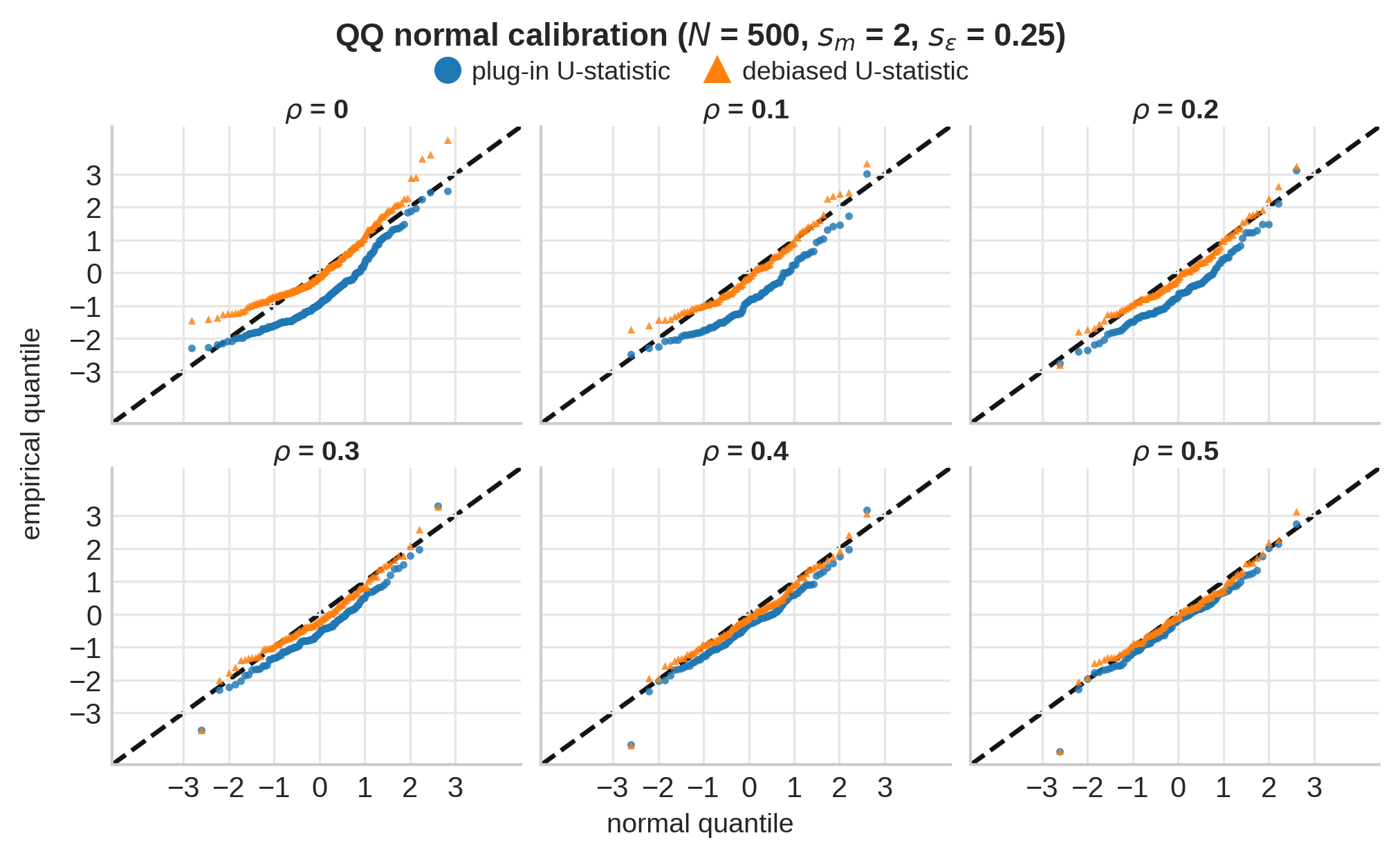}
    \caption{QQ calibration at $n=500$, $s_m=2$, and $s_\epsilon=0.25$.}
    \label{fig:experiment_details_qq_N500_sm2_seps025}
\end{figure}

\begin{figure}[p]
    \centering
    \includegraphics[width=0.92\linewidth,trim=0mm 0mm 0mm 0mm,clip]{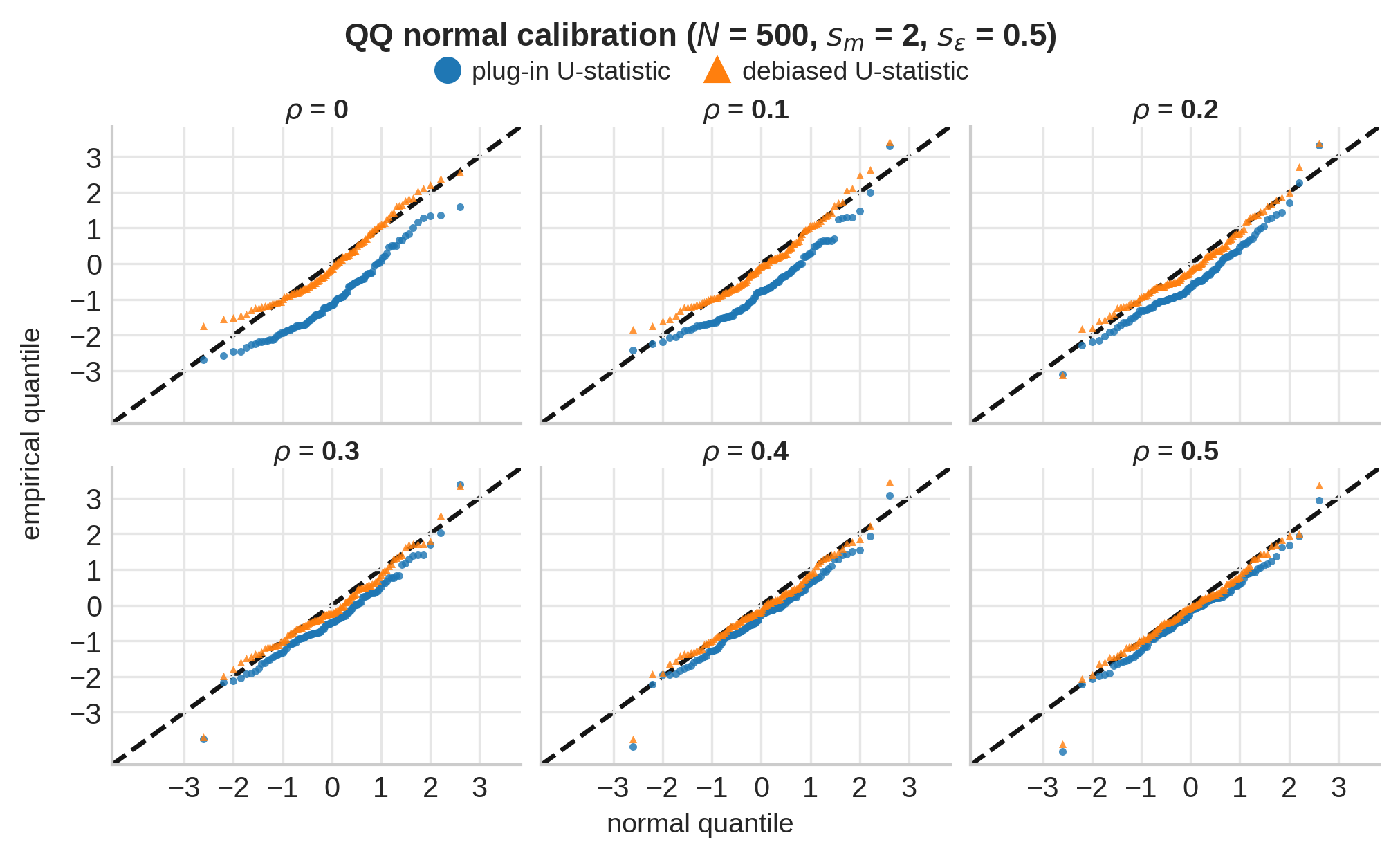}
    \caption{QQ calibration at $n=500$, $s_m=2$, and $s_\epsilon=0.5$.}
    \label{fig:experiment_details_qq_N500_sm2_seps05}
\end{figure}

\begin{figure}[p]
    \centering
    \includegraphics[width=0.92\linewidth,trim=0mm 0mm 0mm 0mm,clip]{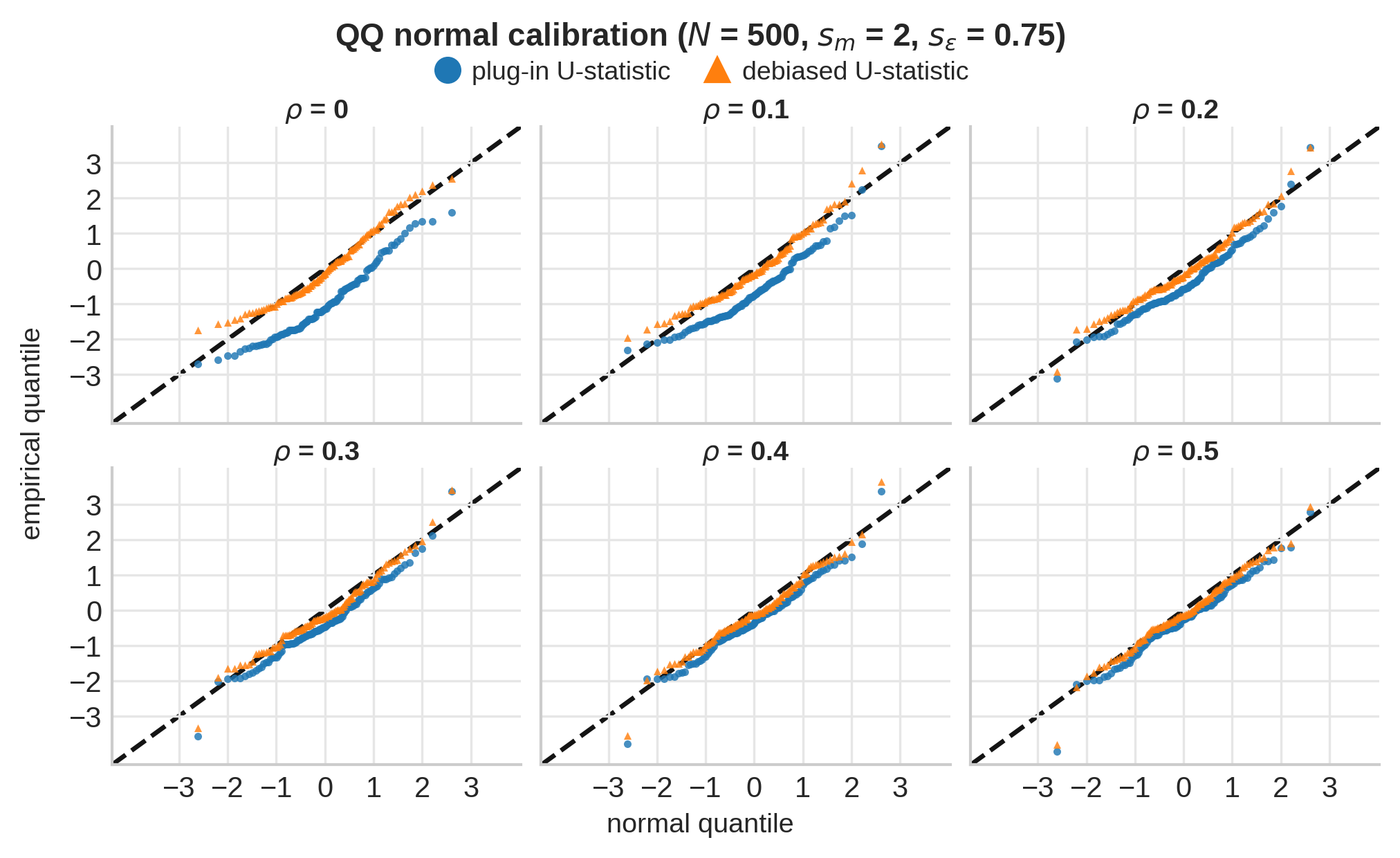}
    \caption{QQ calibration at $n=500$, $s_m=2$, and $s_\epsilon=0.75$.}
    \label{fig:experiment_details_qq_N500_sm2_seps075}
\end{figure}

\clearpage

\section{Validity of asymptotically normal confidence sets}
\label{sec:finite_sample_delta_validity}

\begin{center}
\small
\setlength{\tabcolsep}{5pt}
\renewcommand{\arraystretch}{1.25}
\begin{tabularx}{\linewidth}{@{}>{\raggedright\arraybackslash}p{0.27\linewidth}>{\raggedright\arraybackslash}X@{}}
\toprule
Notation & Meaning \\
\midrule
$\check\Psi_n$ & Cross-fitted one-step estimator of $\Psi_0$ used in the delta-method diagnostic. \\
\midrule
$\widehat\zeta_n$ & Bootstrap quantile for $\|\mathbb H\|_{\mathcal H_{kl}}^2$. \\
\midrule
$\hat\sigma^2$ & Plug-in estimate of the delta-method variance component. \\
\midrule
$z_{1-\beta/2}$ & Standard normal quantile used for the diagnostic comparison. \\
\midrule
$\widehat{\mathfrak D}_n$ & Ratio of estimated quadratic remainder scale to estimated linear-term scale. \\
\bottomrule
\end{tabularx}
\end{center}

\textbf{Delta method.} We have the algebraic identity
\begin{align*}
    &\sqrt{n}\left(\|\check\Psi_n\|^2_{\cH_{kl}} - \|\Psi_0\|^2_{\cH_{kl}}\right)\\
    &=\sqrt{n}\left\langle \check\Psi_n - \Psi_0, \check\Psi_n + \Psi_0\right\rangle_{\cH_{kl}}\\
    &= \sqrt{n}\|\check\Psi_n - \Psi_0\|^2_{\cH_{kl}} + 2\sqrt{n}\langle \check\Psi_n - \Psi_0, \Psi_0\rangle_{\cH_{kl}}.
\end{align*}
We refer to $2\sqrt{n}\langle \check\Psi_n - \Psi_0, \Psi_0\rangle_{\cH_{kl}}$ as the \textit{linear term} and $\sqrt{n}\|\check\Psi_n - \Psi_0\|^2_{\cH_{kl}}$ as the \textit{quadratic remainder}. For a fixed $\Psi_0\neq 0$, we derive a distributional limit for $\sqrt{n}\left(\|\check\Psi_n\|^2_{\cH_{kl}} - \|\Psi_0\|^2_{\cH_{kl}}\right)$. Suppose the conclusion of \Cref{thm:appendix_two_fold_cf_asymp_linearity} holds, namely
\begin{align*}
    \sqrt{n}(\check\Psi_n - \Psi_0) \rightsquigarrow \mathbb{H}.
\end{align*}
By the Continuous Mapping Theorem, $n\|\check\Psi_n - \Psi_0\|^2_{\cH_{kl}}\rightsquigarrow \|\sH\|_{\cH_{kl}}^2$.
Thus $\sqrt{n}\|\check\Psi_n - \Psi_0\|^2_{\cH_{kl}}\rightsquigarrow 0.$
By another application of the Continuous Mapping Theorem, $\sqrt{n}\lra{\check\Psi_n-\Psi_0,\Psi_0}\tod \lra{\sH,\Psi_0}_{\cH_{kl}}$.
Thus
\begin{align*}
    \sqrt{n}\left(\|\check\Psi_n\|^2_{\cH_{kl}} - \|\Psi_0\|^2_{\cH_{kl}}\right) \rightsquigarrow 2\langle \mathbb{H}, \Psi_0\rangle_{\cH_{kl}}.
\end{align*}
Since $2\langle \mathbb{H}, \Psi_0\rangle_{\cH_{kl}}\sim \mathcal{N}(0, 4\mathbb{E}[\langle \phi(W,X,Y), \Psi_0\rangle_{\cH_{kl}}^2])$, this yields a closed form asymptotic variance that can be used to construct a calibrated confidence interval for $\sqrt{n}\left(\|\check\Psi_n\|^2_{\cH_{kl}} - \|\Psi_0\|^2_{\cH_{kl}}\right)$.

However, in finite samples, the quality of the approximation heavily depends on the size of $ \sqrt{n}\|\check\Psi_n - \Psi_0\|^2_{\cH_{kl}}$ relative to $2\sqrt{n}\langle \check\Psi_n - \Psi_0, \Psi_0\rangle_{\cH_{kl}}$. We therefore propose a data-driven diagnostic via the bootstrapped empirical measures $\mathbb{P}^{(b)}_n$. Recall in \Cref{sec: estimator}, we defined $\widehat{\zeta}_n$ as the $(1-\alpha)$-quantile of $\|\mathbb{H}_n^{(b)}\|^2_{\cH_{kl}}$, conditionally on the original sample $(O_1, \dots, O_n)$. The confidence interval $\mathcal{C}_n(\widehat{\zeta}_n)$ contains $0$ if and only if $n\|\check\Psi_n\|^2_{\cH_{kl}}\leq \widehat{\zeta}_n$.

If the linear term has reached its asymptotic regime, then
\[
    2\sqrt n \left\langle \check\Psi_n-\Psi_0,\Psi_0\right\rangle_{\mathcal H_{kl}}
    \overset{d}{\approx}
    N(0,4\sigma^2),
    \qquad
    \sigma^2
    =
    \mathrm{Var}\!\left(
        \left\langle \Psi_0,\mathbb H\right\rangle_{\mathcal H_{kl}}
    \right).
\]
Thus the $(1-\beta)$ quantile of the absolute value of the linear term is approximated by
\[
    2\hat\sigma z_{1-\beta/2},
\]
where $\hat\sigma^2$ is the plug-in variance estimator from
\Cref{prop: sigma_consistency}. On the other hand, since $\hat\zeta_n$ estimates the
$(1-\alpha)$ quantile of $\|\mathbb H\|^2_{\mathcal H_{kl}}$, the quantity
$n^{-1/2}\hat\zeta_n$ estimates the $(1-\alpha)$ quantile of the quadratic
remainder
\[
    \sqrt n\|\check\Psi_n-\Psi_0\|^2_{\mathcal H_{kl}}.
\]
This gives the diagnostic
\[
    \widehat{\mathfrak D}_n
    =
    \frac{n^{-1/2}\hat\zeta_n}
    {2\hat\sigma z_{1-\beta/2}}.
\]

A simple warning sign is that the triangle-inequality-based confidence set (\Cref{eq: curlyI_zetan}) contains zero:
\[
    0\in\mathcal C_n(\hat\zeta_n)
    \quad\Longleftrightarrow\quad
    n\|\check\Psi_n\|^2_{\mathcal H_{kl}}\le \hat\zeta_n .
\]
In this case the estimated signal is no larger than the bootstrap uncertainty
for $\check\Psi_n$ itself. This is precisely the near-null regime where the
derivative of $h\mapsto \|h\|^2_{\mathcal H_{kl}}$ at $\Psi_0$ is too small for the
linear approximation to dominate.

\section{Statistical learning theory}
\label{app:statistical_learning_theory}

\begin{center}
\small
\setlength{\tabcolsep}{5pt}
\renewcommand{\arraystretch}{1.25}
\begin{tabularx}{\linewidth}{@{}>{\raggedright\arraybackslash}p{0.28\linewidth}>{\raggedright\arraybackslash}X@{}}
\toprule
Notation & Meaning \\
\midrule
$\mathcal S_j$, $\mathcal S$ & $(\partial_j\varphi_l)\{Y-m_0(W)\}\in\mathcal H_l$; $\mathcal{S}_j$ with coordinate index suppressed in the derivations. \\
\midrule
$v_{0,j}$, $v$ & $\mathbb{E}_0[\mathcal{S}_j\mid W=w]$ vector-valued true regression function, where subscript $0$ denotes dependence on true population distribution; $v_{0,j}$ with coordinate index suppressed in the derivations. \\
\midrule
$q$, $\mathcal H_q$, $\kappa_q$ & reproducing kernel on $\mathcal W$, its RKHS, and its uniform kernel bound. \\
\midrule
$I_w$, $S_w$, $L_W$, $C_W$ & $L^2$-embedding of $\cH_{q}$, adjoint of $I_w$, integral operator $I_w\circ S_w$, covariance operator $S_w \circ I_w$. \\
\midrule
$I$, $\mu_i$, $e_i$ & Index set, eigenvalues, and eigenfunctions of spectral decomposition of $L_W$ used to define interpolation spaces. (see \Cref{eq:SVD}) \\
\midrule
$[\mathcal H_q]^\alpha$, $[\mathcal G]^\alpha$ & Scalar and vector-valued interpolation spaces with index $\alpha$ \\
\midrule
$\mathcal G$, $S_2(H,H')$ & Vector-valued RKHS (see \Cref{def:inter_ope_norm}) and Hilbert--Schmidt operators from $H$ to $H'$. \\
\midrule
$\beta$, $C_0$ & Exponent for $v_0$ in \Cref{asst:src} (measures smoothness), and operator representer for $v_0$. \\
\midrule
$p$ & Exponent for $\mu_i$ in \Cref{asst:evd} (measures capacity of $\cH_q$) \\
\midrule
$p_{\mathcal S\mid W}$, $\sigma$, $R$ & Conditional law of the label $\mathcal{S}$ and constants in \Cref{asst:mom}. \\
\midrule
$\widehat v_\lambda$, $\widetilde v_\lambda$ & empirical vvKRR solution with estimated labels (i.e. plug-in regression function $\widehat{m}$) in \Cref{eq:empirical_v_lambda}, empirical vvKRR with oracle labels $\mathcal{S}_j$ in \Cref{eq:empirical_v_oracle_lambda}.\\
\midrule
$\widehat C_W$, $C_{W,\lambda}$, $\widehat C_{W,\lambda}$ & Empirical covariance operator $\frac{2}{n}\sum_{i=1}^{n/2}\varphi_q(w_i)\otimes \varphi_q(w_i)$, population regularized covariance operator $\mathbb{E}_0[\varphi_q(W)\otimes \varphi_q(W)] + \lambda I$, empirical regularized covariance operators $\widehat{C}_W + \lambda I$. \\
\bottomrule
\end{tabularx}
\end{center}

For the fixed coordinate $j$, define
\[
    \mathcal S_j
    :=
    (\partial_j\varphi_l)\{Y-m_0(W)\}
    \in \mathcal H_l,
\]
and let
\[
    v_{0,j}(w)
    :=
    \mathbb E_0[\mathcal S_j\mid W=w].
\]
We use $v_{0,j}$ both for a fixed measurable version of this conditional expectation
and for its $L^2(P_{W,0};\mathcal H_l)$-equivalence class.  The analysis is identical for each $j$, so we fix $j$ and suppress it from the notation. In other words, we write
\[
    \mathcal S:=\mathcal S_j,
    \qquad
    v_0:=v_{0,j},
\]
where
\[
    v_0(w)
    :=
    \mathbb E_0[\mathcal S\mid W=w].
\]
We use $v_0$ both for a fixed measurable version of this conditional expectation and for its $L^2(P_{W,0};\mathcal H_l)$-equivalence class.

We make the following assumptions on the kernel $q$ and its RKHS $\cH_q\subseteq \mathbb{R}^{\mathcal{W}}$. 
\begin{enumerate}[itemsep=0em]
    \item \label{assump:separable}
    $\cH_q$ is separable, this is satisfied
    if $q$ is continuous, given that $\mathcal{W}$ is separable;\footnote{This follows from \citet[Lemma 4.33]{steinwart2008support}. Note that the Lemma requires separability of $\mathcal{W}$, which is satisfied since we assume that $\mathcal{W}$ is a second countable locally compact Hausdorff space.}
    \item \label{assump:measurable}
    $q(\cdot,w)$ is measurable for $P_{W,0}$-almost all $w \in \mathcal{W}$;
    \item \label{assump:bounded}
    $q(w,w) \leq \kappa_q^2$ for $P_{W,0}$-almost all $w\in \mathcal{W}$.
\end{enumerate}
The above assumptions are not restrictive in practice, as well-known kernels such as the Gaussian, Laplace and Mat{\'e}rn kernels satisfy them on $\mathcal{W} \subseteq \mathbb{R}^{d_w}$ \citep{sriperumbudur2011universality}.

We now introduce some facts about the interplay between $\cH_q$ and $L^2(P_{W,0}),$ which has been extensively studied by \cite{smale2004shannon,smale2005shannon}, \cite{de2006discretization} and \cite{steinwart2012mercer}. We first define the embedding $I_{w}: \cH_q \rightarrow L^2(P_{W,0})$, mapping a function $f \in \cH_q$ to its $P_{W,0}$-equivalence class $[f]$. Since the kernel $q$ is bounded, we know that this embedding is a well-defined Hilbert-Schmidt operator. Indeed, by \citep[Lemma 2.2 \& 2.3]{steinwart2012mercer}, its Hilbert-Schmidt norm can be computed as
\begin{align*}
    \left\|I_{w}\right\|_{S_{2}\left(\cH_q, L^2(P_{W,0})\right)}^2 = \int_{\mathcal{W}} q(w,w) \mathrm{d}P_{W,0}(w) \leq \kappa_q^2. 
\end{align*}
The adjoint operator $S_{w}:=I_{w}^{*}: L^2(P_{W,0}) \rightarrow \cH_q$ is an integral operator with respect to the kernel $q$, i.e. for $f \in L^2(P_{W,0})$ and $w \in \mathcal{W}$ we have \citep[Theorem 4.27]{steinwart2008support}
$$
\left(S_{w} f\right)(w)=\int_{\mathcal{W}} q\left(w, w'\right) f\left(w'\right) \mathrm{d} P_{W,0}\left(w'\right).
$$
Next, we define the self-adjoint, positive semi-definite and trace class integral operators
$$
L_{W}:=I_{w} S_{w}: L^2(P_{W,0}) \rightarrow L^2(P_{W,0}) \quad \text { and } \quad C_{W}:=S_{w} I_{w}: \cH_q \rightarrow \cH_q.
$$

We now introduce the background required in order to  characterize the smoothness of the target function $v_0$. We review the results of 
\cite{steinwart2012mercer} and \cite{fischer2020sobolev} in constructing scalar-valued interpolation spaces, and \cite{lietal2022optimal} in defining vector-valued interpolation spaces.

{\bf Real-valued Interpolation Space:} By the spectral theorem for self-adjoint compact operators, there exists an at most countable index set $I$, a non-increasing sequence $(\mu_i)_{i\in I} > 0$, and a family $(e_i)_{i \in I} \in \cH_q$, such that $\left([e_i]\right)_{i \in I}$ 
\footnote{We recall that the bracket $[\cdot]$ denotes the embedding that maps $f$ to its equivalence class $I_{w}(f) \in L^2(P_{W,0})$.} is an orthonormal basis (ONB) of $\overline{\text{ran}~I_{w}} \subseteq L^2(P_{W,0})$ and $(\mu_i^{1/2}e_i)_{i\in I}$ is an ONB of $\left(\operatorname{ker} I_{w}\right)^{\perp} \subseteq \cH_q$, and we have 
\begin{equation} \label{eq:SVD}
    L_{W} = \sum_{i\in I} \mu_i \langle\cdot, [e_i] \rangle_{L^2(P_{W,0})}[e_i], \qquad C_{W} = \sum_{i \in I} \mu_i \langle\cdot, \mu_i^{\frac{1}{2}}e_i \rangle_{\cH_q} \left(\mu_i^{\frac{1}{2}}e_i\right).
\end{equation}

For $\alpha \geq 0$, the $\alpha$-interpolation space  \cite{steinwart2012mercer} is defined by
\[[\cH_q]^{\alpha}:=\left\{\sum_{i \in I} a_{i} \mu_{i}^{\alpha / 2}\left[e_{i}\right]:\left(a_{i}\right)_{i \in I} \in \ell_{2}(I)\right\} \subseteq L^2(P_{W,0}),\]
equipped with the inner product
 \[\left\langle \sum_{i \in I}a_i(\mu_i^{\alpha/2}[e_i]), \sum_{i \in I}b_i(\mu_i^{\alpha/2}[e_i]) \right\rangle_{[\cH_q]^{\alpha}} = \sum_{i \in I} a_i b_i,\] 
for $\left(a_{i}\right)_{i\in I}, \left(b_{i}\right)_{i\in I} \in \ell_{2}(I)$. The $\alpha$-interpolation space defines a Hilbert space. Moreover, $\left(\mu_{i}^{\alpha / 2}\left[e_{i}\right]\right)_{i \in I}$ forms an ONB of $[\cH_q]^{\alpha}$ and consequently $[\cH_q]^{\alpha}$ is a separable Hilbert space. In the following, we use the abbreviation $\|\cdot\|_{\alpha}:=\|\cdot\|_{[\cH_q]^{\alpha}}$.

{\bf Vector-valued Reproducing Kernel Hilbert Space.} Let $K: \mathcal{W} \times \mathcal{W} \rightarrow \mathcal{L}(\cH_l)$ be an operator valued positive-semidefinite (psd) kernel. Fix $K$, $w \in \mathcal{W}$, and $h \in \cH_l$, then $\left(K_w h\right)(\cdot):=K(\cdot, w) h$
defines a function from $\mathcal{W}$ to $\cH_l$. The completion of 
\begin{align*}
    \mathcal{G}_{\text {pre }}:=\operatorname{span}\left\{K_w h \mid w \in \mathcal{W}, h \in \cH_l\right\},
\end{align*}
with inner product on $\mathcal{G}_{\text {pre }}$ defined on the elementary elements as $\left\langle K_w h, K_{w'} h^{\prime}\right\rangle_{\mathcal{G}}:=\left\langle h, K\left(w, w'\right) h^{\prime}\right\rangle_{\cH_l}$, defines a vRKHS denoted as $\mathcal{G}$. For a more complete overview of the vector-valued reproducing kernel Hilbert space, we refer the reader to \cite{carmeli2006vector}, \cite{carmeli2010vector} and \cite[][Section 2]{li2024towards}. In the following, we will denote $\mathcal{G}$ as the vRKHS induced by the kernel $K: \mathcal{W} \times \mathcal{W} \rightarrow \mathcal{L}(\cH_l)$ with 
\begin{equation} \label{eq:v_kernel}
K(w, w') := q(w, w')\operatorname{Id}_{\cH_l}, \quad w,w' \in \mathcal{W}.
\end{equation}

\begin{theorem}[vRKHS isomorphism]\label{theo:operep}
For every function $F\in \mathcal{G}$ there exists a unique operator $C \in S_2(\cH_q, \cH_l)$ such that $F(\cdot) = C\varphi_q(\cdot) \in \cH_l$ with $\|C\|_{S_2(\cH_q, \cH_l)} = \|F\|_{\mathcal{G}}$ and vice versa. Hence $\cG \simeq S_2(\cH_q, \cH_l)$ and we denote the isometric isomorphism between $S_2(\cH_q, \cH_l)$ and $\mathcal{G}$ as $\mathfrak{I}$. It follows that $\cG$ can be written as $\mathcal{G}= \left\{F: \mathcal{W} \rightarrow \cH_l \mid F = C\varphi_q(\cdot),\, C \in S_2(\cH_q, \cH_l)\right\}$.
\end{theorem}

{\bf Vector-valued Interpolation Space:} Introduced in \cite{lietal2022optimal}, vector-valued interpolation spaces generalize the notion of scalar-valued interpolation spaces to vRKHS with a separable kernel of the form \eqref{eq:v_kernel}. We refer the reader to \citet{steinwart2012mercer, fischer2020sobolev} for an introduction on scalar-valued interpolation spaces. For two Hilbert space $H,H'$, we use $S_2(H,H')$ to denote the space of Hilbert-Schmidt operators from $H$ to $H'$. 

\begin{definition}[Vector-valued interpolation space]\label{def:inter_ope_norm} Let $q$ be a real-valued kernel with associated RKHS $\cH_q$ and let $[\cH_q]^{\alpha}$ be the real-valued interpolation space associated to $\cH_q$ with some $\alpha \geq 0$. The vector-valued interpolation space $[\mathcal{G}]^{\alpha}$ is defined as
\[[\mathcal{G}]^{\alpha} 
:= {\mathfrak{I}}\left(S_2([\cH_q]^{\alpha},\cH_l) \right) = \{F \mid F = {\mathfrak{I}}(C), ~C \in S_2([\cH_q]^{\alpha},\cH_l)\}.\]
The space $[\mathcal{G}]^{\alpha}$ is a Hilbert space equipped with the inner product $$\left\langle F, G \right\rangle_{\alpha} := \left\langle C, L\right\rangle_{S_2([\cH_q]^{\alpha}, \cH_l)} \qquad (F,G \in [\mathcal{G}]^{\alpha}),$$ where $C = {\mathfrak{I}}^{-1}(F),$ $L = {\mathfrak{I}}^{-1}(G)$. 
\end{definition}

\begin{rem}[Well-specified versus misspecified setting]\label{rem:well_specified} We say that we are in the well-specified setting if $v_0 \in [\cG]^{1}$. In this case, there exists $\bar{F} \in \cG$ such that $v_0 = \bar{F}$ $P_{W,0}-$almost surely and $\|v_0\|_{1} = \|\bar{F}\|_{\cG}$, i.e. $v_0$ admits a representer in $\cG$. 
\end{rem}

\Cref{def:inter_ope_norm} and \Cref{rem:well_specified} motivate the use of following assumption on the smoothness of the target function:
there exists $\beta > 0$ and a constant $B \geq 0$ such that $v_0 \in [\mathcal{G}]^{\beta}$
\begin{equation}\label{asst:src}
\|v_0\|_{\beta} \leq B. \tag{SRC}
\end{equation}
We let $C_0 := {\mathfrak{I}}^{-1}(v_0) \in S_2([\cH_q]^{\beta}, \cH_l)$. \eqref{asst:src} directly generalizes the notion of a so-called Hölder-type source condition in the learning literature \citep{caponnetto2007optimal,fischer2020sobolev,lin2018optimal,lin2020optimal} and allows to characterize the misspecified learning scenario.

In addition to \eqref{asst:src}, we require standard assumptions to obtain the precise learning rate for kernel learning algorithms. We list them below. Recall $\mu_i$ is defined in Eq.~\eqref{eq:SVD}. For constants $D_2 >0$ and $p \in (0,1]$ and for all $i \in I$,
\begin{equation}\label{asst:evd}
    \mu_i \leq D_2i^{-1/p}.\tag{EVD}
\end{equation}
\eqref{asst:evd} is a standard assumption on the \textit{eigenvalue decay} of the integral operator: they describe the interplay between the marginal distribution $P_{W,0}$ and the RKHS $\cH_q$ (see more details in \citealp{caponnetto2007optimal,fischer2020sobolev}). 

Let $p_{\mathcal S\mid W}(w,dz)$ denote a regular conditional distribution of
$\mathcal S$ given $W=w$. We assume that there exist constants
$\sigma,R>0$ such that, for $P_{W,0}$-almost every $w\in\mathcal W$
and every integer $r\geq 2$,
\begin{equation}\label{asst:mom}
\int_{\mathcal H_l}
\|z-v_0(w)\|_{\mathcal H_l}^r
p_{\mathcal S\mid W}(w,dz)
\leq
\frac12 r!\sigma^2R^{r-2}.
\tag{MOM}
\end{equation}
The \eqref{asst:mom} condition on the Markov kernel
$p_{\mathcal S\mid W}(w,dz)$ is a \textit{Bernstein moment condition} used to control the Hilbert-valued regression noise $\mathcal S-v_0(W)$. If $\|\mathcal S\|_{\mathcal H_l}\leq M$ almost surely, then
\eqref{asst:mom} is satisfied with constants depending only on $M$. In particular,
this holds if the derivative feature map $\partial_j\varphi_l$ is uniformly bounded.

\textbf{Two-stage vvKRR rate.}
We now give the rate used to verify the nuisance conditions in the main
asymptotic linearity theorem.  The statement is written for a generic
training fold $I$ of size $N$.  In the two-fold split of the main text,
$N=n/2$; in $K$-fold cross-fitting with fixed $K$, $N\asymp n$.
Throughout this paragraph, the coordinate $j\in[d_y]$ is fixed and the
constants are uniform in $j$.

We overload the $\widehat{\cdot}$ notation to denote both quantities dependent on estimated regression function $\widehat{m}$ and the empirical measure $P^{I}_{W,N}:=\frac1N\sum_{i\in I}\delta_{W_i}$. 
For $i\in I$, write
\[
    \widehat{\mathcal S}_i
    := (\partial_j\varphi_l)\{Y_i-\widehat m(W_i)\},
    \qquad
    \mathcal S_i
    := (\partial_j\varphi_l)\{Y_i-m_0(W_i)\}.
\]
Define the empirical and population covariance operators
\[
    \widehat C_W
    := \frac1N\sum_{i\in I}\varphi_q(W_i)\otimes\varphi_q(W_i),
    \qquad
    C_W
    := \mathbb E_0\{\varphi_q(W)\otimes\varphi_q(W)\},
\]
and their regularized counterparts
\begin{align*}
    \widehat C_{W,\lambda}:=\widehat C_W+\lambda I_{\mathcal H_q},
    \qquad
    C_{W,\lambda}:=C_W+\lambda I_{\mathcal H_q}.
\end{align*}
We recall the definition of the two-stage vvKRR estimator $\widehat{v}_{\lambda}$, and present the analogous oracle learner $\widetilde{v}_{\lambda}$:
\begin{align}
    \widehat v_\lambda
    &= \argmin_{v\in\mathcal G}
    \left\{\frac1N\sum_{i\in I}
    \|\widehat{\mathcal S}_i-v(W_i)\|_{\mathcal H_l}^2
    +\lambda\|v\|_{\mathcal G}^2\right\},
    \label{eq:empirical_v_lambda}\\
    \widetilde v_\lambda
    &:= \argmin_{v\in\mathcal G}
    \left\{\frac1N\sum_{i\in I}
    \|\mathcal S_i-v(W_i)\|_{\mathcal H_l}^2
    +\lambda\|v\|_{\mathcal G}^2\right\}.
    \label{eq:empirical_v_oracle_lambda}
\end{align}
Here $\mathcal G$ is the vector-valued RKHS induced by
$K(w,w')=q(w,w')I_{\mathcal H_l}$.

We define the effective dimension
\begin{align*}
    \mathcal N(\lambda):=\Tr\{C_W(C_W+\lambda I)^{-1}\}.
\end{align*}
We now establish an upper bound for $\mathcal{N}(\lambda)$ and use it to derive an useful inequality. We define the event with respect to the randomness in $P^{I}_{W,N}$
\begin{align*}
    \mathfrak{E} = \left\{\left\|C_{W,\lambda}^{-1/2}(C_W-\widehat C_W)
    C_{W,\lambda}^{-1/2}\right\|_{\mathrm{op}}
    \le \frac23\right\}.
\end{align*}
\begin{lemma}
\label{lem:leverage_bound}
Under \eqref{asst:evd}, there is a constant $D<\infty$ such that
the effective dimension $\mathcal{N}(\lambda) \le D\lambda^{-p}$. On the event $\mathfrak{E}$, we have
\begin{equation}
\label{eq:leverage_bound}
    \frac1N\sum_{i\in I}
    \left\|C_W^{1/2}\widehat C_{W,\lambda}^{-1}\varphi_q(W_i)
    \right\|_{\mathcal H_q}^2
    \le 3D\lambda^{-p}.
\end{equation}
Furthermore, for every $\tau>0$, if 
\begin{equation}
\label{eq:slt_lb_N}
    N \ge
    8\kappa_q^2\tau\lambda^{-1}
    \log\left(
        2e\lambda^{-p}
        \frac{\|C_W\|_{\mathrm{op}}+\lambda}{\|C_W\|_{\mathrm{op}}}
    \right),
\end{equation}
then we have $\mathbb{P}(\mathfrak{E}) \geq 1 - 2e^{-\tau}$.
\end{lemma}

\begin{proof}
The existence of $D$ such that inequality $\mathcal{N}(\lambda)\leq D\lambda^{-p}$ holds follows from \citet[Lemma 11]{fischer2020sobolev} and \eqref{asst:evd}. We have
\begin{align*}
    &\frac1N \sum_{i\in I}
    \left\|C_W^{1/2}\widehat C_{W,\lambda}^{-1}\varphi_q(W_i)
    \right\|_{\mathcal H_q}^2\\
    &= \Tr\!
    \left\{C_W^{1/2}\widehat C_{W,\lambda}^{-1}
    \widehat C_W\widehat C_{W,\lambda}^{-1}C_W^{1/2}\right\}\\
    &= \Tr\!
    \left\{C_W^{1/2}\widehat C_{W,\lambda}^{-\frac{1}{2}}\widehat C_{W,\lambda}^{-\frac{1}{2}}
    \widehat C_W\widehat C_{W,\lambda}^{-\frac{1}{2}}\widehat C_{W,\lambda}^{-\frac{1}{2}}C_W^{1/2}\right\}\\
    &\leq \Tr\!
    \left\{C_W^{1/2}\widehat C_{W,\lambda}^{-\frac{1}{2}}\widehat C_{W,\lambda}^{-\frac{1}{2}}C_W^{1/2}\right\}\\
    &= \Tr\!
    \left\{\widehat C_{W,\lambda}^{-1}C_W\right\}
\end{align*}
where we use
\[
    0\le
    \widehat C_{W,\lambda}^{-1/2}\widehat C_W
    \widehat C_{W,\lambda}^{-1/2}
    \le I_{\mathcal H_q},
\]
by the spectral calculus, and the cyclicity of the trace in the last step.
Now set
\[
    T_\lambda:=C_{W,\lambda}^{-1/2}C_WC_{W,\lambda}^{-1/2},
    \qquad
    R_\lambda:=C_{W,\lambda}^{1/2}\widehat C_{W,\lambda}^{-1}
    C_{W,\lambda}^{1/2}.
\]
The operator $T_\lambda$ is positive trace class and $R_\lambda$ is
positive bounded, whence
\begin{align*}
    \Tr(\widehat C_{W,\lambda}^{-1}C_W)
    &= \Tr(R_\lambda T_\lambda)
     \le \|R_\lambda\|_{\mathrm{op}}\Tr(T_\lambda) \\
    &= \left\|C_{W,\lambda}^{1/2}\widehat C_{W,\lambda}^{-1}
    C_{W,\lambda}^{1/2}\right\|_{\mathrm{op}}
    \Tr\{C_W(C_W+\lambda I)^{-1}\}\\
    &\leq D\lambda^{-p}\left\|C_{W,\lambda}^{1/2}\widehat C_{W,\lambda}^{-1}
    C_{W,\lambda}^{1/2}\right\|_{\mathrm{op}}. 
\end{align*}
We have the algebraic identity
\begin{align*}
    C_{W,\lambda}^{1/2}\widehat C_{W,\lambda}^{-1}C_{W,\lambda}^{1/2}
    =\left[I_{\mathcal H_q}
    -C_{W,\lambda}^{-1/2}(C_W-\widehat C_W)
    C_{W,\lambda}^{-1/2}\right]^{-1}.
\end{align*}
Writing the RHS as a Von Neumann series, we see on the event $\mathfrak{E}$, we have
\begin{align*}
    \left\|C_{W,\lambda}^{1/2}\widehat C_{W,\lambda}^{-1}
    C_{W,\lambda}^{1/2}\right\|_{\mathrm{op}} \leq 3. 
\end{align*}
Therefore \cref{eq:leverage_bound} is proved. The last statement of the lemma on the lower bound of $\mathbb{P}(\mathfrak{E})$ under \cref{eq:slt_lb_N} follows from \citet[Lemma 17]{fischer2020sobolev}.
\end{proof}

\begin{theorem}[Two-stage vvKRR rate]
\label{thm:appendix_vkrr_rate}
Assume \Cref{assump:separable,assump:measurable,assump:bounded,asst:src,asst:evd,asst:mom}.
Suppose also that
$r\mapsto(\partial_j\varphi_l)(r)$ is $L$-Lipschitz as a map
$\mathbb R^{d_y}\to\mathcal H_l$.  Let
\[
    \lambda_N= N^{-1/(\beta+p)}.
\]
Then, for every $\tau>0$ and all sufficiently large $N$ satisfying
\cref{eq:slt_lb_N} with $\lambda=\lambda_N$, there are constants
$D,J<\infty$, independent of $N$ and $\tau$, such that, with probability at
least $1-7e^{-\tau}$ under the joint law of the training fold,
\begin{equation}
\label{eq:appendix_vkrr_rate}
    \big\|[\widehat v_{\lambda_N}]-v_0\big\|_{L^2(P_{W,0};\mathcal H_l)}
    \le
    L\sqrt{3D}\,\lambda_N^{-p/2}
    \|\widehat m-m_0\|_{L^2(P^{I}_{W,N})}
    +\tau\sqrt J\,N^{-\beta/\{2(\beta+p)\}}.
\end{equation}
Plugging in $\lambda_N= N^{-1/(\beta+p)}$ yields
\begin{equation}
\label{eq:appendix_vkrr_rate_simplified}
    \big\|[\widehat v_{\lambda_N}]-v_0\big\|_{L^2(P_{W,0};\mathcal H_l)}
    \le
    C_\tau\left\{
        N^{p/\{2(\beta+p)\}}
        \|\widehat m-m_0\|_{L^2(P^{I}_{W,N})}
        +N^{-\beta/\{2(\beta+p)\}}
    \right\},
\end{equation}
where $C_\tau$ is independent of $N$.
\end{theorem}

\begin{proof}
By the representer theorem and the isomorphism of the vector-valued RKSH 
$\mathcal G\simeq S_2(\mathcal H_q,\mathcal H_l)$, there exist
$\widehat G_\lambda,\widetilde G_\lambda\in S_2(\mathcal H_q,\mathcal H_l)$
such that
\[
    \widehat v_\lambda(w)=\widehat G_\lambda\varphi_q(w),
    \qquad
    \widetilde v_\lambda(w)=\widetilde G_\lambda\varphi_q(w),
\]
with
\begin{align*}
    \widehat G_\lambda
    =\left\{\frac1N\sum_{i\in I}\widehat{\mathcal S}_i
    \otimes\varphi_q(W_i)\right\}\widehat C_{W,\lambda}^{-1},\quad \widetilde G_\lambda
    =\left\{\frac1N\sum_{i\in I}\mathcal S_i
    \otimes\varphi_q(W_i)\right\}\widehat C_{W,\lambda}^{-1}.
\end{align*}
Consequently,
\[
    \widehat G_\lambda-\widetilde G_\lambda
    =\left\{\frac1N\sum_{i\in I}(\widehat{\mathcal S}_i-\mathcal S_i)
    \otimes\varphi_q(W_i)\right\}\widehat C_{W,\lambda}^{-1}.
\]
Using the $L^2(P_{W,0})$ representation of vector-valued RKHS functions,
\[
    \|[\widehat v_\lambda]-[\widetilde v_\lambda]\|_{L^2(P_{W,0};\mathcal H_l)}
    =\| (\widehat G_\lambda-\widetilde G_\lambda)C_W^{1/2}
    \|_{S_2(\mathcal H_q,\mathcal H_l)}.
\]
Let $b_i=C_W^{1/2}\widehat C_{W,\lambda}^{-1}\varphi_q(W_i)$. Applying the Cauchy-Schwarz inequality in the vector space $S_2(\mathcal H_q,\mathcal H_l)$ gives
\begin{align*}
    \|[\widehat v_\lambda]-[\widetilde v_\lambda]\|_{L^2(P_{W,0};\mathcal H_l)}
    &\le
    \left(\frac1N\sum_{i\in I}
    \|\widehat{\mathcal S}_i-\mathcal S_i\|_{\mathcal H_l}^2\right)^{1/2}
    \left(\frac1N\sum_{i\in I}\|b_i\|_{\mathcal H_q}^2\right)^{1/2}.
\end{align*}
Since $\partial_j\varphi_l$ is $L$-Lipschitz, we have
\[
    \left(\frac1N\sum_{i\in I}
    \|\widehat{\mathcal S}_i-\mathcal S_i\|_{\mathcal H_l}^2
    \right)^{1/2}
    \le
    L\|\widehat m-m_0\|_{L^2(P^I_{W,N})}.
\]
Therefore, on the event $\mathfrak E$,
\begin{equation}
\label{eq:feasible_minus_oracle_v_bound}
\begin{aligned}
    \|[\widehat v_\lambda]-[\widetilde v_\lambda]\|_{L^2(P_{W,0};\mathcal H_l)}
    &\le
    L\|\widehat m-m_0\|_{L^2(P^I_{W,N})}
    \left(\frac1N\sum_{i\in I}
    \|C_W^{1/2}\widehat C_{W,\lambda}^{-1}\varphi_q(W_i)\|_{\mathcal H_q}^2
    \right)^{1/2} \\
    &\le
    L\sqrt{3D}\lambda^{-p/2}
    \|\widehat m-m_0\|_{L^2(P^I_{W,N})}.
\end{aligned}
\end{equation}
Under
\eqref{asst:src}, \eqref{asst:evd}, and \eqref{asst:mom}, \citet[Theorem 3]{li2024towards} says that, for
$\lambda_N\asymp N^{-1/(\beta+p)}$, there exists $J<\infty$ such that
with probability at least $1-5e^{-\tau}$,
\[
    \|[\widetilde v_{\lambda_N}]-v_0\|_{L^2(P_{W,0};\mathcal H_l)}
    \le \tau\sqrt J\,N^{-\beta/\{2(\beta+p)\}}.
\]
Taking union bound with event $\mathfrak{E}$, which satisfies $\mathbb{P}(\mathfrak{E})\geq 1-2e^{-\tau}$ under \cref{eq:slt_lb_N}, we have \cref{eq:appendix_vkrr_rate} hold with probability at least $1 - 7e^{-\tau}$.
\end{proof}

\textbf{Consistency of the residual mean embedding.}
\begin{lemma}[Consistency and rate for $\widehat\mu_\xi$]
\label{lem:muxi_consistency}
Assume $\sup_u l(u,u)\le\kappa_l^2$ and
$\|\varphi_l(u)-\varphi_l(v)\|_{\mathcal H_l}\le L_l\|u-v\|_{\mathbb R^{d_y}}$.
For a training fold $I$ of size $N$, define
\[
    \widehat\mu_{\xi}^{I}
    :=\frac1N\sum_{i\in I}\varphi_l\{Y_i-\widehat m(W_i)\},
    \qquad
    \mu_{\xi,0}:=P_0\varphi_l\{Y-m_0(W)\}.
\]
Then
\begin{equation}
\label{eq:muxi_bound}
    \|\widehat\mu_{\xi}^{I}-\mu_{\xi,0}\|_{\mathcal H_l}
    \le
    L_l\|\widehat m-m_0\|_{L^2(P^{I}_{W,N})}
    +
    \left\|\frac1N\sum_{i\in I}
    \left[\varphi_l\{Y_i-m_0(W_i)\}-\mu_{\xi,0}\right]\right\|_{\mathcal H_l}.
\end{equation}
The second term in \cref{eq:muxi_bound} is $O_p(N^{-1/2})$.  Hence,
if $\|\widehat m-m_0\|_{L^2(P^{I}_{W,N})}=o_p(1)$, then
$\|\widehat\mu_{\xi}^{I}-\mu_{\xi,0}\|_{\mathcal H_l}=o_p(1)$; if
$\|\widehat m-m_0\|_{L^2(P^{I}_{W,N})}=O_p(a_N)$, then
$\|\widehat\mu_{\xi}^{I}-\mu_{\xi,0}\|_{\mathcal H_l}=O_p(a_N+N^{-1/2})$.
\end{lemma}

\begin{proof}
By the triangle inequality,
\begin{align*}
    \|\widehat\mu_{\xi}^{I}-\mu_{\xi,0}\|_{\mathcal H_l}
    &\le
    \left\|\frac1N\sum_{i\in I}
    \left[\varphi_l\{Y_i-\widehat m(W_i)\}
    -\varphi_l\{Y_i-m_0(W_i)\}\right]\right\|_{\mathcal H_l}\\
    &\quad+
    \left\|\frac1N\sum_{i\in I}
    \left[\varphi_l\{Y_i-m_0(W_i)\}-\mu_{\xi,0}\right]\right\|_{\mathcal H_l}.
\end{align*}
The first term is bounded by
$L_lN^{-1}\sum_{i\in I}\|\widehat m(W_i)-m_0(W_i)\|\le
L_l\|\widehat m-m_0\|_{L^2(P^{I}_{W,N})}$.  For the second term, the summands
are i.i.d., mean zero, and have norm at most $2\kappa_l$.  Therefore its
second moment is at most $4\kappa_l^2/N$, which implies the asserted
$O_p(N^{-1/2})$ rate.
\end{proof}

\subsection{Sufficient conditions for asymptotic normality}
\label{subsec:suff_cond}

We now instantiate \Cref{asst: consistency,asst: suff_cond_an}.
\begin{corollary}[Rate-based sufficient conditions]
\label{cor:rate_based_sufficient_conditions}
Consider a fixed number $K$ of folds and let all training-fold sizes satisfy
$N_k\asymp n$.  Suppose \Cref{asst:nested_covariates,asst:kernel_boundedness,asst:asst_1,asst: differentiability_kernel_y,asst: boundedness_reg_fn} hold.  For each fold $k$, let $\widehat m^{(-k)}$ be the
first-stage estimator and let $\widehat v_j^{(-k)}$ be the feasible vvKRR
estimator in \eqref{eq:empirical_v_lambda}, fitted with
$\lambda_{N_k}\asymp N_k^{-1/(\beta+p)}$.
Assume that, for some deterministic $a_n\downarrow0$,
\begin{align}
    \max_{k\le K}\|\widehat m^{(-k)}-m_0\|_{L^2(P_{W,0})}
    &=O_p(a_n),
    \label{eq:m_rate_population}\\
    \max_{k\le K}\|\widehat m^{(-k)}-m_0\|_{L^2(P^{(-k)}_{W,N_k})}
    &=O_p(a_n),
    \label{eq:m_rate_empirical}\\
    \max_{k\le K}\|\widehat m^{(-k)}-m_0\|_{L^\infty(P_{W,0})}
    &=o_p(1).
    \label{eq:m_rate_supnorm}
\end{align}
Assume also that \eqref{asst:src}--\eqref{asst:mom} hold for every
$j\in[d_y]$, with constants uniform in $j$, and that
$r\mapsto\partial_j\varphi_l(r)$ is Lipschitz uniformly in $j$. Assume additionally that $\beta+p>1$, so that
\eqref{eq:slt_lb_N} holds for all sufficiently large $N$ when
$\lambda_N=N^{-1/(\beta+p)}$. We have
\begin{equation}
\label{eq:v_rate_abstract}
    \max_{k\le K}\max_{j\le d_y}
    \|\widehat v_j^{(-k)}-v_{0,j}\|_{L^2(P_{W,0};\mathcal H_l)}
    =O_p\left(n^{\frac{p}{2(\beta+p)}}a_n+n^{-\frac{\beta}{2(\beta+p)}}\right).
\end{equation}
Moreover, \Cref{asst: consistency,asst: suff_cond_an} hold provided
\begin{equation}
\label{eq:abstract_second_order_condition}
    a_n^2+a_n\left(n^{\frac{p}{2(\beta+p)}}a_n+n^{-\frac{\beta}{2(\beta+p)}}\right)
    =o(n^{-1/2}).
\end{equation}
In particular, if $a_n=n^{-\gamma_m}$, then
\cref{eq:abstract_second_order_condition} is implied by
\begin{equation}
\label{eq:abstract_smoothness_condition}
    2\gamma_m+\frac{\beta}{2(\beta+p)}>1.
\end{equation}
\end{corollary}

\begin{proof}
Since $K$ is fixed and $N_k\asymp n$, \Cref{thm:appendix_vkrr_rate} and
\cref{eq:m_rate_empirical} give \cref{eq:v_rate_abstract}, uniformly over
folds and coordinates.  The $L^2$ and $L^\infty$ consistency of
$\widehat m^{(-k)}$ follow from \cref{eq:m_rate_population} and
\cref{eq:m_rate_supnorm}.  The consistency of $\widehat v^{(-k)}$ follows
from \cref{eq:v_rate_abstract} whenever the right-hand side is $o_p(1)$;
this is implied by \cref{eq:abstract_second_order_condition}.  By
\Cref{lem:muxi_consistency},
\[
    \max_{k\le K}\|\widehat\mu_{\xi}^{(-k)}-\mu_{\xi,0}\|_{\mathcal H_l}
    =O_p(a_n+n^{-1/2})=o_p(1),
\]
and the analogous bound
$\max_k\|\widehat\mu_X^{(-k)}-\mu_{X,0}\|_{\mathcal H_k}=O_p(n^{-1/2})$
follows from boundedness of $k$.  This verifies \Cref{asst: consistency}. 

For \Cref{asst: suff_cond_an}, combine \cref{eq:m_rate_population} and
\cref{eq:v_rate_abstract}:
\[
    \|\widehat m-m_0\|_{L^2(P_{W,0})}^2
    +\|\widehat m-m_0\|_{L^2(P_{W,0})}
    \|\widehat v-v_0\|_{L^2(P_{W,0})}
    =O_p\!\left(
    a_n^2+a_n(n^{\frac{p}{2(\beta+p)}}a_n+n^{-\frac{\beta}{2(\beta+p)}})
    \right).
\]
Thus \cref{eq:abstract_second_order_condition} gives \Cref{asst: suff_cond_an}.

If $a_n=n^{-\gamma_m}$, the three terms in
\cref{eq:abstract_second_order_condition} have exponents
$2\gamma_m$, $2\gamma_m+\frac{\beta}{2(\beta+p)}-1/2$, and $\gamma_m+\frac{\beta}{2(\beta+p)}$.
The condition $2\gamma_m+\frac{\beta}{2(\beta+p)}>1$ implies each term is
$o(n^{-1/2})$.
\end{proof}

\textbf{Instantiation via Sobolev spaces.}
We now exhibit sufficient conditions for \Cref{cor:rate_based_sufficient_conditions} in terms of smoothness conditions. Let $d_w$ be the ambient dimension of $W$.

We use the following terminology from \citet{chen2025nested}. Let $\mathcal W\subseteq\mathbb R^{d_w}$.
A positive definite kernel $q$ on $\mathcal W$ is called a Sobolev
reproducing kernel of smoothness $a>0$ if its RKHS $\mathcal H_q$ is
norm-equivalent to the real-valued Sobolev space
$H^a(\mathcal W)=W_2^a(\mathcal W)$. That is, $\mathcal H_q$ and
$H^a(\mathcal W)$ contain the same functions and their norms are equivalent. In particular, when $a>d_w/2$, Sobolev embedding theorem \citep{adams2003sobolev} implies that functions in $\mathcal H_q$ have continuous representatives and that point evaluation is bounded. Hence Sobolev reproducing kernels of smoothness $a>d_w/2$ are bounded on compact domains. A standard example is the Mat\'ern-$\nu$ kernel, whose RKHS is norm-equivalent to a Sobolev space of smoothness $a=\nu+d_w/2$.

\begin{asst}[Sobolev smoothness]
\label{ass:sobolev_primitive}
The following conditions hold.
\begin{enumerate}[label=(S\arabic*), itemsep=0.25em]
    \item \label{sob:domain}
    $\mathcal W\subset\mathbb R^{d_w}$ is a compact Lipschitz domain, and
    $P_{W,0}$ has a density with respect to Lebesgue measure bounded above
    and below by positive constants on $\mathcal W$.

    \item \label{sob:m_smooth}
    The regression function satisfies
    $m_0\in H^{s_m}(\mathcal W;\mathbb R^{d_y})$ for some $s_m>d_w/2$.
    The first-stage estimator is a kernel ridge regression estimator with a
    Sobolev reproducing kernel, with regularization chosen so that, up to
    logarithmic factors,
    \[
        \|\widehat m-m_0\|_{L^2(P_{W,0})}=O_p(n^{-\gamma_m}),
        \qquad
        \gamma_m:=\frac{s_m}{2s_m+d_w},
    \]
    together with the empirical and sup-norm consistency requirements in
    \cref{eq:m_rate_empirical} and \cref{eq:m_rate_supnorm}.

    \item \label{sob:v_smooth}
    For every $j\in[d_y]$, the function
    \[
        v_{0,j}(w)
        =
        \mathbb E_0[
            (\partial_j\varphi_l)\{Y-m_0(W)\}
            \mid W=w
        ]
    \]
    belongs to $H^{s_v}(\mathcal W;\mathcal H_l)$ for some $s_v>0$,
    uniformly in $j$.

    \item \label{sob:q_kernel}
    The second-stage kernel $q$ is a Sobolev reproducing kernel of smoothness
    $a_v>d_w/2$ on $\mathcal W$. Assume $0<s_v\le 2a_v$ and $a_v<s_v+d_w/2$. In the vvKRR
    estimator of $v_{0,j}$, use
    \[
        \lambda_{v,n}\asymp n^{-2a_v/(2s_v+d_w)}.
    \]
    
\end{enumerate}
\end{asst}

\begin{corollary}[Sobolev sufficient conditions for asymptotic normality]
\label{cor:sobolev_sufficient_conditions}
Suppose \Cref{asst:nested_covariates,asst:kernel_boundedness,asst:asst_1,asst: differentiability_kernel_y,asst: boundedness_reg_fn,ass:sobolev_primitive} hold.  If
\begin{equation}
\label{eq:sobolev_second_order_condition}
    \frac{2s_m}{2s_m+d_w}
    +
    \frac{s_v}{2s_v+d_w}
    >1,
\end{equation}
then \Cref{asst: consistency,asst: suff_cond_an} hold. 
\end{corollary}

\begin{proof}
Fix $j\in[d_y]$. Since $q$ is a Sobolev reproducing kernel of smoothness
$a_v>d_w/2$ on the compact Lipschitz domain $\mathcal W$, the preliminary
kernel assumptions on $q$ in
\Cref{assump:separable,assump:measurable,assump:bounded} are satisfied.

By the Sobolev equivalence for $\mathcal H_q$, the interpolation space
$[\mathcal H_q]^\beta$ is norm-equivalent to
$H^{a_v\beta}(\mathcal W)$. Taking $\beta=\frac{s_v}{a_v}$ 
therefore gives $[\mathcal H_q]^\beta\simeq H^{s_v}(\mathcal W)$. Hence
\Cref{sob:v_smooth} implies the source condition \eqref{asst:src}. The
standard eigenvalue decay for Sobolev kernels on compact domains, together
with the density assumption in \Cref{sob:domain}, gives \eqref{asst:evd}
with $p=\frac{d_w}{2a_v}$. 

It remains to check \eqref{asst:mom} and Lipschitz assumptions used in
\Cref{thm:appendix_vkrr_rate}. By \Cref{asst:asst_1},
\[
    M_j:=\sup_{r\in\mathbb R^{d_y}}
    \|\partial_j\varphi_l(r)\|_{\mathcal H_l}<\infty .
\]
Thus
\[
    \|\mathcal S_j\|_{\mathcal H_l}
    =
    \|(\partial_j\varphi_l)\{Y-m_0(W)\}\|_{\mathcal H_l}
    \le M_j
\]
almost surely, and also $\|v_{0,j}(W)\|_{\mathcal H_l}\le M_j$ almost
surely. Hence
\[
    \|\mathcal S_j-v_{0,j}(W)\|_{\mathcal H_l}\le 2M_j
\]
almost surely, which implies the Bernstein moment condition
\eqref{asst:mom}. Moreover, the uniform bound on $D^2\varphi_l$ in
\Cref{asst: differentiability_kernel_y} implies that $D\varphi_l$ is
Lipschitz, and therefore each coordinate derivative
$r\mapsto\partial_j\varphi_l(r)$ is Lipschitz as a map
$\mathbb R^{d_y}\to\mathcal H_l$.

By \Cref{sob:m_smooth}, the first-stage estimator satisfies
\cref{eq:m_rate_population}--\cref{eq:m_rate_supnorm} with
\[
    a_n=n^{-\gamma_m},
    \qquad
    \gamma_m=\frac{s_m}{2s_m+d_w},
\]
up to logarithmic factors. Thus
\Cref{cor:rate_based_sufficient_conditions} applies with
\[
    \gamma_m=\frac{s_m}{2s_m+d_w},
    \qquad
    \frac{\beta}{2(\beta+p)}
    =
    \frac{s_v}{2s_v+d_w}.
\]
The condition \cref{eq:sobolev_second_order_condition} is exactly
\[
    2\gamma_m+\frac{\beta}{2(\beta+p)}>1,
\]
which is sufficient for \Cref{asst: suff_cond_an} by
\Cref{cor:rate_based_sufficient_conditions}. 
\end{proof}

\end{document}